\documentclass[11pt]{article}

\usepackage{fullpage}

\usepackage[T1]{fontenc}
\usepackage[utf8]{inputenc}
\usepackage{lmodern}
\usepackage{microtype}
\usepackage{amsmath,mathtools}
\usepackage{amsfonts,amssymb}
\usepackage{amsthm,thmtools}
\usepackage{bm}
\usepackage{bbm}
\usepackage{accents}

\usepackage{graphicx}
\usepackage{tikz}
\AtBeginEnvironment{tikzpicture}{%
  \renewcommand{\small}{\fontsize{9}{11}\selectfont}%
  \renewcommand{\scriptsize}{\fontsize{7}{8}\selectfont}%
}
\usepackage{xcolor}
\usepackage{booktabs}
\usepackage{array}
\usepackage{tabularx}
\usepackage{threeparttable}
\usepackage{caption}
\usepackage{subcaption}
\usepackage{makecell}
\usepackage{pifont}

\usepackage{enumitem}

\usepackage{algorithm}
\usepackage{algpseudocode}
\usepackage[round]{natbib}
\usepackage[colorlinks=true,citecolor=blue,linkcolor=red,urlcolor=magenta]{hyperref}
\usepackage[capitalise,noabbrev]{cleveref}

\newtheorem{theorem}{Theorem}
\newtheorem{lemma}{Lemma}
\newtheorem{proposition}{Proposition}
\newtheorem{corollary}{Corollary}

\theoremstyle{definition}

\newtheorem{assumption}{Assumption}

\theoremstyle{remark}

\crefname{assumption}{assumption}{assumptions}
\Crefname{assumption}{Assumption}{Assumptions}

\newcommand{\R}{\mathbb{R}}
\renewcommand{\P}{\mathbb{P}}
\newcommand{\E}{\mathbb{E}}
\newcommand{\ind}[1]{\mathbf{1}\{#1\}}
\newcommand{\norm}[1]{\lVert#1\rVert}
\newcommand{\KL}{D_{\mathrm{KL}}}
\newcommand{\TV}{\operatorname{TV}}
\newcommand{\Law}{\operatorname{Law}}
\newcommand{\Bern}{\operatorname{Bern}}
\newcommand{\Geom}{\operatorname{Geom}}
\DeclareMathOperator*{\argmin}{arg\,min}
\DeclareMathOperator*{\argmax}{arg\,max}
\newcommand{\IA}{\mathrm{IA}}
\newcommand{\WC}{\mathrm{WC}}
\newcommand{\SEPT}{\mathrm{SEPT}}
\def\eps{{\epsilon}}
\newcommand{\wtO}{\widetilde O}
\newcommand{\Unif}{\mathrm{Unif}}
\def\eqref#1{equation~\ref{#1}}

\title{Nonpreemptive Scheduling While Learning Context-Dependent Service Rates}

\author{
Wansoo Choi$^{1}$ \quad Seoungbin Bae$^{2}$ \quad Dabeen Lee$^{1}$\\[0.3em]
$^{1}$Department of Mathematical Sciences, Seoul National University\\
$^{2}$Department of Industrial \& Systems Engineering, KAIST\\
\texttt{choiws0624@snu.ac.kr, sbbae31@kaist.ac.kr, dabeenl@snu.ac.kr}
}

\date{}

\begin{document}
\maketitle

\begin{abstract}
We study nonpreemptive contextual queueing bandits in a single-server system. Each job is represented by a $d$-dimensional context vector; in each round, a job may arrive with its context drawn from an unknown distribution $\mathcal D$, and its departure probability is determined by a logistic model of that context vector with an unknown parameter $\theta^*$. The server learns from service outcomes while deciding which waiting job to serve and whether to idle, aiming to minimize queue-length regret, the gap between its expected terminal queue length and the minimum achievable by an admissible policy. Once selected, a job must be served until completion, and we refer to this as the nonpreemptive setting. A central challenge is that, even with full model knowledge, the optimal policy cannot in general be characterized by a simple myopic rule, since the optimal action can change with the remaining horizon at the same queue state. Nevertheless, when the model and horizon are known, the optimal action can be obtained through a finite-horizon Bellman recursion. Motivated by this, we propose Learn--Clear--Plan (LCP), which estimates the system and uses the resulting Bellman recursion to make horizon-dependent decisions. LCP achieves $\wtO(\sqrt{d/T})$ queue-length regret, while a lower-bound construction gives $\Omega(\min\{1/\sqrt d,\sqrt{d/T}\})$ regret for every learning policy on some instance, establishing optimality up to polylogarithmic factors when $T\ge d^2$. When the horizon is unknown, no horizon-independent policy achieves vanishing regret against the finite-horizon optimum. We therefore use SEPT, the policy that serves a waiting job with the highest probability of departure, as a fixed reference, and suggest an estimated-SEPT algorithm that achieves a tracking error of $\wtO(\sqrt{d/t})$ without knowing the model.

\end{abstract}

\section{Introduction}

Queueing systems are central to cloud computing, online service platforms, communication networks, and other applications in which heterogeneous jobs compete for limited service capacity \citep{neely2010stochastic}. In many such systems, a job's completion probability depends on observable features, while the corresponding service model is initially unknown. A scheduler must therefore learn context-dependent service rates from departure feedback while simultaneously controlling congestion. Queueing bandits formalize this learning-while-scheduling problem, and queue-length regret measures the excess number of unfinished jobs relative to an optimal policy with full knowledge of the service model \citep{krishnasamy2021learning,stahlbuhk2021learning}.

Contextual queueing bandits extend this framework to job-specific features. Under the assumption that the contexts of arriving jobs are drawn independently from a fixed distribution, \citet{bae2026queue} introduce a preemptive contextual queueing bandit model, in which the scheduler may reconsider the entire queue and select a different job--server pair in the next round after an unsuccessful service attempt. They obtain $\wtO(T^{-1/4})$ queue-length regret in this setting. For the same stochastic-context setting, \citet{bae2026algorithm} subsequently propose the CQB-$\eta$-2 algorithm and improve the upper bound to $\wtO(T^{-1/2})$. They further show that every learning algorithm incurs $\Omega(T^{-1/2})$ regret on some instance, showing that this dependence on $T$ is optimal up to logarithmic factors.

This preemptive abstraction facilitates modeling bandit feedback, characterizing optimal policies, and analyzing regret, but it excludes systems in which an ongoing computation, repair, or communication session cannot be interrupted, or in which preemption and migration are prohibitively costly. To capture such service commitments, we study a contextual queueing bandit model with a \emph{nonpreemptive setting}, where once a job starts service, the server must continue processing it until completion. This change substantially complicates optimal scheduling. In the preemptive setting, an optimal policy is \emph{work-conserving}, i.e., not idle, and selects an available job--server pair with the highest departure probability \citep{bae2026queue}. However, under nonpreemption, the optimal action can depend on the current queue state, remaining horizon, and arriving job's distribution even in the single-server case (see \Cref{sec:challenges-techniques} for a detailed discussion). So, the optimal policy need not admit such a simple myopic characterization as in the preemptive setting, making learning and planning fundamentally intertwined. To study these difficulties in their most basic form, we focus first on the single-server setting. Our main results for known and unknown horizons are summarized in \Cref{tab:main-results}.

\begin{table}
\centering
\caption{Summary of main results. Known/unknown indicates whether the server knows the horizon; IA/WC denote the idling-allowed/work-conserving settings, and SEPT denotes the shortest-expected-processing-time policy. Logarithmic factors and fixed model constants are suppressed.}
\label{tab:main-results}
\small

\begingroup
\setlength{\tabcolsep}{3pt}
\renewcommand{\arraystretch}{1.05}
\renewcommand{\tabularxcolumn}[1]{m{#1}}

\begin{tabularx}{\textwidth}{
|>{\centering\arraybackslash\hsize=.75\hsize}X
|>{\centering\arraybackslash\hsize=1.35\hsize}X
|>{\centering\arraybackslash\hsize=1.15\hsize}X
|>{\centering\arraybackslash\hsize=.75\hsize}X|
}
\hline
\textbf{Setting}
& \textbf{Benchmark / Objective}
& \textbf{Result}
& \textbf{Note} \\
\hline

Known $T$, IA
& Regret against optimal IA
& $\wtO(\sqrt{d/T})$
& \Cref{alg:lcp}, \Cref{thm:known-horizon} \\
\hline

Known $T$, IA/WC
& Regret against optimum
& $\Omega(\min\{d^{-1/2},\sqrt{d/T}\})$
& \Cref{thm:lower-wc} \\
\hline

Unknown, IA/WC
& Vanishing regret against finite-horizon optimum
& Impossible:
$\limsup_{t\to\infty}R_t^{\mathsf S}(\pi)>0$
& \Cref{thm:ia-wc-impossible} \\
\hline

Unknown, WC
& Tracking true SEPT
& $\wtO(\sqrt{d/t})$
& \Cref{alg:anytime}, \Cref{thm:anytime} \\
\hline
\end{tabularx}

\endgroup
\end{table}

\paragraph{Known horizon: Learn--Clear--Plan algorithm and its optimality}
For a known horizon $T$, we compare with the optimal finite-horizon policy that may idle. Our Learn--Clear--Plan (LCP) algorithm estimates the departure probability function of jobs and their departure probability distribution, waits until the queue empties, and then follows an empirical finite-horizon Bellman policy (the policy that follows from the Bellman recursion, which is calculated by using the estimated function and distribution). Under some conditions related to job arrivals, \Cref{thm:known-horizon} shows that it achieves $\wtO(\sqrt{d/T})$ as a regret upper bound. In particular, unlike the preemptive results of \citet{bae2026queue,bae2026algorithm}, this guarantee does not require a lower-eigenvalue bound on the context covariance matrix. We complement this result with the lower bounds in \Cref{thm:lower-wc,thm:lower-ia}: for fixed model and traffic parameters, every learning policy incurs $\Omega(\min\{d^{-1/2}, \sqrt{d/T}\})$ regret against the optimum policy on some instance, regardless of whether idling is permitted or not. Hence, when $T\ge d^2$, our upper bound under fixed horizon is optimal in its dependence on $d$ and $T$ up to logarithms.

\paragraph{Unknown horizon: Impossibility and SEPT tracking.}
When the horizon is unknown, \Cref{thm:ia-wc-impossible} shows that there is no horizon-free policy such that the $t$-th round queue-length regret converges to $0$ as $t\rightarrow \infty$, even when the service model is known. We therefore use SEPT, the policy that serves a waiting job with the highest departure probability whenever the server is free, as a fixed horizon-independent reference policy. Our estimated-SEPT algorithm updates its logistic estimator only between busy periods and follows the resulting priority rule during each busy period. By \Cref{thm:anytime}, its tracking error satisfies $\mathcal T_t^{\SEPT}=\wtO(\sqrt{d/t})$ for fixed model and traffic parameters. Thus, our known-horizon result concerns regret relative to a finite-horizon optimum, whereas our unknown-horizon result concerns tracking a fixed reference policy.

\section{Problem Setting}\label{sec:problem-setting}

We study the nonpreemptive contextual queueing bandit problem in a discrete-time system consisting of a single queue and a single server. Each job is represented by a context vector in $\mathcal X\subseteq\mathbb R^d$, which determines its service-completion probability. We study both a \emph{known-horizon} setting, where the server knows the evaluation horizon $T\in\mathbb N$ and may design its policy accordingly, and an \emph{unknown-horizon} setting, where the server does not know the evaluation horizon and must use a horizon-independent policy (we call this an \emph{anytime} policy). Let $Q_t$ denote the number of unfinished jobs present at the beginning of round $t$, and call it the \emph{queue length} at round $t$. We assume that the system is initially empty, so $Q_1=0$. A key feature of our model is the \emph{nonpreemptive} setting. This means that once the server selects a job, it must continue serving that job in every subsequent round until the job departs or the horizon ends. We call a job that is currently in service a \emph{job in service}, and call all other unfinished jobs \emph{waiting jobs}. The server is called \emph{free} when no job is in service and the queue is not empty, and when it is free, it may start one waiting job or remain \emph{idle}, meaning that it performs no service in that round. If the queue is empty, it idles necessarily.

Each round consists of a \emph{service stage} (\emph{departure process}) followed by an \emph{arrival stage} (\emph{arrival process}). If there is a job in service at the beginning of round $t$, the server must serve that job again. Otherwise, it may start one waiting job of its choice or idle. Here, the server can observe whether the served job is completed. Let $D_t\in\{0,1\}$ be the departure indicator in round $t$, with $D_t=0$ if the server idles. If a job with context $x$ is served in round $t$, its departure indicator has distribution $\Bern(p(x))$, where $p(x):=\mu(x^\top\theta^*)$ is called the \emph{success probability function} and $\mu(z):=(1+e^{-z})^{-1}$. Here, $\theta^*\in\mathbb R^d$ is called the \emph{parameter vector}. Conditional on the contexts, departure indicators are independent across jobs and rounds. Equivalently, if $G(x)$ is the number of service attempts required to complete a job with context $x$, then $G(x)\sim\Geom(p(x))$ and $\E[G(x)]=1/p(x)$ holds. After the service stage, a new job arrives with indicator $A_t\sim\Bern(\lambda)$. Conditional on $A_t=1$, draw $X_t\sim\mathcal D$ from a context distribution $\mathcal{D}$ independently. The arriving job becomes available for service in round $t+1$, and we denote $X_t$ as the context of the job that arrived in round $t$. Then the queue length evolves as $Q_{t+1}=(Q_t-D_t)^++A_t$, where $(q)^+:=\max\{0, q\}$.

Admissible policies may randomize, although the algorithms proposed in this paper are deterministic conditional on the observed history. We consider two problem settings. In the \emph{idling-allowed} (IA) setting, policies are allowed to idle whenever the server is free. In the \emph{work-conserving} (WC) setting, policies must start a waiting job whenever the server is free and the queue is nonempty. We refer to policies admissible in these settings as IA and WC policies, respectively. Thus, every WC policy is also an IA policy. For a horizon $T\in\mathbb N$, denote $J_T^\pi:=\E_\pi[Q_{T+1}]$ as the expected final queue length when the server follows policy $\pi$. Let $V_T^{\IA}$ and $V_T^{\WC}$ denote the minimum expected final queue lengths over all IA and WC policies, respectively, with knowledge of $T$, $\theta^*$, and $\mathcal D$. We define the corresponding regrets as $R_T^{\IA}(\pi):=J_T^\pi-V_T^{\IA}, R_T^{\WC}(\pi):=J_T^\pi-V_T^{\WC}$.
\begin{assumption}
\label{ass:model}
The sequence $(A_t,X_t)_{t\ge1}$ is i.i.d., with $A_t$ independent of $X_t$. The server knows $\lambda\in(0,1)$, $S>0$, and $0<p_-<p_+<1$, while $\theta^*$ and $\mathcal D$ are unknown. Moreover, $\norm{X}_2\le1$ almost surely, $\norm{\theta^*}_2\le S$, $p(X)\in[p_-,p_+]$ almost surely, and $\rho:=\lambda\E_{X\sim \mathcal{D}}\left[1/p(X)]\right]<1$ holds.
\end{assumption}

In \Cref{ass:model}, except for the final condition $\rho<1$, these are nothing but the standard boundedness and stochastic-context assumptions for logistic bandit models \citep{bae2026queue,bae2026algorithm}. The condition $\rho<1$ plays a crucial role in our analysis, as it is a key ingredient in establishing the workload properties in \Cref{lem:workload}, which are repeatedly used throughout the paper. See \Cref{subsec: workload def,app:workload} for its interpretation and detailed use. It also serves as a replacement for the traffic-slack condition used in the prior work.

\section{Main Difficulties and new techniques}
\label{sec:challenges-techniques}
In this section, we highlight the main difficulty arising from the nonpreemptive setting and introduce a new tool, which we call workload, to address this difficulty.

\subsection{Difficulties of nonpreemptive setting}
\label{subsec:difficulties}
A central difficulty in nonpreemptive contextual queueing bandits is the loss of the simple optimal scheduling rule available in the preemptive setting. In the preemptive model of \citet{bae2026queue,bae2026algorithm}, an optimal policy is work-conserving and selects a waiting job with the highest departure probability when the server is free. Under nonpreemption, this myopic rule need not be optimal, as the following counterexamples show. Suppose that a free server has one waiting job with departure probability $p_s$. With one round remaining, serving it immediately is trivially optimal. However, with two rounds remaining, the advantage in expected departures of idling first and then serving the best available job, compared with serving immediately, is at least the following value.
\begin{equation}\label{eq:idle-or-work}
    p_s^2-p_s(1+\lambda)
    +\lambda\left(
        \E\!\left[\max\{p_s,p(X)\}\right]
        -p_s\E[p(X)]
    \right).
\end{equation}
A detailed proof is included in \Cref{app:ia-impossibility}.

We can construct a two-context instance satisfying \Cref{ass:model} such that the quantity (\ref{eq:idle-or-work}) is positive, as shown in \Cref{app:ia-impossibility}. This counterexample shows that a work-conserving policy is not always optimal. Furthermore, even in the WC setting, choosing a job with the highest departure probability is not always optimal, as shown in the following proposition.
\begin{proposition}
\label{prop:wc-reversal}
There is a stable three-context instance with $\lambda=\frac{1}{2}$ and departure probabilities $p_A=\frac{7}{20}$, $p_B=\frac{13}{20}$, and $p_F=\frac{19}{20}$ such that when the server is free and exactly one $A$ job and one $B$ job are waiting, serving $B$ first is the unique optimal action with one round remaining. With fourteen rounds remaining, serving $A$ first is the unique optimal action. This instance satisfies \Cref{ass:model}.
\end{proposition}
The proof of this proposition will be given in \Cref{app:wc-certificate}. Hence, in the nonpreemptive setting, there is no universal horizon-independent priority rule, and we need to know the remaining horizon, the current queue, and the future-arrival distribution to decide which action is optimal in each round. Thus, this explains that determining an optimal action requires solving the finite-horizon Bellman recursion at each round.

\subsection{A New Technique: Workload-Based Arrival Analysis}
\label{subsec: workload def}
In the nonpreemptive setting, the optimal policy depends on the queue state and remaining horizon rather than a simple myopic rule, making departure-based regret analysis difficult. From this viewpoint, we shift our attention to the arrival process, whose law is much simpler and independent of the scheduling policy. For each arriving job with context $X_t$, independently draw its required service time $G(X_t)\sim\Geom(p(X_t))$, and define the \emph{workload} $W_t$ as the sum of the remaining required service times of all jobs present at the beginning of round $t$. Under any WC policy, the following recursion holds.
\begin{equation}
    W_{t+1}=(W_t-1)^+ + A_tG(X_t)
    \label{eq: workload recursion}
\end{equation}
Thus, under a common coupling of arrivals, arrived jobs, and their required service times, the workload is independent of the service policy and order, satisfies $Q_t\le W_t$, and vanishes iff $Q_t=0$. Thus, any two WC queues $Q_t^{(1)}, Q_t^{(2)}$ coupled as above share the same workload process $(W_t)_{t\ge1}$, become empty at the same times, and satisfy $|Q_t^{(1)}-Q_t^{(2)}|\le W_t$ for all $t$. Furthermore, with this definition, $\rho=\lambda\E[1/p(X)]=\lambda\E_{X\sim \mathcal{D}, G(X)\sim \Geom(p(X))}[G(X)]$ is precisely the expected amount of workload brought into the system by arrivals in each round.

We now introduce the following basic workload properties. The constants in the following are formally defined in \Cref{app:workload}, and the properties are proved in \Cref{app:workload-lemma-1-2,app:workload-lemma-3}.
\begin{lemma}\label{lem:workload}
Under \Cref{ass:model}, there exist positive constants $r\in(0,1)$, $\psi(r)\in(0,1)$, $M_G(r)$, $K_r$ and $C_W(r)$ such that the WC workload process started from $W_1=0$ satisfies $\E e^{rW_t}\le K_r$ and $\E W_t\le C_W(r)$ for every $t\ge1$. Moreover, for $u\ge1$, let $\tau_0:=\inf\{j\ge0:W_{u+j}=0\}$. Then,
\begin{equation}\label{eq:workload-tail}
\E\left[
W_{u+h}\ind{\tau_0>h}\,\middle|\,W_u
\right]
\le \frac{e^{rW_u}}{r}\psi(r)^h
\end{equation}
for every $h\ge0$. Hence, $\E[W_{u+h}\ind{\tau_0>h}]\le(K_r/r)\psi(r)^h$. We refer to these two inequalities as the workload tail bounds. Lastly, $\E[W_{t+1}-W_t\mid W_t=w]=\rho-1<0$ for every $w\ge1$, so the workload recursion admits a stationary distribution $W$ that stochastically dominates $W_t$ for every $t\ge1$.
\end{lemma}
Before concluding this subsection, we introduce the notion of a \emph{busy period}, which will be used later in our SEPT analysis. Under a WC policy, a \emph{busy period} is defined as the sequence of consecutive service rounds that begins when the server first serves a job after the system has become empty and ends in the first round whose service stage leaves the system empty, before the arrival at the end of that round. Let $\tau_{\rm bp}$ denote its length.

\section{Known Horizon: Upper and Lower bounds of Regret}
\label{sec:known}

We now consider the known-horizon IA setting and present the \emph{Learn-Clear-Plan} (LCP) algorithm, which achieves $R_T^{\IA}=\widetilde{O}(\sqrt{d/(\lambda T)})$. Before presenting LCP, we define the estimators used by the algorithm and provide high-probability upper bounds, and introduce the concept of \emph{Bellman policy}.

\subsection{Estimating Departure Probabilities and Their Distribution}

Index jobs by their arrival order. Let $X_i$ be the context of job $i$ and $Y_i$ the departure indicator from its first service attempt. Then $(X_i,Y_i)_{i\ge1}$ is an i.i.d.\ sequence with $Y_i\mid X_i\sim\Bern(p(X_i))$. Using the first $n$ pairs, we estimate $\theta^*$ by $\widehat\theta_n$ through empirical risk minimization, and use this estimate to construct a predictor $\widehat p_n$ of $p$. For $y\in\{0,1\}$, define $\ell_\theta(x,y):=\log(1+e^{x^\top\theta})-yx^\top\theta$, and let
\begin{equation}\label{eq:logistic-erm}
    \widehat\theta_n
    \in
    \argmin_{\norm{\theta}_2\le S}
    \frac1n\sum_{i=1}^n\ell_\theta(X_i,Y_i),
    \qquad
    \widehat p_n(x)
    :=
    \operatorname{proj}_{[p_-,p_+]}
    \!\left(\mu(x^\top\widehat\theta_n)\right).
\end{equation}
We next estimate the distribution of future jobs' departure probabilities. Using the contexts from pairs $(X_i,Y_i)_{i=n+1}^{2n}$ and $\widehat p_n$, we define the empirical distribution $\widehat F_n:=\frac1n\sum_{i=n+1}^{2n}\delta_{\widehat p_n(X_i)}$. This distribution estimates $F_{\widehat p_n}:=\Law(\widehat p_n(X))$, where $X\sim\mathcal D$ is independent of the training block.

To quantify the estimation errors, for any predictor $q:\R^d\to[p_-,p_+]$, define $\mathcal E(q):=\E_{X\sim\mathcal D}|q(X)-p(X)|$ and denote $W_1$ as the 1-Wasserstein distance. Let $\bar S:=\max\{1,S\}$. For $n\ge5$, define $\eps_n(\delta):=\min(1,\,\sqrt{({8e^{\bar S}\left(d\log(32\bar S n)+\log(1/\delta)\right)+1})/{2n}})$ and set $\eps_n(\delta)=1$ for $n<5$. We further define $\eta_n(\delta):=(p_+-p_-)\sqrt{{\log(2/\delta)}/{2n}}$.
\begin{lemma}
\label{lem:estimation-bounds}
For every $n\ge1$ and $\delta\in(0,1)$, each of the following bounds holds with probability at least $1-\delta$.
\begin{equation}\label{eq:estimation-bounds}
\begin{aligned}
    \mathcal E(\widehat p_n)
    \le \eps_n(\delta), \qquad
    W_1(\widehat F_n,F_{\widehat p_n})
    \le \eta_n(\delta).
\end{aligned}
\end{equation}
\end{lemma}

The proofs of the two inequalities are provided in \Cref{app:function-estimation-bounds,app:distribution-estimation-bounds}. We refer to them as concentration inequalities.

\subsection{Bellman Policy for the Known-Horizon IA Setting}
\label{subsec:bellman-policy}

We now explain how to characterize an optimal policy in the known-horizon IA setting through a Bellman recursion when $p$ and $\mathcal D$ are known. We first represent the current state mathematically and then present the corresponding finite-horizon Bellman recursion. We first consider a future-job success-probability distribution $F$ with finite support. Let $r_1,\ldots,r_m$ be the distinct values consisting of the support of $F$ together with the success probabilities of all currently present jobs, and let $\nu_j=F(\{r_j\})$; thus $\nu_j=0$ for a value that appears only among the currently present jobs. Represent the queue at the beginning of a round by a state $(c,a)$, where $c\in\mathbb Z_{\ge0}^m$, $c_j$ is the number of waiting jobs with success probability $r_j$, and $a\in\{0,1,\ldots,m\}$ identifies the job type in service; $a=0$ means that the server is free. The job in service is excluded from $c$ and is represented by $a$. Let $V_h(c,a)$ be the minimum expected terminal queue length with $h$ rounds remaining and starting state $(c, a)$, so that $V_0(c,a)=\sum_j c_j+\ind{a\ne0}$, the number of unfinished jobs. Writing $e_j$ for the $j$-th standard basis vector, define the arrival operator, which averages over the arrival at the end of the current round, by $(\mathcal A f)(c,a)=(1-\lambda)f(c,a)+\lambda\sum_{j=1}^m\nu_j f(c+e_j,a)$.

If $a\ne0$, nonpreemption forces the continuation of service, giving
\begin{equation}\label{eq:active-bellman}
V_h(c,a)
=r_a(\mathcal A V_{h-1})(c,0)
+(1-r_a)(\mathcal A V_{h-1})(c,a).
\end{equation}
If the server is free, it may idle or start a waiting job:
\begin{equation}\label{eq:free-bellman}
\begin{aligned}
&V_h(c,0)\\
&=\min\Bigl\{
(\mathcal A V_{h-1})(c,0),
\min_{j:c_j>0}\bigl[
r_j(\mathcal A V_{h-1})(c-e_j,0)
+(1-r_j)(\mathcal A V_{h-1})(c-e_j,j)
\bigr]\Bigr\}.
\end{aligned}
\end{equation}
Here, the first term in (\ref{eq:free-bellman}) corresponds to idling, and each term in the inner minimum corresponds to starting a type-$j$ job. The inner minimum is omitted when $c=0$. Each term is the expected optimal continuation value resulting from the corresponding action. Hence, if the first term attains the minimum, idling is optimal, while if the term indexed by $j$ reaches the minimum, starting a type-$j$ job is optimal. Repeating such minimizing choices whenever the server is free yields an optimal finite-horizon policy. For a general $F$, an analogous recursion represents the waiting jobs by a multiset of success probabilities and replaces the arrival sum by an integral with respect to $F$.

We call a policy that makes these minimizing choices at every free-server state a \emph{Bellman policy} for the corresponding model and horizon. The \emph{true Bellman policy} based on $(p, F_p)$ uses $p(x)$ as a success probability function and $F_p=\Law(p(X))$, $X\sim\mathcal D$, for future-arrivals success-probability distribution. Thus, a true Bellman policy is an optimal IA policy. In particular, starting from an empty queue with $T$ rounds remaining, it attains $V_T^{\IA}$. Suppose now that the true model is $(p,F_p)$, but the server has only an estimated success-probability function $q$ and an estimated future-job distribution $\widehat F$. The \emph{estimated Bellman policy} based on $(q,\widehat F)$ treats $(q,\widehat F)$ as the true model and follows the minimizing actions of the corresponding finite-horizon Bellman recursion.

\subsection{The Learn-Clear-Plan Algorithm}

We now introduce the \emph{Learn--Clear--Plan} (LCP) algorithm, which first estimates the unknown model, clears the queue, and then follows the estimated Bellman policy based on that estimated model.

\begin{algorithm}
\caption{Learn--Clear--Plan}
\label{alg:lcp}
\begin{algorithmic}[1]
\Require horizon $T$, planning-window length $H<T$
\State $L\gets T-H$, $n\gets\lfloor\lambda L/8\rfloor$
\State $Z\gets()$, $\textsc{Plan}\gets0$
\For{$t=1,\ldots,T$}
    \If{$t>L$, $Q_t=0$, and $\textsc{Plan}=0$}
        \If{$|Z|\ge2n$ and $n\ge1$}
            \State compute $\widehat p_n$ from the first $n$ pairs in $Z$
            \State $ \widehat F_n \gets \frac1n\sum_{i=n+1}^{2n} \delta_{\widehat p_n(X_i)} $
        \Else
            \State $\widehat p_n(x)\gets(p_-+p_+)/2$, $\widehat F_n\gets\delta_{(p_-+p_+)/2}$
        \EndIf
        \State $\textsc{Plan}\gets1$
    \EndIf
    \If{$\textsc{Plan}=1$}
        \State take an optimal action from the Bellman recursion with $T-t+1$ rounds under $(\widehat p_n,\widehat F_n)$
    \Else
        \State continue the job in service; otherwise serve FCFS, or idle if $Q_t=0$
    \EndIf
    \State observe $D_t$
    \If{job $i$ begins service in round $t$}
        \State set $Y_i\gets D_t$ and append $(X_i,Y_i)$ to $Z$
    \EndIf
    \State observe $A_t$ and, if $A_t=1$, the arriving context $X_t$
\EndFor
\end{algorithmic}
\end{algorithm}
\noindent\textbf{Learn.} Choose a planning-window length $H<T$, let $L:=T-H$, and set $n:=\lfloor\lambda L/8\rfloor$. During the first $L$ rounds, whenever the server is free, it starts the earliest-arriving waiting job; we refer to this rule as \emph{FCFS policy}, meaning that first-come, first-served. When job $i$ with context $X_i$ first begins service, the server retains only the departure indicator $Y_i$ from its first service attempt. If at least $2n$ samples are collected, the first $n$ pairs $(X_i,Y_i)$ are used to construct $\widehat p_n$, while the contexts of the next $n$ jobs are used to construct $\widehat F_n$. Separating the samples into two blocks ensures that, conditional on $\widehat p_n$, the contexts in the second block remain i.i.d.
\par\noindent\textbf{Clear.} If the queue remains nonempty after round $L$, the server continues serving jobs according to FCFS without idling until it reaches a round that begins with an empty queue. This provides an empty initial state for the subsequent planning phase.
\par\noindent\textbf{Plan.} Once the queue has been cleared, the server follows the \emph{estimated-model Bellman policy} based on $(\widehat p_n,\widehat F_n)$ for the current queue state and the remaining number of rounds. Importantly, the current queue need not consist only of jobs whose estimated success probabilities lie in the support of $\widehat F_n$. The distribution $\widehat F_n$ describes only the success probabilities of future arrivals assumed in the Bellman calculation; the jobs already present in the queue are part of the given current state. Thus, if a job with context $X$ arrives in the actual system with $\widehat p_n(X)$ outside the support of $\widehat F_n$ during the Plan phase, it is simply included in the current state as an additional type with zero future-arrival mass. This leaves $\widehat F_n$ unchanged while allowing the job to be considered among the available actions in subsequent Bellman calculations.

\subsection{Upper bound of regret}

\begin{theorem}\label{thm:known-horizon}
Under \Cref{ass:model}, let $1\le H<T$, $L=T-H$, $n=\lfloor\lambda L/8\rfloor\ge1$, and $\delta_T=T^{-4}$. Let $\pi^{\mathrm{LCP}}$ be the policy in \Cref{alg:lcp}. Then
\begin{align}
    R_T^{\IA}(\pi^{\mathrm{LCP}})
    \le{}&
    \frac{K_r}{r}\psi(r)^H
    +H\left(
        e^{-\lambda L/8}
        +K_re^{-r\lambda L/4}
        +2\delta_T
    \right)
    \notag\\
    &+
    \frac{2H^2}{p_-}\eps_n(\delta_T)
    +
    \frac{(H-1)H(H+1)(H+2)}{12}
    \eta_n(\delta_T).
    \label{eq:known-horizon-bound}
\end{align}
Here, $r, \psi(r), K_r$ are constants as given in \Cref{lem:workload}. For $H=\lceil\log^2(eT)\rceil$ and sufficiently large $T$, $R_T^{\IA}(\pi^{\mathrm{LCP}})=\widetilde{O}(\sqrt{d/(\lambda T)})$ for fixed model and traffic parameters. The total number of Bellman recursion terms evaluated over all decision epochs in the Plan phase is at most $\exp(O(\log^3 T))$.
\end{theorem}

\paragraph{Sketch of Proof.}
Consider the event that Clear occurs before the end of the horizon, at least $2n$ jobs have begun service by round $L$, and both bounds in \Cref{lem:estimation-bounds} hold. The complementary event contributes the first two terms of \Cref{eq:known-horizon-bound} by concentration inequalities and the workload tail bound in \Cref{lem:workload}. In this event, Plan starts from an empty queue with $h\le H$ rounds remaining. Let $q=\widehat p_n$ and let $\widehat\pi$ be the estimated Bellman policy based on $(q,\widehat F_n)$. For an empty-start $h$-round system, let $J_h^r(\pi)$ denote the expected terminal queue length under the success probability function $r$, and let $V_h^r$ be its optimal value among all IA policies. We decompose regret for this event as follows.
\[
J_h^p(\widehat\pi)-V_h^p
=
\bigl(J_h^p(\widehat\pi)-J_h^q(\widehat\pi)\bigr)
+\bigl(J_h^q(\widehat\pi)-V_h^q\bigr)
+\bigl(V_h^q-V_h^p\bigr).
\]
The first and third terms compare models that differ in their success-probability functions, whereas the middle term captures the suboptimality under $(q,F_q)$ of the Bellman policy computed using $(q,\widehat F_n)$. We now invoke \Cref{eq:same-context-transfer,eq:distribution-transfer} to bound these terms. The first inequality below bounds the difference in the expected terminal queue length when the same policy is used under $p$ and $q$ in terms of $\mathcal E(q)$, while the second bounds the loss from using the Bellman policy computed with $\widehat F_n$ instead of $F_q$, with $q$ fixed, in terms of $W_1(F_q,\widehat F_n)$:
\begin{equation}
\left|J_h^p(\pi)-J_h^q(\pi)\right|
\le \frac{h^2}{p_-}\mathcal E(q),
\qquad
J_h^q(\widehat\pi)-V_h^q
\le
\frac{(h-1)h(h+1)(h+2)}{12}
W_1(F_q,\widehat F_n).
\label{eq:planning-transfer-bounds}
\end{equation}
Applying the first inequality to $\widehat\pi$ and to each of the optimal policies under $p$ and $q$ bounds the first and third terms by $2H^2\eps_n(\delta_T)/p_-$, which yields the third term in \Cref{thm:known-horizon}. The second inequality, together with \Cref{lem:estimation-bounds}, bounds the middle term by the final term in \Cref{eq:known-horizon-bound}. The complete proof appears in \Cref{app:known-proof}.

\subsection{Lower Bounds of Regret}
\label{sec:lower}

We next establish a worst-case regret lower bound for every learning policy, including randomized policies. We construct a family of hard instances with dummy, baseline, and candidate jobs that share the same context distribution $\mathcal D$ and model constants $\lambda,S,p_-,p_+$, and differ only in $\theta^*$. These instances are statistically hard to distinguish and are chosen so that, if the final round begins with the server free and exactly one baseline and one candidate waiting, either job selection is suboptimal on some instance, with a departure-probability (regret) gap of $\Omega(\min\{1/\sqrt d,\sqrt{d/T}\})$.

For WC policies, we define an event $H$ that, followed by the arrival of a candidate, produces this final decision state. Stationary workload domination bounds the probability of this event away from zero, and standard testing arguments imply that, on some instance, the learner takes a suboptimal final action with sufficiently high probability, yielding the desired regret lower bound. The IA proof follows the same strategy, but requires an additional contradiction argument because the event $H$ requires the server to start a waiting job when it becomes free, while an IA policy may instead choose to idle. The complete proofs appear in \Cref{app:lower-proof}.

\begin{theorem}\label{thm:lower-wc}\label{thm:lower-ia}
For every $d\ge3$, there exist a context distribution $\mathcal D$ on $\mathbb R^d$ and constants $\lambda\in(0,1)$, $S>0$, and $0<p_-<p_+<1$ with the following property. For every $T\ge4$, every setting $\mathsf S\in\{\IA,\WC\}$, and every learning policy $\pi$ in that setting, including randomized policies, there exists a parameter vector $\theta^*$ such that $(\mathcal D,\theta^*)$ satisfies \Cref{ass:model} and
\begin{equation}
R_T^{\mathsf S}(\pi)\ge
c_{\mathsf S}\min\left(1/\sqrt d,\sqrt{d/T}\right),
\end{equation}
even if $\mathcal D$ is known to the server. The constants $c_{\IA},c_{\WC}>0$ and the model constants can be chosen independently of $d$ and $T$.
\end{theorem}

For $T\ge d^2$, the IA lower bound is $\Omega(\sqrt{d/T})$, matching the upper LCP bound up to polylogarithmic factors for the fixed model and the traffic parameters. Thus, LCP achieves the optimal worst-case regret rate up to these factors in the known-horizon IA setting.

\section{Unknown Horizon: Impossibility and SEPT Tracking}
\label{sec:anytime}

\subsection{Nonexistence of an Anytime Policy with Vanishing Regret}

A natural question in the unknown-horizon setting is whether an anytime IA or WC policy $\pi$ can satisfy $R_t^{\IA}(\pi)\to0$ (resp.\ $R_t^{\WC}(\pi)\to0$) as $t\to\infty$. However, the two counterexamples in \Cref{eq:idle-or-work,prop:wc-reversal} show that this is impossible for both settings, even when $\theta^*$ and $\mathcal D$ are known.

\begin{theorem}\label{thm:ia-wc-impossible}
For each setting $\mathsf S\in\{\IA,\WC\}$, there exists a fixed instance satisfying \Cref{ass:model}, with two contexts for IA and three for WC, such that every anytime $\mathsf S$ policy $\pi$, even with full knowledge of $\theta^*$ and $\mathcal D$, satisfies
\begin{equation}
    \limsup_{t\to\infty}R_t^{\mathsf S}(\pi)>0.
\end{equation}
\end{theorem}

We present only the proof idea for the WC setting here. The IA case is proved similarly using the queue state in \Cref{eq:idle-or-work}, for which the optimal action depends on the remaining horizon. For WC, we use the state in \Cref{prop:wc-reversal}, where the server is free with exactly one $A$ job and one $B$ job waiting. Using workload domination, we show that this state occurs at the beginning of the round $s$, for every $s\ge5$, with probability bounded away from zero. In this event, starting $B$ is uniquely optimal if the evaluation horizon is $s$, whereas starting $A$ is uniquely optimal if the evaluation horizon is $s+13$. Thus, neither choosing $A$ nor choosing $B$ is optimal for both evaluation horizons, and each suboptimal choice incurs a fixed positive regret gap. Therefore, for any randomized action distribution over $A$ and $B$, at least one of the regrets at horizons $s$ and $s+13$ is bounded below by a positive constant. This contradicts the convergence of regret to zero. The complete WC proof appears in \Cref{app:wc-impossibility}, while the complete IA proof appears in \Cref{app:ia-impossibility} within the same appendix section.

Thus, in both IA and WC settings, no anytime policy, including randomized policies, can achieve vanishing regret against the finite-horizon optimum, i.e., $R_t^{\mathsf S}(\pi)\to0$ as $t\to\infty$. Moreover, in the free-server states of these counterexamples, no action is optimal for all horizons: each choice is suboptimal for some horizon.

\subsection{SEPT policy and estimated SEPT Algorithm}

Because the optimal action at a given state can vary with the evaluation horizon, and no anytime policy achieves vanishing regret against the finite-horizon optimum, we adopt the \emph{shortest-expected-processing-time} (SEPT) policy as an anytime reference policy. SEPT is the WC policy that, whenever the server is free, starts a waiting job with the highest success probability, breaking ties by arrival order. Since we cannot follow the true SEPT policy because $p$ is unknown, our goal is to construct an anytime policy whose expected terminal queue length approaches that of true SEPT as the evaluation horizon grows, without assuming that SEPT is optimal in the present nonpreemptive setting. To this end, we introduce the \emph{estimated-SEPT} algorithm, a WC policy that updates an estimate $\widehat p$ only at the end of each busy period, keeps it fixed throughout the next busy period, and uses it in place of $p$ to apply the same priority rule as SEPT.
\begin{algorithm}
\caption{Busy-period estimated SEPT}
\label{alg:anytime}
\begin{algorithmic}[1]
\State $Z\gets()$, $\widehat p(x)\gets(p_-+p_+)/2$
\For{$t=1,2,\ldots$}
  \If{a job is in service}
    \State continue serving that job
  \ElsIf{$Q_t>0$}
    \State serve the earliest arrival in $\argmax_{j\text{ waiting}}\widehat p(X_j)$
  \Else
    \State idle
  \EndIf
  \State observe $D_t$; if job $i$ starts service, set $Y_i\gets D_t$ and append $(X_i,Y_i)$ to $Z$
  \If{the queue is empty after service and before the arrival, and $|Z|\ge1$}
    \State recompute $\widehat p$ using \Cref{eq:logistic-erm} from all pairs in $Z$
  \EndIf
  \State observe $A_t$ and, if $A_t=1$, the arriving context $X_t$
\EndFor
\end{algorithmic}
\end{algorithm}
Updating $\widehat p$ only between busy periods ensures that the predictor used within a busy period depends only on observations preceding that busy period and is therefore independent of that busy period's arriving job contexts. Let $J_t^{\rm alg}$ and $J_t^{\SEPT}$ denote the expected final queue lengths after $t$ rounds under estimated-SEPT and true SEPT, respectively, and define $\mathcal T_t^{\SEPT}:=|J_t^{\rm alg}-J_t^{\SEPT}|$. For $t\ge6$, put $u=\lfloor t/2\rfloor$, $m_t=\lfloor\lambda u/2\rfloor$, and $\delta_t=t^{-4}$. If $m_t\ge5$, set $ \bar\eps_t=\min\left(1,\sqrt{({8e^{\bar S}\left(d\log(32\bar S t)+\log(\pi^2t^2/6\delta_t)\right)+1})/{2m_t}}\right), b_t=e^{-\lambda u/8}+\delta_t$.
\begin{theorem}\label{thm:anytime}
Under \Cref{ass:model}, if $m_t\ge5$, then \Cref{alg:anytime} satisfies
\begin{equation}
 \mathcal T_t^{\SEPT}
 \le \Phi_r(\bar\eps_t)
 +\frac{M_G(r)}{r(1-\psi(r))}\psi(r)^u
 +\frac{\sqrt{K_rb_t}}r.
 \label{eq:anytime-bound}
\end{equation}
Here, $\Phi_r(\eps)$ is the function defined on $\eps\in(0,1]$ and satisfies $\Phi_r(\eps)=O(\eps(1+\log^3(1/\eps)))$, and $r, \psi(r), M_G(r), K_r$ are constants as given in \Cref{lem:workload}. Consequently, for fixed model and traffic parameters, $\mathcal T_t^{\SEPT} =\wtO(\sqrt{d/(\lambda t)})$. If $m_t<5$, the uniform bound $\mathcal T_t^{\SEPT}\le C_W(r)$ holds.
\end{theorem}
The proof uses ideas similar to those in the LCP analysis. Under common arrivals, contexts, and service times, estimated-SEPT and true SEPT have identical workload processes and busy period boundaries, so only the last busy period can affect the terminal queue-length difference. We decompose the tracking error according to its starting round. For recent busy periods with small prediction error, \Cref{lem:busy-period-perturbation} together with the workload tail bound gives the first term in \Cref{eq:anytime-bound}; old busy periods are controlled directly by the workload tail bound, giving the second. Finally, \Cref{lem:busy-period-confidence} controls the contribution of recent busy periods with large prediction error by the third term. The complete proof appears in \Cref{app:anytime-proof}.

\section{Conclusion and Future Work}

We studied contextual queueing bandits under nonpreemptive service. For a known horizon, our LCP algorithm achieves $\wtO(\sqrt{d/(\lambda T)})$ regret against the IA optimum, matching our lower bound up to logarithmic factors for $T\ge d^2$. For an unknown horizon, we show that no anytime policy achieves vanishing regret against the finite-horizon IA or WC optimum, and that our estimated-SEPT algorithm tracks true SEPT at rate $\wtO(\sqrt{d/(\lambda t)})$.

Extending these results to multiple servers is challenging because the workload contributed by an arriving job depends on which server will eventually process it in the future. Thus, unlike in the single-server setting, the workload is policy-dependent, and our common-workload and common-empty-time coupling arguments no longer apply directly. Nevertheless, the techniques developed in this work, together with the multi-server analysis of \citet{bae2026algorithm}, suggest that under appropriate assumptions it may be possible to obtain upper and lower regret bounds of a similar form in the multi-server nonpreemptive setting. Establishing such results is an interesting direction for future work.

\bibliographystyle{plainnat}
\bibliography{ref}

\clearpage
\appendix

\crefalias{section}{appendix}
\crefalias{subsection}{appendix}
\crefalias{subsubsection}{appendix}

\section{Appendix Overview}

The appendix is organized as follows.
\begin{enumerate}
\item \Cref{app:related-work} provides additional discussion of related work.
\item \Cref{app:workload} proves the workload bounds in \Cref{lem:workload} and establishes additional workload and busy-period properties used throughout the subsequent analyses.
\item \Cref{app:statistics} proves the departure-probability and distribution estimation bounds in \Cref{lem:estimation-bounds}, together with a simultaneous prediction bound.
\item \Cref{app:known-proof} gives the complete proof of the known-horizon upper bound in \Cref{thm:known-horizon}, together with the supporting lemmas used in the analysis.
\item \Cref{app:lower-proof} constructs the hard family of instances and proves the regret lower bounds in \Cref{thm:lower-wc} for both the WC and IA settings.
\item \Cref{app:benchmark-proofs} proves \Cref{thm:ia-wc-impossible}, namely, that for both $\mathsf S\in\{\IA,\WC\}$, there is no anytime $\mathsf S$ policy $\pi$ such that $R_t^{\mathsf S}(\pi)\to0$ as $t\to\infty$, even when the model is fully known.
\item Finally, \Cref{app:anytime-proof} gives the complete proof of the tracking guarantee for the estimated-SEPT algorithm in \Cref{thm:anytime}, together with the supporting lemmas used in the analysis.
\end{enumerate}

\section{Additional Related Work}
\label{app:related-work}

\paragraph{Queueing bandits and learning while scheduling.}
Queueing bandits study learning-while-scheduling problems in which unknown service rates must be learned while controlling queue lengths \citep{krishnasamy2016regret,krishnasamy2021learning}. Queue-length regret has subsequently been studied under several queueing architectures and information structures \citep{stahlbuhk2021learning,liang2018minimizing}. Related work considers job dispatching with unknown service rates \citep{choudhury2021job}, decentralized learning in queueing systems \citep{sentenac2021decentralized,freund2022efficient}, integrated online learning and queue control \citep{hsu2022integrated}, and learning-enhanced MaxWeight-type scheduling in multi-server systems \citep{yang2023learning}. Other work combines Lyapunov-drift scheduling with adversarial bandit feedback \citep{huang2024lyapunov}. Most of these models do not simultaneously incorporate job-specific contextual service probabilities and nonpreemptive service commitments, which are the two central features of our setting.

\paragraph{Contextual queueing and logistic bandits.}
Context-aware queueing models allow service outcomes to depend on job or queue features. \citet{kim2024queueing} study a multi-class multi-server queueing-matching problem with feature-dependent multinomial-logit service preferences. The closest models to ours are the contextual queueing bandits of \citet{bae2026queue,bae2026algorithm}. In the stochastic-context setting most directly comparable to ours, jobs arrive with random contexts, and in each round the scheduler selects a job--server pair whose departure probability is determined by a logistic model. The objective is to minimize the expected queue length at a terminal horizon, or equivalently to control queue-length regret relative to a known-model oracle. Their setting is preemptive: after an unsuccessful service attempt, the scheduler may select a different job--server pair in the next round. Consequently, when the service model is known, an optimal policy is work-conserving and selects, in each nonempty round, an available job--server pair with the largest departure probability.

Our nonpreemptive model differs in that, once a job--server pair is selected, the same pair must be served until the job departs. Hence the current decision constrains future actions, and the finite-horizon optimal policy need not admit the same myopic characterization. In addition, the stochastic-context guarantees of \citet{bae2026queue,bae2026algorithm} use a pointwise traffic-slackness condition and a positive lower bound on a feature second-moment matrix, whereas our main single-server analysis uses the workload condition $\lambda\E[1/p(X)]<1$ and does not require such a covariance lower bound.

Our logistic service model is also related to generalized linear and logistic bandits \citep{filippi2010parametric,li2017provably,faury2020improved, abeille2021instance}.

\paragraph{Nonpreemptive stochastic scheduling and learning.}
Nonpreemptive scheduling has been extensively studied when service-time distributions are known. For fixed collections of jobs, SEPT and related priority rules are optimal under suitable stochastic-order assumptions for classical objectives such as expected flow time \citep{weber1986scheduling}, while other priority rules can maximize the number of successful completions before a due date under additional structural conditions \citep{righter1989scheduling}. More recently, learning has been incorporated into stochastic scheduling. \citet{merlis2023preemption} study single-machine scheduling when job-type duration distributions are initially unknown, considering both nonpreemptive and preemptive settings and comparing their learning costs. \citet{zhong2025learning} study scheduling with unknown model primitives in multiclass many-server queues with abandonment.

Our setting differs from these works in several respects. Jobs arrive stochastically over time and carry individual contextual features, their service-completion probabilities are governed by an unknown logistic model, and the objective is the expected queue length at a finite terminal horizon. Moreover, under nonpreemption, each selected job occupies the server until completion, so the effect of a scheduling decision persists across multiple rounds and interacts with future arrivals. Consequently, neither the classical SEPT optimality results nor existing learning-to-schedule guarantees directly characterize the finite-horizon benchmark studied here.

\section{Workload properties and stationary workload}
\label{app:workload}

This section proves the properties stated in \Cref{lem:workload} and establishes several additional properties of the stationary workload associated with the workload recursion. These results will be used later as key ingredients in the regret lower-bound analyses.

We begin by defining some constants introduced in \Cref{lem:workload}. For $r>0$ such that $(1-p_-)e^r<1$, define
\[
    M_G(r)
    :=
    \E e^{rG(X)}
    =
    \E\frac{p(X)e^r}{1-(1-p(X))e^r},
    \qquad
    \psi(r)
    :=
    e^{-r}\bigl(1-\lambda+\lambda M_G(r)\bigr).
\]
Since $\psi(0)=1$ and $\psi'(0)=\rho-1<0$, there exists $r>0$ satisfying both $\psi(r)<1$ and $(1-p_-)e^r<1$. Fix any such $r$ and define
\[
K_r:=\frac{1-\lambda+\lambda M_G(r)}{1-\psi(r)}, \qquad
C_W(r):=\frac{1}{r}\log K_r.
\]
These five constants: $r\in (0, 1), \psi(r)\in (0, 1), M_G(r), K_r$ and $C_W(r)$ are what we introduced in \Cref{lem:workload}.

For convenience, we divide the conclusions of \Cref{lem:workload} into three items for the purposes of this appendix. We refer to the bounds $\E e^{rW_t}\le K_r$ and $\E W_t\le C_W(r)$ for every $t\ge1$ as \emph{Item 1: uniformly bounded expected workload}. We refer to the conditional workload bound in \Cref{eq:workload-tail}, together with its unconditional consequence $\E[W_{u+h}\ind{\tau_0>h}]\le(K_r/r)\psi(r)^h$, as \emph{Item 2: workload tail bound}. Finally, we refer to the existence of a stationary workload distribution $W$ that stochastically dominates $W_t$ for every $t\ge1$ as \emph{Item 3: stationary workload domination}.

\subsection{\texorpdfstring{Proofs of Items 1 and 2 of \Cref{lem:workload}}{Proofs of Items 1 and 2 of the workload lemma}}
\label{app:workload-lemma-1-2}

\begin{proof}
\noindent\emph{Item 1: Uniformly bounded expected workload.}\\
Let \(Z_t=A_tG_t\) be the workload brought by the arrival at the end of round \(t\), with \(Z_t=0\) when \(A_t=0\). Conditional on \(A_t=1\) and \(X_t\), we have \(Z_t=G_t\sim\operatorname{Geom}(p(X_t))\). Since the arrival indicator is independent of the context and the latent service requirement,
\[
 B_r:=\E e^{rZ_t}=1-\lambda+\lambda\E_{X\sim D} e^{rG(X)}
 =1-\lambda+\lambda M_G(r).
\]
Here, $G(X)$ is a random variable that satisfies $G(X)\sim\operatorname{Geom}(p(X))$.\\
If $W_t=w\ge1$, the workload is reduced by one during round $t$ and increased by $Z_t$ at the end of the round. Hence
\begin{align*}
 \E(e^{rW_{t+1}}\mid W_t=w)
 &=\E e^{r(w-1+Z_t)}
 =e^{r(w-1)}B_r
 =e^{rw}\psi(r).
\end{align*}
If $W_t=0$, no work is served and $W_{t+1}=Z_t$. Splitting according to whether $W_t$ is zero therefore gives
\begin{align*}
 \E e^{rW_{t+1}}
 &=\E\left[e^{rW_t}e^{-r}\E (e^{rZ_t}\mid W_t)\ind{W_t>0}\right]+\E\left[\E (e^{rZ_t}\mid W_t)\ind{W_t=0}\right]\\
 &=\psi(r)\E\left(e^{rW_t}\ind{W_t>0}\right)+B_r\P(W_t=0)\\
 &\le \psi(r)\E e^{rW_t}+B_r.
\end{align*}
Starting from $W_1=0$ and applying this inequality successively,
\begin{align*}
 \E e^{rW_t}
 &\le \psi(r)^{t-1}\E e^{rW_1}
 +B_r\sum_{j=0}^{t-2}\psi(r)^j\\
 &=\psi(r)^{t-1}+B_r\frac{1-\psi(r)^{t-1}}{1-\psi(r)}\\
 &=K_r+\psi(r)^{t-1}(1-K_r)\le K_r.
\end{align*}
Here we used $K_r=B_r/(1-\psi(r))$, and the last inequality holds because $B_r\ge1$ and $0<\psi(r)<1$, so $K_r=B_r/(1-\psi(r))\ge1$. Using Jensen's inequality for the function $x\mapsto e^{rx}$ then yields
\[
 \E W_t\le\frac1r\log\E e^{rW_t}
 \le\frac1r\log K_r=C_W(r).
\]
\noindent\emph{Item 2: Workload tail bound.}\\
Fix a positive integer $u$ and let $\tau_0=\inf\{h\ge0:W_{u+h}=0\}$. On $\{\tau_0>h\}$ the workload is positive through the first $h$ updates. Since the set of events $\{\tau_0>h\}$ is equal to $\{\tau_0>h-1\}\cap \{W_{u+h}>0\}$, we can get the following recurrence inequality.
\begin{align*}
 &\E\left(e^{rW_{u+h}}\ind{\tau_0>h}\mid W_u\right)\\
 &\quad=\E\left[
  \ind{\tau_0>h-1}
  \E\left(e^{rW_{u+h}}\ind{W_{u+h}>0}\mid W_{u+h-1}\right)
  \middle|W_u\right]\\
 &\quad\le\E\left[\E\left(e^{rW_{u+h}}\mid W_{u+h-1}\right)\ind{\tau_0>h-1}\middle| W_u\right]\\
 &\quad\le\psi(r)\E\left(e^{rW_{u+h-1}}\ind{\tau_0>h-1}\mid W_u\right)\\
\end{align*}
Applying this inequality successively, we get
\begin{align*}
 &\E\left(e^{rW_{u+h}}\ind{\tau_0>h}\mid W_u\right)
 \le \E\left(e^{rW_{u}}\ind{\tau_0>0}\mid W_u\right)\psi(r)^h\leq e^{rW_u}\psi(r)^h. 
\end{align*}

Because $w\le e^{rw}/r$ for $w\ge0$, we obtain
\[
 \E\left(W_{u+h}\ind{\tau_0>h}\mid W_u\right)
 \le\frac{e^{rW_u}}r\psi(r)^h,
\]
which is \Cref{eq:workload-tail}. Finally, by taking the expectation of $W_u$, we obtain a useful inequality like below.
\begin{equation}
 \E \left(W_{u+h}\ind{\tau_0>h}\right)=\E \left[\E\left(W_{u+h}\ind{\tau_0>h}\mid W_u\right)\right]
 \le\frac{\E (e^{rW_u})}r\psi(r)^h \le \frac{K_r}r\psi(r)^h
 \label{eq:W(u+h) upper bound}
\end{equation}
\end{proof}

\subsection{Upper bound of long busy period probability}
\label{app: busy-period-upper-bound}

We now show that the probability that a busy period lasts for more than $h$ rounds is upper bounded by a function that decays exponentially in $h$. This requires a slightly different stopped workload because an arrival in the round that empties the system right before that arrival belongs to the next busy period. Suppose a busy period starts in round $s$ (so in round $s-1$, the queue becomes empty after arrival and a new job arrives immediately later), and let $G$ be the service time of its first job (which arrived at the end of round $s-1$). Set $\widetilde W_0=G$ and, for $h\ge0$, define
\[
 \widetilde W_{h+1}=
 \begin{cases}
  \widetilde W_h-1+Z_{s+h},&\widetilde W_h\ge2,\\
  0,&\widetilde W_h\le1.
 \end{cases}
\]
Thus $\widetilde W_h$ is the amount of work belonging to this busy period after $h$ service rounds (in other words, at the beginning of round $s+h$), and it is easy to check that $\widetilde W_h$ remains zero after it hits zero for the first time. So if we define $\tau_{\rm bp}$ as the smallest $h$ such that the queue $\widetilde W_h$ hits zero, then $\{\tau_{\rm bp}>h\}=\{\widetilde W_h>0\}$ holds. In particular, when $\widetilde W_h=1$, the last unit is served in round $s+h$ and the process is stopped just before the arrival at the end of that round is added.

If $\widetilde W_h=w\ge2$, the same calculation as above gives
\[
 \E\left(e^{r\widetilde W_{h+1}}\ind{\widetilde W_{h+1}>0}\mid\widetilde W_h=w\right)
 =e^{rw}\psi(r).
\]
If $w\le1$, the left-hand side is zero. Using the same inductive argument that we used to prove \Cref{eq:workload-tail}, we get
\begin{equation}
 \E\left(e^{r\widetilde W_h}\ind{\tau_{\rm bp}>h}\mid G\right)
 \le e^{rG}\psi(r)^h.
 \label{eq:busy-period-workload-tail}
\end{equation}
Here, we used $W_0=G$. Since $e^{r\widetilde W_h}\ge1$ on $\{\tau_{\rm bp}>h\}$,
\[
 \P(\tau_{\rm bp}>h\mid G)
 \le e^{rG}\psi(r)^h.
\]
Averaging over $G$ and using $\E e^{rG}=M_G(r)$ proves the desired result.

\subsection{\texorpdfstring{Proof of Item 3 of \Cref{lem:workload} and Stationary Workload Identities} {Proof of Item 3 and stationary workload identities}}
\label{app:workload-lemma-3}

We now prove the existence of a stationary workload, which we refer to above as Item 3 of \Cref{lem:workload}, using the assumption $\rho<1$. We then establish several properties of the stationary workload. All of these results are summarized in the following lemma.

\begin{lemma}\label{lem:stationary-workload}
The recursion in \Cref{eq: workload recursion} admits a stationary distribution. Let $W$ denote a random variable with this stationary distribution, and call it the stationary workload. Then the process started from zero and satisfying \Cref{eq: workload recursion} is stochastically dominated by it. Let $A\sim\Bern(\lambda)$ and, independently of $A$, let $X\sim\mathcal D$ and $G\mid X\sim\Geom(p(X))$. With $Z=AG$,
\begin{equation}
 \P(W>0)=\rho,
 \qquad
 \E W=\frac{\E Z^2+\rho-2\rho^2}{2(1-\rho)}.
 \label{eq:stationary-workload}
\end{equation}
\end{lemma}

\begin{proof}
For every integer $w\ge1$, workload satisfies the following negative-drift inequality.
\begin{align*}
    \E[W_{t+1}-W_t\mid W_t=w]
    &= \E[A_tG(X_t)]-1 \\
    &=\lambda\E_{X\sim \mathcal{D}, G(X)\sim \Geom((p(X)))}[G(X)]-1 \\
    &=\lambda \E_{X\sim \mathcal{D}}[1/p(X)]-1=\rho-1<0.
\end{align*}
From every finite state, a sufficiently long run without arrivals reaches zero. The countable-state Foster--Lyapunov criterion therefore gives positive recurrence. The exponential drift in the proof of \Cref{lem:workload} gives a finite stationary exponential moment. Because $w\mapsto(w-1)^++Z_t$ is nondecreasing, coupling a zero-started process and a stationary process with the same increments gives stochastic domination.

For the stationary workload $W$, let $W'=(W-1)^++Z$, where $Z$ is independent of $W$. Since $W'\stackrel{d}=W$ by the definition of the stationary workload, taking expectations of $W, W'$ gives
\[
 0=\E(W'-W)=-\P(W>0)+\E Z,
\]
and hence $\P(W>0)=\E Z=\rho$. Moreover, since $(W-1)^+=W-\ind{W\ge1}$ and $((W-1)^+)^2=W^2-2W+\ind{W\ge1}$,
\[
 \E(W-1)^+=\E W-\rho,
 \qquad
 \E((W-1)^+)^2=\E W^2-2\E W+\rho.
\]
Squaring $W'=(W-1)^++Z$, taking expectations, and using the independence of $W$ and $Z$ gives
\begin{align*}
 \E W^2
 &=\E((W-1)^+)^2+2\E((W-1)^+)\E Z+\E Z^2\\
 &=\E W^2-2\E W+\rho+2\rho(\E W-\rho)+\E Z^2.
\end{align*}
Canceling $\E W^2$ from both sides and collecting the terms in $\E W$ gives
\[
 2(1-\rho)\E W=\E Z^2+\rho-2\rho^2.
\]
Division by $2(1-\rho)>0$ proves \Cref{eq:stationary-workload}.
\end{proof}

\section{Estimating Departure Probabilities}
\label{app:statistics}

This section proves the two bounds in \Cref{lem:estimation-bounds}. These inequalities will serve as key ingredients in the regret analysis.

\subsection{\texorpdfstring{Proof of first item in \Cref{lem:estimation-bounds}}{Proof of the population-prediction lemma}}
\label{app:function-estimation-bounds}

\begin{proof}
Let $z=x^\top\theta$. The gradient and Hessian of the logistic loss are
\[
 \nabla\ell_\theta(x,y)=(\mu(z)-y)x,
 \qquad
 \nabla^2\ell_\theta(x,y)=\mu'(z)xx^\top.
\]

Because $|\mu(z)-y|\le1$ and $\|x\|_2\le1$, $\|\nabla\ell_\theta(x,y)\|_2\le1$. Thus the loss is $1$-Lipschitz in $\theta$. The parameter ball has diameter at most $2S\le2\bar S$, so the mean-value theorem gives
\[
 |\ell_\theta(x,y)-\ell_{\theta'}(x,y)|\le2\bar S
 \qquad
 \text{for all }\|\theta\|_2,\|\theta'\|_2\le S.
\]

We next verify exp-concavity. Since $|z|\le S\le\bar S$, if $y=0$, then
\[
 \frac{\mu'(z)}{\mu(z)^2}=\frac{1-\mu(z)}{\mu(z)}=e^{-z}\ge e^{-\bar S};
\]
if $y=1$, then
\[
 \frac{\mu'(z)}{(1-\mu(z))^2}=\frac{\mu(z)}{1-\mu(z)}=e^z\ge e^{-\bar S}.
\]
In either case,
\[
 \nabla^2\ell_\theta(x,y)
 \succeq e^{-\bar S}\nabla\ell_\theta(x,y)\nabla\ell_\theta(x,y)^\top,
\]
which is the second-order characterization of $e^{-\bar S}$-exp-concavity on the convex parameter ball.

Before applying the excess-risk theorem, observe that $\theta^*$ minimizes the population logistic risk over the parameter ball. Indeed, for every $\theta$ in this ball, since $Y\sim \Bern(p(X))$ conditional on the context $X$, the following equality about KL divergence holds.
\begin{equation}
\begin{aligned}
 \E\ell_\theta(X,Y)-\E\ell_{\theta^*}(X,Y)
 &= \left(\E_{X} \left[\log\left(1+e^{X^\top \theta}\right)\right]-\E\left[p(X)X^\top \theta\right]\right) \\
 &\;\;\quad-\left(\E_{X} \left[\log\left(1+e^{X^\top \theta^*}\right)\right]-\E\left[p(X)X^\top \theta^*\right]\right) \\
 &=\E_{X}\left[\log\left(\frac{1+e^{X^\top \theta}}{1+e^{X^\top \theta^
 *}}\right)-p(X)\log\left(\frac{e^{X^\top \theta}}{e^{X^\top \theta^*}}\right)\right] \\
 &=\E_{X}\left[p(X)\log\left(\frac{1+e^{-X^\top \theta}}{1+e^{-X^\top \theta^
 *}}\right)+(1-p(X))\log\left(\frac{1+e^{X^\top \theta}}{1+e^{X^\top \theta^
 *}}\right)\right] \\
 &=\E_{X}\left[p(X)\log\left(\frac{p(X)}{\mu(X^\top \theta)}\right)+(1-p(X))\log\left(\frac{1-p(X)}{1-\mu(X^\top \theta)}\right)\right] \\
&=\E_X\KL\left(\Bern(p(X))\,\middle\|\,\Bern(\mu(X^\top\theta))\right)\ge0
 \label{eq:logistic-risk-kl}
\end{aligned}
\end{equation}

Since we estimate $\theta^*$ by $\widehat \theta_n$ by using empirical risk minimization, we can apply Theorem~1 of \cite{mehta2017fast}. In the notation of that theorem, the feasible set is a $d$-dimensional closed ball of radius $S$ to which the unknown vector $\theta^*$ belongs, so its diameter is $2S\leq 2\bar S$. For every $(x, y)$, the loss function $\ell_\theta(x,y)$ viewed as a function of the variable $\theta$ has Lipschitz constant $1$ as shown above. Its exp-concavity parameter is $e^{-\bar S}$, and for every $(x, y), \theta, \theta^*$, $|\ell_\theta(x,y)-\ell_{\theta^*}(x,y)|\leq |\theta-\theta^*|\leq 2\bar S$ holds because the loss function has Lipschitz constant $1$. The theorem contains the factor $2\bar S\vee e^{\bar S}$, which equals $e^{\bar S}$ because $e^{\bar S}\ge2\bar S$ for $\bar S\ge1$. Its logarithmic term becomes $\log(16\cdot1\cdot2\bar S\cdot n)=\log(32\bar S n)$. Consequently, for every $n\ge5$, with probability at least $1-\delta$,
\begin{equation}
\begin{aligned}
 \E\ell_{\widehat\theta_n}(X,Y)-\E\ell_{\theta^*}(X,Y)\le\frac{8e^{\bar S}\left(d\log(32\bar S n)+\log(1/\delta)\right)+1}{n}.
\end{aligned}
 \label{eq:exp-concave-excess}
\end{equation}
By \Cref{eq:logistic-risk-kl}, the left-hand side of \Cref{eq:exp-concave-excess} equals the expected Bernoulli KL divergence over all $X\sim \mathcal{D}$. Pinsker's inequality gives $|p-q|\le\sqrt{\KL(\Bern(p)\|\Bern(q))/2}$. Applying Pinsker's inequality pointwise for each $X$ and then using the Cauchy-Schwarz inequality, we get the following inequality.
\begin{align*}
 \E_X|p(X)-\mu(X^\top\widehat\theta_n)|
 &\le \E_X\sqrt{\frac12\KL\left(\Bern(p(X))\,\middle\|\,\Bern(\mu(X^\top\widehat\theta_n))\right)}\\
 &\le \sqrt{\E_X\left[\frac12\KL\left(\Bern(p(X))\,\middle\|\,\Bern(\mu(X^\top\widehat\theta_n))\right)\right]\cdot \E_X[1]} \\
&=\sqrt{\frac{1}{2}\E_X\KL\left(\Bern(p(X))\,\middle\|\,\Bern(\mu(X^\top\widehat\theta_n))\right)} \\
 &= \sqrt{\frac12\left(\E\ell_{\widehat\theta_n}(X,Y)-\E\ell_{\theta^*}(X,Y)\right)}
\end{align*}
Projection onto $[p_-,p_+]$ cannot increase the absolute error because $p(X)$ already lies in this interval. So $|p(X)-\widehat{p}(X)|\leq |p(X)-\mu(X^\top \widehat\theta_n)|$ always holds by the definition of $\widehat{p}(X)$, and this gives $\E_X|p(X)-\widehat{p}(X)|\leq  \sqrt{\frac12\left(\E\ell_{\widehat\theta_n}(X,Y)-\E\ell_{\theta^*}(X,Y)\right)}$ by the above inequality. Substituting \Cref{eq:exp-concave-excess} proves desired result. For $n<5$, the definition gives $\eps_n(\delta)=1$, while $|\widehat p_n(X)-p(X)|\le1$. So this shows first item of \Cref{lem:estimation-bounds} holds for every $n\geq 1$.
\end{proof}

\subsection{A simultaneous prefix bound}

\begin{corollary}\label{cor:simultaneous-prediction}
For $\delta\in(0,1)$, define $\delta_n=6\delta/(\pi^2n^2)$. With probability at least $1-\delta$, simultaneously for every $n\ge1$,
\[
 \mathcal E(\widehat p_n)\le\eps_n(\delta_n).
\]
\end{corollary}

\begin{proof}
For each fixed $n$, $\mathcal E(\widehat p_n)\le\eps_n(\delta_n)$ fails with probability at most $\delta_n$ by first item of \Cref{lem:estimation-bounds}. So the probability that there exists a natural number such that $\mathcal E(\widehat p_n)\le\eps_n(\delta_n)$ d $\sum_{n\ge1}\delta_n=\delta$, a union bound proves the result. No stopping-time conditioning is used: all estimators are defined on the same infinite i.i.d. sequence of job-indexed observations $(X_i,Y_i)_{i\ge1}$.
\end{proof}

\subsection{\texorpdfstring{Proof of second item in \Cref{lem:estimation-bounds}}{Proof of the Wasserstein distance lemma}}
\label{app:distribution-estimation-bounds}

Conditional on $\widehat p_n$, the variables $\widehat p_n(X'_1),\ldots,\widehat p_n(X'_n)$ are i.i.d. with law $F_{\widehat p_n}$ and support in $[p_-,p_+]$. The Dvoretzky--Kiefer--Wolfowitz inequality \citep{dvoretzky1956asymptotic} gives, with probability at least $1-\delta$,
\[
 \sup_z|\widehat F_n((\! -\infty,z])-F_{\widehat p_n}((\! -\infty,z])|
 \le\sqrt{\frac{\log(2/\delta)}{2n}}.
\]
For probability distributions supported on $[p_-,p_+]$, the one-dimensional identity for $W_1$ proves second item of \Cref{lem:estimation-bounds} as follows.
\begin{align*}
 W_1(\widehat F_n,F_{\widehat p_n})
 &=\int_{p_-}^{p_+}|\widehat F_n((\! -\infty,z])-F_{\widehat p_n}((\! -\infty,z])|\,dz\\
 &\le(p_+-p_-)\sqrt{\frac{\log(2/\delta)}{2n}},
\end{align*}
Furthermore, couple $\widehat p_n(X)$ and $p(X)$ using the same independent $X\sim\mathcal D$ and applying the definition of Wasserstein distance, we get following inequality.
\[
 W_1(F_{\widehat p_n},F_p)=
\inf_{\substack{(P,\widehat P)\ \text{coupling}\\
P\sim F,\ \widehat P\sim\widehat F}}
\mathbb E\bigl[|P-\widehat P|\bigr]\le\E|p(X)-\widehat p_n(X)|=\mathcal E(\widehat p_n).
\]

\section{Known-Horizon Analysis}
\label{app:known-proof}

We now prove \Cref{thm:known-horizon}. The proof uses three ingredients; \Cref{lem:planning-transfer}, \Cref{lem:sample-availability}, and \Cref{lem:empty-horizon-monotonicity}. The first is used to compare the empirical Bellman model with the true system. The second ensures that the learning phase can collect a sufficient amount of data to make an empirical distribution and an estimated probability function, with high probability. The third one compares the minimum expected queue length of a queue that starts with fewer remaining rounds to the minimum expected queue length of the $T$-round benchmark, and this allows us to calculate the regret upper bound correctly.

The first ingredient is the planning-transfer lemma. If the system assumes its unknown probability function and distribution to be $q:\mathbb{R}^d\to[p_-,p_+]$ and $\widehat{F}$ and uses the Bellman recursion that follows $(q, \widehat{F})$ (for convenience of exposition, we write this Bellman recursion as the empirical Bellman recursion) to find out the optimal decision in each round, this empirical Bellman recursion differs from the true optimal system (the system that makes its decision according to the true Bellman recursion) in two ways. First, for the probability function, it replaces $p(X)$ by $q(X)$ for every job. Second, for the distribution, it uses the distribution $\widehat{F}$ which is an estimated distribution of $F_q=\operatorname{Law}(q(X))$, although the real distribution of departure probabilities of jobs is $F_p=\operatorname{Law}(p(X))$. By focusing on these probability/distribution differences, we can find an upper bound of regret between a queue that follows the empirical Bellman recursion and a queue that follows the true Bellman recursion (i.e. a queue with the smallest expected queue length). Here, we upper bound the prediction error of $p, q$ by the $L^1(D)$ distance of $p, q$, and the distributional error by the Wasserstein distance between $F_q, \widehat{F}$.

For an $h$-round system that follows the probability function $q$ and starts empty, let $J_h^p(\pi)$ be the expected terminal queue under policy $\pi$, and let $V_h^p$ be its minimum over the relevant policy class in the given setting. Then the following lemma holds for the estimated probability function $q$ and empirical distribution $\widehat{F}$.

\medskip
\begin{lemma}[Planning transfer]\label{lem:planning-transfer}
Fix \(q:\mathbb{R}^d\to[p_-,p_+]\) before an \(h\)-round process starts with an empty queue. Let \(F_q=\operatorname{Law}(q(X))\), let \(\widehat F\) be a distribution on \([p_-,p_+]\), and let \(\widehat\pi\) be an optimal policy following the Bellman recursion defined by \((q,\widehat F)\). When \(\widehat\pi\) is run in the true contextual model,
\begin{equation}
    J_h^p(\widehat\pi)-V_h^p
    \le
    \frac{2h^2}{p_-}\mathbb{E}|q(X)-p(X)|
    +
    \frac{(h-1)h(h+1)(h+2)}{12}
    W_1(F_q,\widehat F).
    \label{eq:planning-transfer}
\end{equation}
The result holds both in the idling-allowed setting and the work-conserving setting.
\end{lemma}

The second ingredient ensures that the FCFS learning phase has observed a complete prefix of jobs in arrival order. This matters because the statistical events are defined on the job-indexed sequence \((X_i,Y_i)_{i\ge1}\), independently of the random time at which the algorithm happens to have served enough jobs.

\medskip
\begin{lemma}[Sample availability]\label{lem:sample-availability}
Let \(N_L\) be the number of jobs that have started service by the end of round \(L\), and let \(n=\lfloor\lambda L/8\rfloor\). Under FCFS,
\begin{equation}
    \mathbb{P}(N_L<2n)
    \le
    e^{-\lambda L/8}+K_re^{-r\lambda L/4}.
    \label{eq:chernoff bound on N_L}
\end{equation}
On \(\{N_L\ge2n\}\), the first \(2n\) observations \((X_i,Y_i)\) recorded by the algorithm correspond to the first \(2n\) jobs in arrival order.
\end{lemma}

The final ingredient is a monotonicity property of the finite-horizon empty-state value.

\medskip
\begin{lemma}[Empty-state horizon monotonicity]\label{lem:empty-horizon-monotonicity}
The minimum terminal cost over policies that may idle is nondecreasing when jobs are added to the initial state. In particular, for systems that start empty,
\begin{equation}
    V_h^{\mathrm{IA}}\le V_{h+1}^{\mathrm{IA}}.
    \label{eq: monotonicity of V_h}
\end{equation}
\end{lemma}

The proofs of the above three lemmas will appear after the proof of \Cref{thm:known-horizon}.

\subsection{\texorpdfstring{Proof of \Cref{thm:known-horizon}}{Proof of the known-horizon theorem}}

Now, we start to prove \Cref{thm:known-horizon}.

\begin{proof}
Index the jobs in arrival order and let $(X_i,Y_i)_{i\ge1}$ denote their contexts and first-round departure indicators. This sequence is i.i.d. and does not depend on when the jobs begin service. Define
\begin{align}
 \mathcal G_1&=\{N_L\ge2n\},\\
 \mathcal G_2&=\{\mathcal E(\widehat p_n)\le\eps_n(\delta_T)\},\\
 \mathcal G_3&=\{W_1(\widehat F_n,F_{\widehat p_n})\le\eta_n(\delta_T)\}.
\end{align}
The estimator in $\mathcal G_2$ uses the first $n$ observations, whereas $\mathcal G_3$ uses the estimated probability function and contexts of jobs $n+1,\ldots,2n$. Neither event is conditioned on $\mathcal G_1$. Hence, \Cref{lem:sample-availability,lem:estimation-bounds} and a union bound give
\begin{equation}
 \P((\mathcal G_1\cap\mathcal G_2\cap\mathcal G_3)^c)
 \le \P(\mathcal G_1^c)+\P(\mathcal G_2^c)+\P(\mathcal G_3^c)
 \le e^{-\lambda L/8}+K_re^{-r\lambda L/4}+2\delta_T.
 \label{eq:known-good-event}
\end{equation}

Let $\mathcal G=\mathcal G_1\cap\mathcal G_2\cap\mathcal G_3$ and define the first empty round in the planning window by
\[
 \tau=\min\{s\in\{L+1,\ldots,T\}:Q_s=0\},
\]
with $\tau=\infty$ when the set is empty (i.e. after starting round $L$, the first round in which the queue hits zero at the end of that round is round $T$, or the queue never becomes empty at the end of the remaining rounds). On $\{\tau\le T\}$, define $h=T-\tau+1\le H$. Conditional on the history at the beginning of round $\tau$, the queue is empty. Thus, on $\mathcal{G}\cap \{\tau\le T\}$, the arrivals, contexts, and departure indicators from round $\tau$ onward are independent of the first $2n$ job-indexed observations used for estimation, since those $2n$ jobs have already completed before the start of round $\tau$.

On $\mathcal G\cap\{\tau\le T\}$, set $q=\widehat p_n$. After the queue hits zero at the beginning of round $\tau$, the final queue length $Q_{T+1}$ is only related to the arrival, departure and job selections in rounds $\tau, ..., T$, and all the previous ones are completely independent of it. Since the Bellman policy used after the start of $\tau$ is optimal under $(q,\widehat F_n)$, applying \Cref{lem:planning-transfer} and then \Cref{lem:empty-horizon-monotonicity} shows that the conditional terminal cost $J_h^p(\widehat\pi)$ is at most
\begin{equation}
\begin{aligned}
 V_h^{\IA}+\frac{2h^2}{p_-}\eps_n(\delta_T)
 +\frac{(h-1)h(h+1)(h+2)}{12}\eta_n(\delta_T)
 \le{}& V_T^{\IA}+\frac{2H^2}{p_-}\eps_n(\delta_T)\\
 &+\frac{(H-1)H(H+1)(H+2)}{12}\eta_n(\delta_T).
 \label{eq:gamma bound of J}
\end{aligned}
\end{equation}

On $\mathcal G^c\cap\{\tau\le T\}$, the queue is empty at the start of round $\tau$, and hence there are only $h\le H$ rounds remaining. During these remaining rounds, at most $H$ arrivals can occur, so $Q_{T+1}\le H$ regardless of the planning model.

It remains to control the event $\{\tau=\infty\}$. Before planning, \Cref{alg:lcp} is work conserving and, in the event $\{\tau=\infty\}$, the queue remains non-empty at the beginning of rounds $L+1, ..., T$; our queue must follow a work-conserving policy starting from round $L+1$ to round $T$. Recall $Q_s=0$ if and only if $W_s=0$ in a work-conserving policy. If we define $\tau_0=\inf\{j\ge0:W_{L+1+j}=0\}$, on $\{\tau=\infty,W_{T+1}>0\}$, the workload has not reached zero in any of the first $H$ updates, and hence $\tau_0>H$. On $\{\tau=\infty,W_{T+1}=0\}$, the terminal queue is zero. Consequently, \Cref{eq:workload-tail} gives
\begin{equation}
\begin{aligned}
 \E\left(Q_{T+1}\ind{\tau=\infty}\right)
 &\le\E\left(W_{T+1}\ind{\tau=\infty}\right)\\
 &=\E\left(W_{T+1}\ind{\tau=\infty, W_{T+1}>0}\right)\\
 &\le\E\left(W_{T+1}\ind{\tau_0>H}\right)\\
 &\le\frac{\psi(r)^H}{r}\E e^{rW_{L+1}}
 \le\frac{K_r}{r}\psi(r)^H.
 \label{eq:exponential decay bound}
\end{aligned}
\end{equation}
For clarity, let
\[
 \Gamma_H=\frac{2H^2}{p_-}\eps_n(\delta_T)
 +\frac{(H-1)H(H+1)(H+2)}{12}\eta_n(\delta_T).
\]
For an event $E$, define the event-restricted terminal cost
\[
J_T^\pi(E)
:=
\mathbb E\!\left[Q_{T+1}^\pi \mathbf 1_E\right].
\]
Then we have
\[
J_T^{\pi^{\mathrm{LCP}}}
=
J_T^{\pi^{\mathrm{LCP}}}(\mathcal G\cap\{\tau\le T\})
+
J_T^{\pi^{\mathrm{LCP}}}(\mathcal G^c\cap\{\tau\le T\})
+
J_T^{\pi^{\mathrm{LCP}}}(\{\tau=\infty\}).
\]
We bound these three terms separately by the properties we've shown above.

For each event $E\in \mathcal G\cap\{\tau\le T\}$, we define $J_h^p(\widehat\pi, E)$ (here $h=T-\tau+1$) as the conditional terminal cost of the queue when the first $1, ..., \tau-1$ rounds have happened as $E$. In other words, $J_h^p(\widehat\pi, E)=\E \left(Q^{\pi^{\mathrm{LCP}}}_{T+1}\mid E\right)$. The queue has been cleared by the beginning of round $\tau$, so the planning phase starts from an empty queue with at most $H$ rounds remaining. So by \Cref{eq:gamma bound of J}, $J_h^p(\widehat\pi, E)\leq V_T^{\IA}+\Gamma_H$ holds. Hence,

\begin{align*}
J_T^{\pi^{\mathrm{LCP}}}(\mathcal G\cap\{\tau\le T\})
&= \E \left[\E \left(Q^{\pi^{\mathrm{LCP}}}_{T+1}\mid E\right)\mathbf 1(E\in \mathcal G\cap\{\tau\le T\})\right] \\
&= \E \left[J_h^p(\widehat\pi, E)\mathbf 1(E\in \mathcal G\cap\{\tau\le T\})\right] \\
&\leq \E \left[\left(V_T^{\IA}+\Gamma_H\right)\mathbf 1(E\in \mathcal G\cap\{\tau\le T\})\right] \\
&=\left(V_T^{\IA}+\Gamma_H\right)\mathbb P(\mathcal G\cap\{\tau\le T\}).
\end{align*}

For $\mathcal G^c\cap\{\tau\le T\}$, as we've shown that $Q_{T+1}^{\pi^{\mathrm{LCP}}}\le H$ holds regardless of the planning model, we can get the following inequality directly.
\[
J_T^{\pi^{\mathrm{LCP}}}(\mathcal G^c\cap\{\tau\le T\})
\le
H\mathbb P(\mathcal G^c\cap\{\tau\le T\}).
\]

Finally, on $\{\tau=\infty\}$, the following inequality holds by \Cref{eq:exponential decay bound}.
\[
J_T^{\pi^{\mathrm{LCP}}}(\{\tau=\infty\})
\le
\frac{K_r}{r}\psi(r)^H.
\]
Combining the three bounds gives
\begin{align*}
J_T^{\pi^{\mathrm{LCP}}}
&\le
\left(V_T^{\IA}+\Gamma_H\right)\mathbb P(\mathcal G\cap\{\tau\le T\})
+
H\mathbb P(\mathcal G^c\cap\{\tau\le T\})
+
\frac{K_r}{r}\psi(r)^H. \\
&\le
V_T^{\IA}+\Gamma_H
+
H\mathbb P(\mathcal G^c)
+
\frac{K_r}{r}\psi(r)^H.
\end{align*}
Subtracting $V_T^{\IA}$ and using $V_T^{\IA}\ge0$ yields
\begin{align*}
R_T^{\IA}(\pi^{\mathrm{LCP}})
&=
J_T^{\pi^{\mathrm{LCP}}}-V_T^{\IA} \\
&\le
\Gamma_H
+
H\mathbb P(\mathcal G^c)
+
\frac{K_r}{r}\psi(r)^H .
\end{align*}
Substituting \Cref{eq:known-good-event} proves \Cref{eq:known-horizon-bound}.

For $H=\lceil\log^2(eT)\rceil$ and sufficiently large $T$, $L\ge T/2$ and $n\ge\lambda T/32$. Hence $\eps_n(\delta_T)=\wtO(\sqrt{d/(\lambda T)})$ and $\eta_n(\delta_T)=\wtO((\lambda T)^{-1/2})$. The factors $H^2$ and $(H-1)H(H+1)(H+2)$ are polylogarithmic. Since $\psi(r)<1$, the workload and sampling-tail terms are smaller than any inverse polynomial in $T$. This proves the fact that regret is $\wtO(\sqrt{d/(\lambda T)})$. An upper bound on the total number of computations required to evaluate the Bellman recursions during the planning stage is given in \Cref{app:bellman-recursion-count}.
\end{proof}

\subsection{Total Bellman Computation Cost during the Plan Phase}
\label{app:bellman-recursion-count}

We now bound the total number of Bellman terms evaluated during the planning stage. Suppose that the planning stage starts with $h$ rounds remaining and with empirical distribution $\widehat F_n$, where $h\le H=T-L$. During planning, the algorithm recomputes the empirical Bellman recursion from the current state. For $s=1,\ldots,h$, let $B_s$ be the number of elementary Bellman terms evaluated when there are $s$ rounds remaining and the server is free. We bound $\sum_{s=1}^h B_s$.

It is useful to view one evaluation of the Bellman recursion as a finite directed acyclic computation graph (DAG). A node at depth $k$ represents a reachable state with $s-k$ rounds remaining, and directed edges point from a state to the one-step successor states appearing on the right-hand side of the Bellman recursion (here, Bellman recursion means the recursion formulas of $V_{k}(c, a), (\mathcal AV_{k})(c, a)$). The terminal nodes, corresponding to states of $0$ rounds remaining, have values given directly by the terminal queue length and therefore do not require any Bellman update. Thus $B_s$ is bounded by the number of outgoing edges from nonterminal nodes in this computation graph.

Since the planning stage starts from an empty queue, at the beginning of a round with $s$ rounds remaining, at most $h-s$ jobs have arrived during the planning stage. Thus the system contains at most $h-s$ jobs. Also, at most $h-s$ fitted departure-probability labels outside the original support of $\widehat F_n$ can have appeared as present-only labels. These labels are added to the finite label set used in the Bellman recursion and are assigned zero future-arrival mass, so the distribution $\widehat F_n$ itself does not change. Therefore, if $m$ denotes the size of the finite label set used in this Bellman computation, then
\[
    m\le n+h-s.
\]

In this Bellman computation, future arrivals are assumed to have departure-probability labels in this fixed finite label set. Hence every reachable state can be represented as $(c,a)$, where $c\in\mathbb Z_{\ge0}^m$ is the count vector and $a\in\{0,1,\ldots,m\}$ denotes the label of the job in service, with $a=0$ representing a free server. After $k$ further rounds of the recursion, $k=0,\ldots,s-1$, the total number of jobs present is at most $h-s+k$. Thus the number of possible count vectors at that depth is at most
\[
    \#\{c\in\mathbb Z_{\ge0}^m:\|c\|_1\le h-s+k\}
    =
    \binom{m+h-s+k}{m}.
\]
For each count vector, the service label $a$ has at most $m+1$ possible values. Therefore the number of nonterminal reachable states in this Bellman computation is at most
\[
    (m+1)\sum_{k=0}^{s-1}\binom{m+h-s+k}{m}
    \le
    (m+1)\binom{m+h}{m+1},
\]
where the last inequality follows from the hockey-stick identity.

For each nonterminal reachable state, evaluating its value requires at most $m+1$ values of the form $\mathcal A V$ at the next step, and each such value depends on at most $m+1$ next-stage values of $V$. Thus, in the corresponding computation graph, each nonterminal node has outdegree at most $(m+1)^2$. Hence,
\[
    B_s
    \le
    (m+1)^3\binom{m+h}{m+1}.
\]
Using $m\le n+h-s$, we obtain
\[
    B_s
    \le
    (n+h-s+1)^3
    \binom{n+2h-s}{h-1}
    \le (n+h+1)^3\binom{n+2h-s}{h-1}
\]
Consequently,
\begin{align}
    \sum_{s=1}^h B_s
    &\le
    (n+h+1)^3
    \sum_{s=1}^h
    \binom{n+2h-s}{h-1} \notag\\
    &\le
    (n+h+1)^3
    \binom{n+2h}{h}.
    \label{eq:planner-computation-count}
\end{align}
Since $h\le H$, the total number of elementary Bellman terms evaluated during the whole planning stage is at most
\[
    (n+H+1)^3\binom{n+2H}{H}.
\]
When $n=\Theta(T)$ and $H=\Theta(\log^2 T)$, the logarithm of this bound is $O(H\log T)=O(\log^3 T)$.

\subsection{\texorpdfstring{Proof of \Cref{lem:planning-transfer}}{Proof of the planning comparison lemma}}
\label{app:model-transfer}

We first compare two Bellman recursions that may differ in two ways: the two systems may assign different departure probabilities to the same present jobs, and the departure-probability distributions of future arrivals may be different. This slightly more general comparison will be used later with suitable specializations.

Let $V_k^F(s)$ be the minimum expected terminal queue length from state $s$ with $k$ rounds remaining, when the departure probability of each future arrival is drawn from $F$. Equivalently, $V_k^F(s)$ is the value function obtained from the Bellman recursion for the model that uses the departure probabilities recorded in the current state for present jobs and draws the departure probability of each future arrival from $F$. Let $Q_k^F(s,a)$ be the corresponding one-step value obtained by taking an admissible action $a$ in state $s$ and then acting optimally. Here admissibility includes the nonpreemptive constraint: if a job is currently in service, the only admissible action is to continue serving that job. For a policy $\pi$ and an empty initial state, write $J_h^F(\pi)$ for its expected terminal queue length, and write $V_h^F$ for the corresponding minimum over the relevant policy class.

We compare two states $s$ and $\widetilde s$ whose currently present jobs are paired, and whose jobs in service are paired as well. A paired present job may carry different departure-probability labels in the two states: we denote its label by $p_j$ in $s$ and by $\widetilde p_j$ in $\widetilde s$. Define
\[
 \Delta(s,\widetilde s)
 :=
 \sum_{j\text{ present}} |p_j-\widetilde p_j|.
\]

\begin{lemma}\label{lem:paired-bellman}
Let $F$ and $\widetilde F$ be distributions of the departure probabilities of future jobs, with $W_1(F,\widetilde F)=\eta$. For states $s$ and $\widetilde s$ described above and every $k\ge0$,
\begin{equation}
 |V_k^F(s)-V_k^{\widetilde F}(\widetilde s)|
 \le \frac{k(k+1)}2\Delta(s,\widetilde s)+\frac{(k-1)k(k+1)}6\eta.
 \label{eq:paired-bellman}
\end{equation}
The bound holds both when policies may idle and when they must be work conserving.
\end{lemma}

\begin{proof}
For convenience of explanation, we define $a_k:=\frac{k(k+1)}2, g_k:=\frac{(k-1)k(k+1)}6$. We will prove our lemma by induction on $k$. For $k=0$, both terminal values equal the common number of jobs in the two states, so \Cref{eq:paired-bellman} holds because $a_0=g_0=0$. Assume it holds with $k-1$ rounds remaining.

Take any action that is admissible in both states. If the action serves a job which is in both of $s, \widetilde s$, let $p$ and $\widetilde p$ be its departure probabilities and put $\delta=|p-\widetilde p|$. Use a common uniform random variable $U\sim \Unif[0,1]$ to generate the two departure indicators. If the selected job has departure probabilities $p$ and $\widetilde p$ in the systems started with $s, \widetilde s$ respectively, set
\[
D=\mathbf 1\{U\le p\},
\qquad
\widetilde D=\mathbf 1\{U\le \widetilde p\}.
\]
Then $D\sim\Bern(p)$, $\widetilde D\sim\Bern(\widetilde p)$, and
\[
\mathbb P(D\ne \widetilde D)=|p-\widetilde p|=\delta.
\]
Use the same arrival indicator $A\sim\Bern(\lambda)$ in both systems. Conditional on $A=1$, choose an optimal coupling $(P,\widetilde P)$ of $F$ and $\widetilde F$ such that
\[
\mathbb E|P-\widetilde P|=W_1(F,\widetilde F)=\eta.
\]
Existence of this coupling comes from the definition of Wasserstein distance. The two arriving jobs are then paired and assigned departure-probability labels $P$ and $\widetilde P$, respectively. If $A=0$, no new job is added to either system.

First, we see the case when the two departure indicators agree, i.e. both $0$ or both $1$. In both cases, the successor states then contain the same set of jobs, so we can use the induction hypothesis in both cases because they have the same set of jobs. If both indicators are $1$, its contribution $\delta$ is removed from $\Delta(s,\widetilde s)$; if both indicators are $0$, that contribution is retained. A new arrival adds $|P-\widetilde P|$ in both cases. In either case, the expected sum of the probability differences in the successor states is at most $\Delta(s,\widetilde s)+\eta$. The induction hypothesis therefore bounds the expected difference of their continuation values by
\begin{align}
 a_{k-1}\bigl(\Delta(s,\widetilde s)+\eta\bigr)+g_{k-1}\eta
 =a_{k-1}\Delta(s,\widetilde s)+(a_{k-1}+g_{k-1})\eta.
 \label{eq:paired-bellman-agree}
\end{align}

We now see the case when the two departure indicators are different. WLOG let $p\leq \widetilde p$, so in the queue started with state $s$ the departure indicator was $1$, but in the queue started with state $s$ the departure indicator was $0$. On their disagreement event, one successor state has exactly one more job than the other. Over the remaining $k-1$ rounds, the terminal queue difference caused by adding or deleting one job in the starting queue is at most $k$: the initial job-count difference is one, and the two systems can differ by at most one departure in each remaining round. The disagreement event has probability $\delta$, so its contribution is at most $k\delta$. Combining this contribution with \Cref{eq:paired-bellman-agree} and using $\delta\le\Delta(s,\widetilde s)$ gives
\begin{align}
 |Q_k^F(s,a)-Q_k^{\widetilde F}(\widetilde s,a)|
 &\le\left(a_{k-1}\Delta(s,\widetilde s)
 +(a_{k-1}+g_{k-1})\eta\right)+k\Delta(s,\widetilde s)\\
 &=a_k\Delta(s,\widetilde s)+g_k\eta,
 \label{eq:paired-common-action}
\end{align}
where $a_k=a_{k-1}+k$ and $g_k=g_{k-1}+a_{k-1}$. If the common action is to idle, no departure indicator is generated, and at the beginning of the next round the two queues have the same set of jobs. So by using the induction hypothesis again, it's not difficult to show the same bound holds without the $k\delta$ term.

Let $a^*$ minimize $Q_k^F(s,a)$ and let $\widetilde a^*$ minimize $Q_k^{\widetilde F}(\widetilde s,a)$ (so that $Q_k^F(s,a^*)=V_k^F(s), Q_k^{\widetilde F}(\widetilde s,a)=V_k^{\widetilde F}(\widetilde s)$). Of course these $a^*, \widetilde a^*$ denote the job in service when the server is not free in the starting state $s$. Applying \Cref{eq:paired-common-action} to $\widetilde a^*$ gives
\[
 V_k^F(s)-V_k^{\widetilde F}(\widetilde s)
 \le Q_k^F(s,\widetilde a^*)-Q_k^{\widetilde F}(\widetilde s,\widetilde a^*)
 \le a_k\Delta(s,\widetilde s)+g_k\eta.
\]
Applying it to $a^*$ with the two models interchanged gives the reverse inequality. This proves \Cref{eq:paired-bellman}.
\end{proof}

Setting $s=\widetilde s$ (so all the probabilities are equal in the initial queue) in \Cref{lem:paired-bellman} gives
\[
 |V_k^F(s)-V_k^{\widetilde F}(s)|\le g_k\eta.
\]
The common-action bound \Cref{eq:paired-common-action}, again with $s=\widetilde s$, gives $|Q_k^F(s,a)-Q_k^{\widetilde F}(s,a)|\le g_k\eta$ for every admissible action $a$. Let $\widetilde\pi_k(s)$ be an optimal action with $k$ rounds remaining under $\widetilde F$ (so we can think that $\widetilde\pi_k(s)$ is the optimal decision according to the Bellman recursion defined by the starting queue $s$, the given success probability for each job, and the future success-probability distribution $\widetilde F$). Of course it denotes the job in service when the server is not free in the starting state $s$. Since $Q_k^{\widetilde F}(s,\widetilde\pi_k(s))=V_k^{\widetilde F}(s)$,
\begin{align*}
 Q_k^F(s,\widetilde\pi_k(s))-V_k^F(s)
 &=Q_k^F(s,\widetilde\pi_k(s))-Q_k^{\widetilde F}(s,\widetilde\pi_k(s))\\
 &\quad+V_k^{\widetilde F}(s)-V_k^F(s)\\
 &\le2g_k\eta.
\end{align*}

Let $S_k$ be the state with $k$ rounds remaining when $\widetilde\pi$ (recall that this policy selects its decision according to the Bellman recursion, which is defined by the given success probability for each job, and the future success-probability distribution $\widetilde F$) is run in the model with future-job distribution $F$. By the definition of $Q_k^F$,
\[
 Q_k^F(S_k,\widetilde\pi_k(S_k))
 =\E\left(V_{k-1}^F(S_{k-1})\mid S_k\right).
\]

So $\E \left(Q_k^F(S_k,\widetilde\pi_k(S_k))\right)=\E \left[\E\left(V_{k-1}^F(S_{k-1})\mid S_k\right)\right]=\E \left(V_{k-1}^F(S_{k-1})\right)$ holds. And also by the definition and assumptions of \Cref{lem:planning-transfer}, $S_h=\varnothing$, $\E \left(V_h^F(\varnothing)\right)=V_h^F$ and $J_h^F(\widetilde\pi)=\E(S_0)=\E\left(V_0^F(S_0)\right)$ hold. Taking expectations and summing over $k=h,h-1,\ldots,1$ telescopes:
\begin{align}
 J_h^F(\widetilde\pi)-V_h^F
 &=\sum_{k=1}^h\left(\E \left(V_{k-1}^F(S_{k-1})\right)-\E\left(V_k^F(S_k)\right)\right)\notag\\
 &=\sum_{k=1}^h\left(\E \left(Q_k^F(S_k,\widetilde\pi_k(S_k))\right)-\E\left(V_k^F(S_k)\right)\right)\notag\\
 &=\sum_{k=1}^h\E \left(Q_k^F(S_k,\widetilde\pi_k(S_k))-V_k^F(S_k)\right)\notag\\
 &\le2\eta\sum_{k=1}^hg_k
 =2\eta\binom{h+2}{4}
 =\frac{(h-1)h(h+1)(h+2)}{12}\eta.
 \label{eq:distribution-transfer}
\end{align}

We next control the change from fitted to true departure probabilities for the same contexts.

\begin{lemma}\label{lem:geometric-coupling}
For $p,q\in[p_-,p_+]$, there is a coupling of $G_p\sim\Geom(p)$ and $G_q\sim\Geom(q)$ such that
\[
 \P(G_p\ne G_q)=\frac{|p-q|}{\max(p,q)}\le\frac{|p-q|}{p_-}.
\]
\end{lemma}

\begin{proof}
Let $U_1,U_2,\ldots$ be i.i.d. uniform random variables and set $G_p=\inf\{g:U_g\le p\}$ and $G_q=\inf\{g:U_g\le q\}$. If $p\le q$, then at the first time $g$ such that $U_g\le q$ holds, the service times agree exactly when $U_g\le p$, so this event has conditional probability $p/q$. Thus $\P(G_p\ne G_q)=1-p/q=(q-p)/q$. The other case is symmetric.
\end{proof}

For this part, $J_h^q(\pi)$ and $V_h^q$ denote the cost of a fixed policy $\pi$ and the optimal value when the same contexts are used but every departure probability $p(X)$ is replaced by $q(X)$. Fix $q$ before an $h$-round process starts empty, and use the same arrivals and contexts in the true and fitted systems (i.e. systems using the same distribution of context vectors, but different departure probability functions). Couple the two service times of each arriving job using \Cref{lem:geometric-coupling}. Define $\mathcal I_h$ as the set of jobs arriving during these $h$ rounds, and define $\mathcal{H}$ as the set of events such that there is at least one job such that it started to be served in both queues during the $h$ rounds, but the service times of that job in the two queues were different. Then, conditional on their contexts and $\mathcal I_h$, each job $X_i\in \mathcal I_h$ has service time disagreement probability at most $\frac{|p(X_i)-q(X_i)|}{p_-}$. So a union bound gives
\[
 \P(\mathcal{H}\mid \mathcal I_h)
 \le\frac1{p_-}\sum_{i\in\mathcal I_h}|p(X_i)-q(X_i)|.
\]
Because $q$ is fixed before the process begins and each arriving context has distribution $\mathcal D$,
\begin{align*}
 \E\sum_{i\in\mathcal I_h}|p(X_i)-q(X_i)|
 &=\E\left[|\mathcal I_h|\cdot \E_{X\sim \mathcal{D}}\left|p(X)-q(X)\right|\right] \\
 &=\lambda h\,\E_{X\sim \mathcal{D}}|p(X)-q(X)|
 \le h\mathcal E(q).
\end{align*}
Thus the probability that there exists a job that is served in both queues during the $h$ rounds and the service time of that job is different in the two queues is at most $h\mathcal E(q)/p_-$. If every pair of service times agrees, the two systems have identical histories and take identical actions under the same nonanticipating policy. On the disagreement event, each terminal queue contains at most the $h$ jobs that could have arrived, so the final queue length difference is at most $h$. Therefore, for every fixed policy $\pi$,
\begin{equation}
 |J_h^p(\pi)-J_h^q(\pi)|\le\frac{h^2}{p_-}\mathcal E(q).
 \label{eq:same-context-transfer}
\end{equation}
Let $\pi_p$ and $\pi_q$ be optimal in the true and fitted systems, respectively (so $J_h^p(\pi_p)=V_h^p, J_h^q(\pi_q)=V_h^q$). Applying \Cref{eq:same-context-transfer} to these two policies gives
\begin{align*}
 V_h^p-V_h^q
 &\le J_h^p(\pi_q)-J_h^q(\pi_q)
 \le\frac{h^2}{p_-}\mathcal E(q),\\
 V_h^q-V_h^p
 &\le J_h^q(\pi_p)-J_h^p(\pi_p)
 \le\frac{h^2}{p_-}\mathcal E(q).
\end{align*}
Hence $|V_h^p-V_h^q|\le h^2\mathcal E(q)/p_-$.

Finally, compare the true model that follows $(p, F_p)$ with the intermediate model $(q,F_q)$ and then with the empirical model $(q,\widehat F)$:
\begin{align*}
 J_h^p(\widehat\pi)-V_h^p
={}&\left(J_h^p(\widehat\pi)-J_h^q(\widehat\pi)\right)
 +\left(J_h^q(\widehat\pi)-V_h^q\right)
 +\left(V_h^q-V_h^p\right).
\end{align*}
The first and third terms are bounded by \Cref{eq:same-context-transfer}; \Cref{eq:distribution-transfer} bounds the middle term with $\eta=W_1(F_q,\widehat F)$. Adding the three bounds gives \Cref{eq:planning-transfer}.

\subsection{\texorpdfstring{Proof of \Cref{lem:sample-availability}}{Proof of the sample-availability lemma}}

\begin{proof}
Let $M_L=\sum_{t=1}^LA_t$ be the number of arrivals through round $L$. Since $M_L\sim\mathrm{Binomial}(L,\lambda)$ and $\E(M_L)=\lambda L$, the multiplicative Chernoff bound with relative deviation $1/2$ gives $\P(M_L<\lambda L/2)\le e^{-\lambda L/8}$. Also, $2n\le\lambda L/4$ by definition. Therefore, on $\{M_L\ge\lambda L/2,N_L<2n\}$, the number of jobs that have arrived but not begun service by the end of round $L$ is $M_L-N_L>\lambda L/2-2n\geq \lambda L/4$. So more than $\lambda L/4$ jobs have arrived but not begun service by the end of round $L$. Each such job contributes at least one unit to $W_{L+1}$, so $W_{L+1}>\lambda L/4$. Markov's inequality applied to $e^{rW_{L+1}}$ and \Cref{lem:workload} give
\begin{equation}
\begin{aligned}
 \P(M_L\ge\lambda L/2,N_L<2n)
 &\le\P(W_{L+1}>\lambda L/4) \\
 &=\P(e^{rW_{L+1}}>e^{r\lambda L/4}) \\
 &\le \frac{\E (e^{rW_{L+1}})}{e^{r\lambda L/4}}\le K_re^{-r\lambda L/4}.
\end{aligned}
\end{equation}
So by a union bound, we can get the desired inequality as follows.
\begin{equation}
\begin{aligned}
 \P(N_L<2n)
 &=\P(M_L\ge\lambda L/2,N_L<2n)+\P(M_L<\lambda L/2,N_L<2n) \\
 &\le \P(M_L\ge\lambda L/2,N_L<2n)+\P(M_L<\lambda L/2) \\
 &\le K_re^{-r\lambda L/4}+e^{-\lambda L/8}
\end{aligned}
\end{equation}
Under FCFS, jobs begin service in arrival order. Hence, whenever $N_L\ge2n$, the first $2n$ recorded observations are $(X_i,Y_i)_{i=1}^{2n}$.
\end{proof}

\subsection{\texorpdfstring{Proof of \Cref{lem:empty-horizon-monotonicity}}{Proof of the horizon-monotonicity lemma}}

\begin{proof}
Let $s$ be an initial queue and let $s^+$ be the queue that contains all jobs of $s$ together with some additional jobs. Fix a policy $\pi^+$ for $s^+$. We construct a policy $\pi$ for $s$ by maintaining a virtual copy of the larger system, and show that by this algorithm its expected terminal queue length is not bigger than that of $\pi^+$ started from $s^+$.

Couple the two systems by the same future arrivals and by the same service times for every common job. Whenever $\pi^+$ selects a common job, the smaller system selects the same job; whenever $\pi^+$ selects one of the additional jobs, the smaller system idles. Idling is admissible in the policy class considered here. By induction over the number of rounds, we can directly prove that every job in the smaller system is also present in the virtual larger system, and the algorithm always works well in every round. Hence the terminal queue under the constructed policy is not bigger than the terminal queue under $\pi^+$. This shows that, if we take an infimum of the expected queue length (at the end) of the queue started with $s$ among all possible policies, it is not bigger than the expected queue length of the queue started with $s^+$ and using policy $\phi^+$. Now, taking an infimum on both sides of the inequality over all possible policies $\pi^+$, we can prove monotonicity in the initial state.

Now start an $(h+1)$-round system with an empty queue. No job is available for service in its first round, and the server must idle. After the arrival at the end of that round, the queue state $S$ is either empty or contains one waiting job. Since what we've proved above implies $V_h^{\IA}(S)\geq V_h^{\IA}(\varnothing)=V_h^{\IA}$, the Bellman recursion and the initial-state monotonicity give
\[
 V_{h+1}^{\IA}=\E V_h^{\IA}(S)
 \ge \E V_h^{\IA}(\varnothing)=V_h^{\IA}.
\]
\end{proof}

\section{Lower Bounds}
\label{app:lower-proof}

This section constructs the common family of instances and proves \Cref{thm:lower-wc,thm:lower-ia}. For work-conserving policies, the last round reduces to choosing the faster of two waiting jobs. For policies that may idle, we first show that any policy with small regret reaches the same decision state with probability bounded away from zero.

\subsection{The hard family}
\label{app:lower-family}

We first define the distribution $\mathcal{D}$, job context vectors, probability constants ($\lambda, p_-, p_+$), and parameter vector $\theta^*$ that will be used to make a lower bound of regret as follows.

\begin{enumerate}
\item Job types and distribution $\mathcal{D}$: \\
Let $m=d-2$, and let there be $d$ types of jobs in our distribution: a dummy job $C_D$, a baseline job $C_0$, and $m$ candidate jobs $C_1,\ldots,C_m$. Conditional on an arrival, let's define each job's arrival probability
\[
 \P(C_D)=\nu_D, \quad\P(C_0)=\nu_0,
 \quad
 \P(C_j)=\frac{\nu_C}{m}.
\]
where $\nu_D, \nu_0, \nu_C$ are three nonnegative real numbers that sum to $1$. So $\mathcal{D}$ is the distribution that consists of jobs $C_D, C_0, C_1, \ldots, C_m$ with arrival probability as above.

\item Job context vectors and probability constants:\\
For the $d$-dimensional standard basis set $(e_j)_{j=1}^d$, define
\[
 x_D=e_1,
 \qquad
 x_0=\frac{e_2}{\sqrt2},
 \qquad
 x_{C_j}=\frac{e_2+e_{j+2}}{\sqrt2}.
\]
Choose fixed $z_D,z_0,a_0,S$ ($a_0, S>0$) such that $z_D>z_0+a_0$ and $S^2>z_D^2+2z_0^2$, and set
\[
 c_S=\sqrt{\frac{S^2-z_D^2-2z_0^2}{2}},
 \qquad
 \ell_0=\min_{|z-z_0|\le a_0}\mu'(z)>0.
\]
For example, we can choose $z_D=\frac{S}{2}, z_0=-\frac{S}{2}, a_0=0.99S$.

\medskip
Now, we choose fixed bounds $p_-\le\mu(z_0-a_0)$ and $p_+\ge\mu(z_D)$ with $0<p_-<p_+<1$, and a fixed $\lambda>0$ small enough that the following inequality holds.
\begin{equation}
 \lambda<p_-,
 \qquad
 \frac{\lambda(1-\lambda)}{p_--\lambda}
 \le\frac{(1-\mu(z_D))^4}{16}
 \label{eq:lower-lambda}
\end{equation}
We can always choose such $\lambda>0$ since the functions $\lambda\mapsto \lambda, \lambda\mapsto \frac{\lambda(1-\lambda)}{p_--\lambda}$ (defined on $0<\lambda<p_-$) go to $0$ as $\lambda \rightarrow 0^+$.

\item Parameter vector $\theta^*$:\\
For fixed constants $\alpha>0$ that will be determined later, let $a_T=\min\left(a_0,\frac{c_S}{\sqrt{m}}, \alpha\sqrt{\frac{m}{T}}\right)$. Next, for each $\omega\in\{-1,+1\}^m$, set the parameter vector corresponding to $\omega$ as
\[
 \theta_\omega=(z_D,\sqrt2z_0,\sqrt2\omega_1a_T,\ldots,\sqrt2\omega_ma_T).
\]
Then the dummy job has departure probability $p_D=\mu(z_D)$, the baseline job has departure probability $p_0=\mu(z_0)$, and the candidate job $C_j$ has departure probability $p_{C_j}=\mu(z_0+\omega_ja_T)$.

\end{enumerate}

The following lemma shows that the above definitions satisfy \Cref{ass:model}.

\medskip
\begin{lemma}\label{lem:lower-feasibility}
The above construction satisfies \Cref{ass:model} for every $d, T$.
\end{lemma}

\begin{proof}
Since $a_T\le \frac{c_S}{\sqrt m}$ by definition and $w_i^2=1$ for all $i\in [m]$,
\begin{align*}
 \|\theta_\omega\|_2^2
 &=z_D^2+2z_0^2+2(w_1^2+...+w_m^2)a_T^2 \\
 &=z_D^2+2z_0^2+2ma_T^2 \\
 &\le z_D^2+2z_0^2+2c_S^2=S^2.
\end{align*}

Since $a_T\le a_0$ by definition, $z_0+w_ja_T\in \{z_0-a_T, z_0+a_T\}\subset [z_0-a_0,z_D]$ holds. Also $z_0, z_D\in [z_0-a_0,z_D]$, so every job's departure probability lies in $[\mu(z_0-a_0), \mu(z_D)]\subset [p_-,p_+]$. Finally,
\[
 \rho_\omega=\lambda\E_\omega\frac1{p(X)}\le\frac{\lambda}{p_-}<1.
\]
\end{proof}

\subsection{Information in the observed history}

Fix an adaptive learning policy $\pi$. For each $\omega\in\{-1,+1\}^m$, let $P_\omega^T$ denote the probability law induced by running $\pi$ for $T$ rounds on the instance with parameter $\theta_\omega$. This law is defined on the complete interaction history through round $T$, including all arrival indicators and contexts, the learner's internal randomization and resulting actions, and all observed departure indicators. Thus, for any event $\mathcal E$ determined by the interaction through round $T$, $P_\omega^T(\mathcal E)$ is the probability that $\mathcal E$ occurs under parameter $\theta_\omega$ and policy $\pi$. We suppress the dependence of $P_\omega^T$ on the fixed policy $\pi$ in the notation.\\
Let $\omega^{(j)}$ be the vector obtained by changing the sign of coordinate $j$ ($\omega_j$) of $\omega$, and let $N_j(T)$ be the number of rounds through $T$ in which a candidate-$j$ job is served.

\begin{lemma}\label{lem:neighboring-kl}
For every adaptive nonpreemptive policy, including policies that may idle,
\begin{align}
 \KL(P_\omega^T\|P_{\omega^{(j)}}^T)
 &\le\frac{a_T^2}{2}\E_\omega N_j(T)
 \le\frac{\lambda \nu_CT a_T^2}{2mp_-},
 \label{eq:neighboring-kl}\\
 \TV(P_\omega^T,P_{\omega^{(j)}}^T)
 &\le a_T\sqrt{\frac{\lambda \nu_CT}{4mp_-}}
 \le\frac\alpha2\sqrt{\frac{\lambda \nu_C}{p_-}}.
 \label{eq:neighboring-tv}
\end{align}
\end{lemma}

\begin{proof}
Only the departure indicators observed while serving candidate-$j$ jobs have different conditional distributions under $\omega$ and $\omega^{(j)}$. The two neighboring logits are $z_0+\omega_ja_T$ and $z_0-\omega_ja_T$. If we define $A(z)=\log(1+e^z)$, then $A'(z)=\mu(z)$ and $\KL\left(\Bern(\mu(u))\,\middle\|\,\Bern(\mu(v))\right)=A(v)-A(u)-A'(u)(v-u)$ hold, and the following inequality can be obtained from Taylor's theorem.
\begin{align*}
 &\KL\left(\Bern(\mu(z_0+\omega_ja_T))\,\middle\|\,\Bern(\mu(z_0-\omega_ja_T))\right) \\
 &=A(z_0+\omega_ja_T)-A(z_0-\omega_ja_T)-2A'(z_0-\omega_ja_T)\omega_ja_T \\ 
 &=\frac{1}{2}A''(\varepsilon)(2\omega a_T)^2\le2\left(\sup_z\frac{e^z}{(1+e^z)^2}\right)a_T^2
 \le\frac{a_T^2}{2}
\end{align*}
Here, $\varepsilon$ is some real number in $(z_0-\omega_ja_T, z_0+\omega_ja_T)$ obtained by second order Taylor's theorem, and we used $A''(\varepsilon)=\frac{e^\varepsilon}{(1+e^\varepsilon)^2}\leq \sup_z\frac{e^z}{(1+e^z)^2}=\frac{1}{4}$.

Let $I_t^{(j)}$ indicate that a candidate-$j$ job is served in round $t$. Let $\mathcal F_t^-$ denote the information available immediately before the departure indicator $D_t$ is observed; this includes the complete history through round $t-1$, as well as the learner's randomization and resulting action in round $t$. In particular, $I_t^{(j)}$ is $\mathcal F_t^-$-measurable.

Conditional on $\mathcal F_t^-$, the two neighboring instances have identical departure-indicator laws when $I_t^{(j)}=0$. When $I_t^{(j)}=1$, the conditional laws of $D_t$ under $\omega$ and $\omega^{(j)}$ are, respectively,

$$
 \Bern(\mu(z_0+\omega_j a_T))
 \quad\text{and}\quad
 \Bern(\mu(z_0-\omega_j a_T)).
$$

All other conditional kernels, including those governing arrivals, contexts, learner randomization, and actions, are identical under the two instances. Therefore, the KL chain rule gives
\begin{align*}
\KL(P_\omega^T\|P_{\omega^{(j)}}^T)
&=
\sum_{t=1}^T
\E_\omega\left[
I_t^{(j)}
\KL\left(
\Bern(\mu(z_0+\omega_j a_T))
\,\middle\|\,
\Bern(\mu(z_0-\omega_j a_T))
\right)
\right]\ \\
&=
\left(\sum_{t=1}^T\E_\omega I_t^{(j)}\right)
\KL\left(
\Bern(\mu(z_0+\omega_j a_T))
\,\middle\|\,
\Bern(\mu(z_0-\omega_j a_T))
\right)\ \\
&=
\left(\E_\omega\sum_{t=1}^TI_t^{(j)}\right)
\KL\left(
\Bern(\mu(z_0+\omega_j a_T))
\,\middle\|\,
\Bern(\mu(z_0-\omega_j a_T))
\right)\ \\
&=
E_\omega N_j(T)\cdot 
\KL\left(
\Bern(\mu(z_0+\omega_j a_T))
\,\middle\|\,
\Bern(\mu(z_0-\omega_j a_T))
\right)\ \\
&\le
\frac{a_T^2}{2}\E_\omega N_j(T),
\end{align*}
where $N_j(T)=\sum_{t=1}^T I_t^{(j)}$. This proves the first inequality in \Cref{eq:neighboring-kl}.

Pre-sample the geometric service time of every candidate-$j$ job. Regardless of the scheduling decisions, the number of rounds that one specific candidate-$j$ job is served by time $T$ cannot exceed its service time. Therefore $N_j(T)\leq \sum_{t=1}^T\ind{A_t=1,\,X_t=C_j}G_t$, and the following holds.
\begin{align*}
 \E_\omega N_j(T)
 &\le\E_\omega\left(\sum_{t=1}^T\ind{A_t=1,\,X_t=C_j}G_t\right)\\
 &=\sum_{t=1}^T\E_\omega\left(\ind{A_t=1,\,X_t=C_j}G_t\right)\\
 &=T\cdot \E_\omega\left(\ind{A_t=1,\,X_t=C_j}G_t\right)\\
 &=T\cdot \lambda\frac{\nu_C}{m}\E_\omega(G_t\mid X_t=C_j) \\
 &=\frac{\lambda \nu_CT}{mp(x_{C_j})}
 \le\frac{\lambda \nu_CT}{mp_-}
\end{align*}
This proves the second inequality in \Cref{eq:neighboring-kl}.

For \Cref{eq:neighboring-tv}, using Pinsker's inequality $\TV(P, Q)\leq \sqrt{\frac{1}{2}\KL(P\|Q)}$ gives the first bound in \Cref{eq:neighboring-tv}, and $a_T\le\alpha\sqrt{m/T}$ gives the second.
\end{proof}

For later use, let $P_\omega^t, P_\omega^{t,\mathrm{pre}}$ denote the law of the history up to the end of round $t$, and the law of the history after service in round $t$ and before the arrival at the end of that round. This history is a measurable function of the complete observations through round $t$, so the definition of total variance between two distributions $P_\omega^t, P_{\omega^{(j)}}^t$ gives
\begin{equation}
 \TV(P_\omega^{t,\mathrm{pre}},P_{\omega^{(j)}}^{t,\mathrm{pre}})
 \le \TV(P_\omega^t,P_{\omega^{(j)}}^t).
 \label{eq:pre-arrival-tv}
\end{equation}

Define
\begin{equation}
 \Delta_0=1-\frac\lambda{p_-},
 \qquad
 \zeta=\lambda^3\nu_D\nu_0\nu_Cp_D(1-p_D),
 \qquad
 g_D=(1-\lambda)^2p_D(1-p_D)^2.
 \label{eq:lower-constants}
\end{equation}
It's easy to check that these three constants are all positive. Choose the fixed $\alpha>0$ which is introduced in \Cref{app:lower-family} so that following inequality holds.
\[
 \alpha
 \le\sqrt{\frac{p_-}{\lambda \nu_C}}\min\left(\frac{\zeta\Delta_0}{2},\frac{3\zeta}{16}\right)
\]
This condition was chosen to make the inequality $\frac\alpha2\sqrt{\frac{\lambda \nu_C}{p_-}}\le\min\left(\frac{\zeta\Delta_0}{4},\frac{3\zeta}{32}\right)$ hold. So by this condition, $\TV(P_\omega^T,P_{\omega^{(j)}}^T)\leq \frac{\zeta\Delta_0}{4}, \frac{3\zeta}{32}$ holds.

\subsection{\texorpdfstring{Proof of \Cref{thm:lower-wc} for WC case}{Proof of the work-conserving lower bound}}

\begin{proof}
Under every work-conserving policy, the workload started with $W_1=0$ and satisfies \Cref{eq: workload recursion} is stochastically dominated by stationary workload(\Cref{lem:workload}). By \Cref{lem:stationary-workload}, for every $\omega$,
\begin{equation}
 \P_\omega(W_{T-3}=0)\ge \P_\omega(W=0)=1-\rho_\omega\ge\Delta_0
 \label{eq:lower-empty}
\end{equation}

where $W$ denotes the stationary distribution's workload. The above inequality holds because the zero-started workload $W_t$ is stochastically dominated by the stationary workload $W$ for every $t\ge1$.

\begin{figure}[t]
\centering
\resizebox{\linewidth}{!}{%
\begin{tikzpicture}[
    x=1.75cm,
    y=1cm,
    roundlabel/.style={
        font=\small,
        align=center
    },
    phase/.style={
        font=\scriptsize,
        text=gray!75!black,
        align=center
    },
    eventbox/.style={
        font=\scriptsize,
        align=center,
        text width=1.58cm,
        inner sep=0pt,
        anchor=north
    },
    event/.style={
        font=\small,
        align=center
    },
    state/.style={
        font=\scriptsize,
        align=center,
        text width=1.80cm,
        inner sep=0pt
    },
    job/.style={
        rectangle,
        draw=black,
        minimum width=9mm,
        minimum height=7mm,
        inner sep=1pt,
        font=\small
    },
    dummy/.style={fill=blue!12},
    baseline/.style={fill=orange!14},
    candidate/.style={fill=green!12}
]

\fill[yellow!12] (-0.12,0.52) rectangle (5,-4.40);

\draw[thick] (0,0)--(8,0);

\foreach \x in {0,2,4,6,8}
    \draw[thick] (\x,0.12)--(\x,-1.60);

\foreach \x in {1,3,5,7}
    \draw[dashed,gray] (\x,0.12)--(\x,-1.60);

\node[roundlabel] at (1,0.32) {round $T-3$};
\node[roundlabel] at (3,0.32) {round $T-2$};
\node[roundlabel] at (5,0.32) {round $T-1$};
\node[roundlabel] at (7,0.32) {round $T$};

\foreach \x in {0.5,2.5,4.5,6.5}
    \node[phase] at (\x,-0.25) {service};

\foreach \x in {1.5,3.5,5.5,7.5}
    \node[phase] at (\x,-0.25) {arrival};

\node[eventbox] at (0.5,-0.48)
    {queue empty\\(no service)};

\node[eventbox] at (1.5,-0.48)
    {$C_D$ arrives\\w.p. $\lambda\nu_D$};

\node[eventbox] at (2.5,-0.48)
    {serve $C_D$\\fails w.p. $1-p_D$};

\node[eventbox] at (3.5,-0.48)
    {$C_0$ arrives\\w.p. $\lambda\nu_0$};

\node[eventbox] at (4.5,-0.48)
    {continue $C_D$\\succeeds w.p. $p_D$};

\node[eventbox] at (5.5,-0.48)
    {$C_j$ arrives\\w.p. $\lambda\nu_C/m$};

\node[eventbox] at (6.5,-0.48)
    {choose $C_0$\\or $C_j$};

\node[eventbox] at (7.5,-0.48)
    {arrival irrelevant\\to terminal choice};

\draw[gray!50] (0,-1.60)--(8,-1.60);

\begin{scope}[yshift=-3.25cm]

    \draw[thick] (-0.35,0)--(0.35,0);
    \node[state,text width=2.35cm] at (0,-0.50)
        {empty\\$W_{T-3}=Q_{T-3}=0$};

    \draw[thick] (1.65,0)--(2.35,0);
    \node[job,dummy] at (2,0.35) {$C_D$};
    \node[state] at (2,-0.50)
        {start of\\round $T-2$\\server free};

    \draw[thick] (3.65,0)--(4.35,0);
    \node[
        job,
        dummy,
        double,
        double distance=0.7pt
    ] at (4,0.35) {$C_D$};
    \node[job,baseline] at (4,1.08) {$C_0$};
    \node[state] at (4,-0.50)
        {start of\\round $T-1$\\$C_D$ in service};

    \draw[thick] (4.65,0)--(5.35,0);
    \node[job,baseline] at (5,0.35) {$C_0$};
    \coordinate (prearrivaltop) at (5,0.70);
    \node[state] at (5,-0.50)
        {after service,\\before arrival};

    \draw[thick] (5.65,0)--(6.35,0);
    \node[job,baseline] at (6,0.35) {$C_0$};
    \node[job,candidate] at (6,1.05) {$C_j$};
    \node[state] at (6,-0.50)
        {start of\\round $T$\\server free};

    \draw[->,thick] (6.40,0.55)--(7.15,0.55);
    \node[event] at (7.45,0.98)
        {select $C_0$ or $C_j$};
    \node[state,text width=2.25cm] at (7.45,0.05)
        {choosing the slower job costs at least $\ell_0a_T$};

\end{scope}

\draw[dashed,gray] (5,-1.60)--(prearrivaltop);

\draw[magenta!55!red,very thick]
    (-0.12,-4.40)--(5,-4.40);

\draw[magenta!55!red,very thick]
    (-0.12,-4.30)--(-0.12,-4.50);

\draw[magenta!55!red,very thick]
    (5,-4.30)--(5,-4.50);

\node[event] at (2.44,-4.70)
    {event $\mathcal H$};

\end{tikzpicture}%
}
\caption{ The final four-round construction used in the lower-bound proof. Solid vertical lines mark round boundaries, and dashed lines separate service from the end-of-round arrival. The shaded region corresponds to $\mathcal H$; the candidate-$j$ arrival is appended afterward. }
\label{fig:lower-bound-event}
\end{figure}
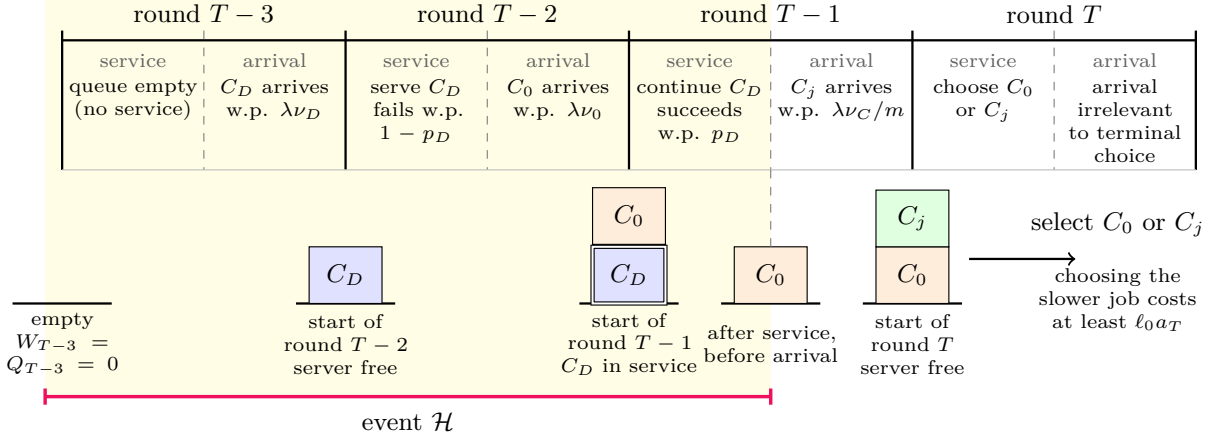

Starting from this event $\{W_{T-3}=0\}$, consider the following arrivals and departures: First, the server must idle since $W_{T-3}=0$, and then a dummy arrives at the end of round $T-3$. The server starts to process that dummy job in round $T-2$, but it fails to succeed. Next, a baseline job arrives at the end of round $T-2$. In round $T-1$, the server must do the failed dummy job, and it succeeds with it in round $T-1$. Let $\mathcal H$ denote this event before the arrival at the end of round $T-1$. The above \Cref{fig:lower-bound-event} illustrates this event. By independence of events,

\begin{equation}
\begin{aligned}
\P_\omega(\mathcal H)
&=\P_\omega(W_{T-3}=0)\cdot (\lambda \nu_D)\cdot (1-p_D)\cdot (\lambda \nu_0)\cdot p_D \\
&\ge\Delta_0\lambda^2\nu_D\nu_0p_D(1-p_D)
 =\frac{\zeta\Delta_0}{\lambda \nu_C}.
 \label{eq:wc-terminal-event}
\end{aligned}
\end{equation}

Conditional on $\mathcal H$, candidate $C_j$ arrives at the end of round $T-1$ with probability $\lambda \nu_C/m$. The server is free in round $T$ and must choose between $C_0$ and $C_j$. If $\omega_j=+1$, $C_j$ has the higher departure probability; if $\omega_j=-1$, $C_0$ does. In either case, the mean-value theorem gives
\[
 |p_{C_j}-p_0|
 =|\mu(z_0+\omega_ja_T)-\mu(z_0)|
 \ge\ell_0a_T.
\]
Thus, serving the slower job rather than the faster one increases the expected terminal queue length by $|p_{C_j}-p_0|\ge \ell_0a_T$.

Fix $j$ and pair an instance $\omega$ with $\omega^{(j)}$. Given the history just before the arrival at the end of round $T-1$, let the Markov kernel $K_j$ append a candidate-$j$ arrival and then apply the learner's randomized decision in round $T$. Write $\overline P_{\omega,j}=P_\omega^{T-1,\mathrm{pre}}K_j$, and let $\mathcal S_j(\omega)$ be the event that, conditional on $\mathcal{H}$ and on $C_j$ arriving at the end of round $T-1$, the decision selects the optimal job between $C_0$ and $C_j$ in round $T$ for the parameter vector $\theta_\omega$. So if $\omega_j=1$ then $\mathcal S_j(\omega)$ is the event that the decision selects $C_j$, and if $\omega_j=-1$ then $\mathcal S_j(\omega)$ is the event that the decision selects $C_0$. This directly shows that, within this conditional experiment, $\mathcal S_j(\omega)^c=\mathcal S_j(\omega^{(j)})$ since $[\omega^{(j)}]_j=-\omega_j$. By these definitions, data processing for $K_j$, followed by \Cref{eq:pre-arrival-tv}, gives
\begin{align*}
 &\overline P_{\omega,j}(\mathcal H,\mathcal S_j(\omega)^c)+\overline P_{\omega^{(j)},j}(\mathcal H,\mathcal S_j(\omega^{(j)})^c)\\
 &\quad=\overline P_{\omega,j}(\mathcal H,\mathcal S_j(\omega)^c)+\overline P_{\omega^{(j)},j}(\mathcal H,\mathcal S_j(\omega))\\
 &\quad=\overline P_{\omega,j}(\mathcal H)-\overline P_{\omega,j}(\mathcal H,\mathcal S_j(\omega))
 +\overline P_{\omega^{(j)},j}(\mathcal H,\mathcal S_j(\omega))\\
 &\quad=\overline P_{\omega,j}(\mathcal H)-\left(\overline P_{\omega,j}(\mathcal H,\mathcal S_j(\omega))
 -\overline P_{\omega^{(j)},j}(\mathcal H,\mathcal S_j(\omega))\right)\\
 &\quad\ge\P_\omega(\mathcal H)-\TV(\overline P_{\omega,j},\overline P_{\omega^{(j)},j})\\
 &\quad\ge\P_\omega(\mathcal H)-\TV(P_\omega^{T-1,\mathrm{pre}},P_{\omega^{(j)}}^{T-1,\mathrm{pre}})\\
 &\quad\ge\P_\omega(\mathcal H)-\TV(P_\omega^{T-1},P_{\omega^{(j)}}^{T-1}).
\end{align*}

In more detail, the first inequality follows from the definition of total variation, while the second follows from the data-processing inequality $\TV(PK,QK)\le \TV(P,Q)$, which holds for any probability measures $P$ and $Q$ on the same measurable space and any common Markov kernel $K$ from that space to another measurable space. We also used the identity $\overline P_{\omega,j}(\mathcal H)=P_\omega(\mathcal H)$, which holds because $K_j$ leaves the pre-arrival history unchanged.

In the original experiment, a candidate-$j$ arrival occurs with probability $\lambda \nu_C/m$, independently of the history immediately before that arrival; conditional on this arrival, the action law is the one produced by $K_j$. Multiplying by the arrival probability and applying \Cref{eq:neighboring-tv,eq:wc-terminal-event}, the paired probability of a wrong final action is at least
\begin{equation}
\begin{aligned}
\frac{\lambda \nu_C}{m}\left(\overline P_{\omega,j}(\mathcal H,\mathcal S_j(\omega)^c)+\overline P_{\omega^{(j)},j}(\mathcal H,\mathcal S_j(\omega^{(j)})^c)\right)
&\geq\frac{\lambda \nu_C}{m}\left(\P_\omega(\mathcal H)-\TV(P_\omega^{T-1},P_{\omega^{(j)}}^{T-1})\right) \\
&\geq \frac{\lambda \nu_C}{m}\left(\frac{\zeta\Delta_0}{\lambda \nu_C}-\frac{\zeta\Delta_0}{4}\right) \\
&=\frac{\zeta\Delta_0}{m}\left(1-\frac{\lambda\nu_C}{4}\right)\ge\frac{3\zeta\Delta_0}{4m}.
\end{aligned}
\end{equation}

Now let $p_{\mathrm{err}}(\omega)=\sum_{j=1}^m \frac{\lambda \nu_C}{m}\overline P_{\omega,j}(\mathcal H,\mathcal S_j(\omega)^c)$. So this is the probability that, for the system with parameter vector $\theta_\omega$, the event $\mathcal{H}$ occurs, a candidate job arrives at the end of round $T-1$, and the system chooses the slower job in the final round. Now, by a double-counting argument, the following equalities and inequality hold. More specifically, the third equality follows by applying the double-counting argument to the vertices and edges of the $m$-dimensional hypercube, whose vertex set is $\{-1, 1\}^m$.

\begin{equation}
\begin{aligned}
\E_{\omega\sim\Unif(\{-1,1\}^m)}[p_{\mathrm{err}}(\omega)]
&=\frac{1}{2^m}\sum_{\omega\in \{-1, 1\}^m}p_{\mathrm{err}}(\omega) \\
&=\frac{1}{2^m}\sum_{\omega\in \{-1, 1\}^m}\sum_{j=1}^m \frac{\lambda \nu_C}{m}\overline P_{\omega,j}(\mathcal H,\mathcal S_j(\omega)^c) \\
&=\frac{1}{2^m}\sum_{\substack{\{\omega, \omega^{(j)}\}\; \mathrm{unordered \; pair} \\ \omega, \omega^{(j)}\in \{-1, 1\}^m}}\frac{\lambda \nu_C}{m}\left(\overline P_{\omega,j}(\mathcal H,\mathcal S_j(\omega)^c)+\overline P_{\omega^{(j)},j}(\mathcal H,\mathcal S_j(\omega^{(j)})^c)\right) \\
&\geq \frac{1}{2^m}\sum_{\substack{\{\omega, \omega^{(j)}\}\; \mathrm{unordered \; pair} \\ \omega, \omega^{(j)}\in \{-1, 1\}^m}} \frac{3\zeta\Delta_0}{4m} \\
&=\frac{3\zeta\Delta_0}{4m}\cdot \frac{m2^{m-1}}{2^m}=\frac{3\zeta\Delta_0}{8}\\
\end{aligned}
\end{equation}

Therefore $\max_{\omega\in \{-1, 1\}^m} p_{\mathrm{err}}(\omega)\geq \E_{\omega\sim\Unif(\{-1,1\}^m)}[p_{\mathrm{err}}(\omega)]\geq \frac{3\zeta\Delta_0}{8}$ holds, so if we choose $\omega_0\in\arg\max_{\omega \in \{-1, 1\}^m} p_{\mathrm{err}}(\omega)$, then $p_{\mathrm{err}}(\omega_0)\geq \frac{3\zeta\Delta_0}{8}$. Changing only the final action to the better member of $\{C_0,C_j\}$ is a feasible work-conserving comparison policy. Thus the queue-length regret for the instance with parameter vector $\theta_{\omega_0}$ is at least $|p_0-p_{C_j}|\cdot p_{\mathrm{err}}(\omega_0)\geq \ell_0a_T\cdot \frac{3\zeta\Delta_0}{8}$, and this instance has regret at least this value.

So this means that for any fixed policy $\pi$, there must be some $\omega\in \{-1, 1\}^m$ such that, if the unknown parameter vector were actually $\theta_\omega$, then the regret would be at least $\ell_0a_T\cdot \frac{3\zeta\Delta_0}{8}$. Since this works for every work-conserving policy $\pi$, this shows that whenever we use a work-conserving policy for $T$ rounds, there exist a distribution $\mathcal{D}$ and a parameter vector $\theta^*$ such that the regret is bounded below by $\ell_0a_T\cdot \frac{3\zeta\Delta_0}{8}$.

Finally, define $b_{d,T}=d^{-1/2}\wedge\sqrt{d/T}$. Since $b_{d,T}\le d^{-1/2}\le1/\sqrt3$ and $\frac{d}{3}\leq m=d-2<d$ hold for $d\geq 3$,
\[
 a_0\ge\sqrt3a_0b_{d,T},
 \qquad \frac{c_S}{\sqrt m}>\frac{c_S}{\sqrt d}\ge c_Sb_{d,T}.
\]
If $T\ge d^2$, then $b_{d,T}=\sqrt{d/T}$ and $\sqrt{m/T}\ge b_{d,T}/\sqrt3$. If $T<d^2$, then $b_{d,T}=d^{-1/2}$ and $\sqrt{m/T}>\sqrt m/d\ge 1/\sqrt{3d}=b_{d,T}/\sqrt3$. Consequently,
\[
 a_T\ge\min\left(\sqrt3a_0,c_S,\frac\alpha{\sqrt3}\right)b_{d,T}.
\]
Therefore, the theorem follows with
\[
 c_{\WC}=\frac{3\ell_0\zeta\Delta_0}{8}
 \min\left(\sqrt3a_0,c_S,\frac\alpha{\sqrt3}\right),
\]
and this lower bound is $\Omega\left(\min\left\{\frac{1}{\sqrt{d}}, \sqrt{\frac{d}{T}}\right\}\right)$.
\end{proof}

\subsection{Two facts for policies that may idle}

\begin{lemma}\label{lem:uniform-ia-value}
Under $\lambda<p_-$,
\[
 \sup_TV_T^{\IA}\le\sup_TV_T^{\WC}
 \le\frac{\lambda(1-\lambda)}{p_--\lambda}.
\]
\end{lemma}

\begin{proof}
Run FCFS. In every round with a job in service, couple its departure indicator with a Bernoulli-$p_-$ indicator so that a departure in the comparison system also occurs in the true system. The true queue is then stochastically dominated by the late-arrival Geo/Geo/1 queue
\[
 \widetilde Q_{t+1}=\left(\widetilde Q_t-S_t\right)^++A_t,
 \qquad
 \P(S_t=1\mid\widetilde Q_t>0)=p_-.
\]
Its drift outside zero is $\lambda-p_-<0$, so it has a stationary version and its zero-started process is stochastically dominated by the stationary version's queue. Let $(\widetilde Q,S,A)$ have the stationary one-round distribution. Taking expectations in the recursion gives
\begin{equation}
\begin{aligned}
0&=\E\left[\left((\widetilde Q-S)^++A\right)\right]-\E\widetilde Q \\
&=\E\left[\left(A-S\right)\cdot \mathbf{1}\left(\widetilde{Q}>0\right)\right]+\E\left[A\cdot \mathbf{1}(\widetilde Q=0)\right] \\
&=\E A-\E\left[S\cdot \mathbf{1}\left(\widetilde{Q}>0\right)\right] \\
&=\lambda-p_-\P(\widetilde Q>0)
\end{aligned}
\end{equation}

and therefore $\E\left[S\cdot \mathbf{1}\left(\widetilde{Q}>0\right)\right]=\E A=\lambda$ and $\P(\widetilde Q>0)=\lambda/p_-$. Since $\E(S\mid\widetilde Q)=\P(S_t=1\mid\widetilde Q_t>0)=p_-$ for $\widetilde Q>0$,
\begin{equation}
\begin{aligned}
\E(\widetilde QS)
&=\E\left[\widetilde QS\cdot \mathbf{1}(\widetilde Q>0)\right]+\E\left[\widetilde QS\cdot \mathbf{1}(\widetilde Q=0)\right] \\
&=\E\left[\widetilde QS\cdot \mathbf{1}(\widetilde Q>0)\right] \\
&=\E\left[\E(S\mid\widetilde Q)\widetilde Q\cdot \mathbf{1}(\widetilde Q>0)\right] \\
&=p_-\E\left[\widetilde Q\cdot \mathbf{1}(\widetilde Q>0)\right]=p_-\E \widetilde Q
\end{aligned}
\end{equation}
Moreover, $A$ is independent of $(\widetilde Q,S)$, $A^2=A$, and $S^2=S$. Expanding the square of the stationary recursion gives
\begin{align*}
 0
 &=\E\left[\left((\widetilde Q-S)^++A\right)^2\right]-\E\widetilde Q^2\\
 &=\E \left[S\cdot \mathbf{1}(\widetilde Q>0)\right]+\E A-2\E\left[\widetilde QS\cdot \mathbf{1}(\widetilde Q>0)\right]
   +2\E\left[\widetilde QA\cdot \mathbf{1}(\widetilde Q>0)\right]-2\E\left[SA\cdot \mathbf{1}(\widetilde Q>0)\right]\\
 &=2\lambda-2p_-\E\widetilde Q+2\lambda \E \widetilde Q-2\lambda^2 \\
 &=-2(p_--\lambda)\E\widetilde Q+2\lambda(1-\lambda).
\end{align*}
Thus $\E\widetilde Q=\frac{\lambda(1-\lambda)}{p_--\lambda}$, and this shows that the expected final queue length of the FCFS policy is at most $\frac{\lambda(1-\lambda)}{p_--\lambda}$ by stochastic dominance. FCFS is feasible for both benchmark classes, so the result is proved.
\end{proof}

\begin{lemma}\label{lem:dummy-start}
Suppose the server is free, only a dummy job with departure probability $p_D$ is waiting, three rounds remain, and every present or future job has departure probability at most $p_D$. For every policy that idles in the first round, there is a policy that serves the dummy immediately and has expected terminal queue smaller by at least $g_D$.
\end{lemma}

\begin{proof}
We first establish a two-round comparison. Suppose a job with departure probability $q$ is waiting and every other present or future job has departure probability at most $q$. If this job is served immediately, let $M$ be the expected highest departure probability available in the final round, conditional on its departure in the first round. The expected number of departures is
\[
 q+(1-q)q+qM=2q-q^2+qM.
\]
Now suppose another waiting job has departure probability $p\le q$ and is served first. If it departs, the $q$ job is available in the last round; if it remains, nonpreemption requires its service to continue. The expected number of departures is therefore
\[
 p+(1-p)p+pq=2p-p^2+pq.
\]
Because the $p$ job remains available after a first-round departure of the $q$ job, $M\ge p$. Subtracting the two expressions gives
\[
 (2q-q^2+qM)-(2p-p^2+pq)
 =(q-p)(2-q-p)+q(M-p)\ge0.
\]
Serving the $q$ job also dominates idling: after idling, at most one departure can occur, with probability at most $q$, whereas $2q-q^2+qM\ge q$.

Return to the three-round state in the lemma and condition on the arrivals and their contexts during the next two rounds. Compare the best continuation after serving the dummy in the first round with the best continuation after idling. If the dummy departs in the first round, the former system begins the second round with the same arrivals and without the dummy. Because policies in this comparison may idle, deleting one job cannot increase the minimum expected terminal queue: the system without the dummy can simulate a policy for the larger system and omit every action assigned to the deleted job. If the dummy does not depart, geometric memorylessness gives it a fresh $\Geom(p_D)$ residual service time. The two systems then begin the second round with the same jobs, except that immediate service is committed to the dummy whereas the idling policy may choose its next action. The dummy has the highest departure probability, so the two-round comparison above shows that continuing its service is optimal. Immediate service in the first round is therefore no worse for every fixed realization of the arrivals and contexts.

It remains to lower-bound the strict improvement. Condition on no arrival at the ends of the first two rounds; this event has probability $(1-\lambda)^2$. Apart from an arrival at the end of the third round, which contributes equally under both first actions, the dummy is the only job. If it is served immediately, it remains at the terminal time with probability $(1-p_D)^3$. After idling in the first round, even the best continuation can serve it in only two rounds, so it remains with probability $(1-p_D)^2$. The conditional reduction in expected terminal queue is therefore
\[
 (1-p_D)^2-(1-p_D)^3=p_D(1-p_D)^2.
\]
The comparison is weakly favorable on every other realization of the arrivals. Averaging thus gives the unconditional improvement
\[
 (1-\lambda)^2(1-p_D)^2p_D=g_D,
\]
so immediate service reduces the expected terminal queue by at least $g_D$.
\end{proof}

\subsection{\texorpdfstring{Proof of \Cref{thm:lower-ia} for IA case}{Proof of the lower bound for policies that may idle}}

\begin{proof}
Define
\[
 c_0=\min\left(\frac{(1-p_D)^4}{16a_0},
 \frac{3\lambda \nu_Dg_D}{32a_0},
 \frac{11\ell_0\zeta}{32}\right).
\]
Suppose for a contradiction that there was a policy $\pi$ (which can be idle) that satisfies $\max_\omega R_T^{\IA}(\pi, \omega)<c_0a_T$. Here, $R_T^{\IA}(\pi, \omega)$ is the regret between policy $\pi$ and the optimum policy among all idling-allowed policies when the unknown parameter vector was $\theta_{\omega}$.

All departure probabilities are at most $p_D$ by definition. Conditional on $Q^\pi_{T-3}>0$, the probability that no job departs in rounds $T-3,T-2,T-1,T$ is at least $(1-p_D)^4$, regardless of the policy's actions. Here we used the fact that every job has departure failure probability at least $1-p_D$, and idling can be viewed as having a departure failure probability of $1$. On this event the terminal queue is nonempty, and hence
\[
 J_T^\pi=\E\left[Q^\pi_{T+1}\cdot \mathbf{1}\left(Q^\pi_{T-3}>0,\text{no departure on last $4$ rounds}\right)\right]\ge(1-p_D)^4\P_\omega(Q^\pi_{T-3}>0)
\]
when our parameter vector was $\theta_{\omega}$. The assumed regret bound gives $J_T^\pi=V_T^{\IA}+R_T^{\IA}(\pi)<V_T^{\IA}+c_0a_T$. By \Cref{lem:uniform-ia-value,eq:lower-lambda}, $V_T^{\IA}\le(1-p_D)^4/16$, and the definition of $c_0$ gives $c_0a_T\le(1-p_D)^4a_T/16a_0\le(1-p_D)^4/16$. Combining these inequalities with the preceding lower bound yields
\begin{equation}
 \P_\omega(Q^\pi_{T-3}>0)<\frac18.
 \label{eq:ia-nearly-empty}
\end{equation}

On $\{Q^\pi_{T-3}=0\}$, consider the event that a dummy arrives at the end of round $T-3$. Let $\mathcal I$ be the event that the learner idles rather than serving this dummy in round $T-2$. By \Cref{lem:dummy-start}, changing only this idle action gives a feasible policy (call it $\pi^*$) whose expected terminal cost is smaller by at least $g_D\P_\omega(Q^\pi_{T-3}=0,A_{T-3}=1,X_{T-3}=C_D,\mathcal I)$. Since $V_T^{\IA}$ is the minimum over all policies that may idle,
\[
 g_D\P_\omega(Q^\pi_{T-3}=0,A_{T-3}=1,X_{T-3}=C_D,\mathcal I)\leq J^\pi_T-J^{\pi^*}_T\leq J^\pi_T-V^{\IA}_T=R_T^{\IA}(\pi)
 <c_0a_T
\]
and this gives $\P_\omega(Q^\pi_{T-3}=0,A_{T-3}=1,X_{T-3}=C_D,\mathcal I)\leq \frac{c_0a_T}{g_D}\le\frac{3\lambda \nu_D}{32}$.

Now we find a lower bound of $\P_\omega(Q^\pi_{T-3}=0,A_{T-3}=1,X_{T-3}=C_D)$ to get a lower bound of $\P_\omega(Q^\pi_{T-3}=0,A_{T-3}=1,X_{T-3}=C_D, \mathcal I^c)$. The dummy arrival at the end of round $T-3$ is independent of $Q^\pi_{T-3}$. So by \Cref{eq:ia-nearly-empty}, $\P_\omega(Q^\pi_{T-3}=0,A_{T-3}=1,X_{T-3}=C_D)= \P_\omega(Q^\pi_{T-3}=0)\cdot \lambda\nu_D>\frac{7\lambda \nu_D}{8}$, and $\P_\omega(Q^\pi_{T-3}=0,A_{T-3}=1,X_{T-3}=C_D, \mathcal I^c)>\frac{7\lambda \nu_D}{8}-\frac{3\lambda \nu_D}{32}=\frac{25\lambda \nu_D}{32}$. This shows that the probability that the following statements all hold is larger than $25\lambda \nu_D/32$: $Q^\pi_{T-3}=0$, a dummy job arrives at the end of round $T-3$, and the learner starts the dummy in round $T-2$.

Next intersect this event with the event that the dummy fails to depart in round $T-2$, and a baseline job arrives in round $T-2$, followed by departure of the dummy in round $T-1$. Let $\mathcal H$ denote the resulting event before the arrival at the end of round $T-1$. Independence gives
\begin{equation}
\begin{aligned}
 \P_\omega(\mathcal H)
 &=\P_\omega(Q^\pi_{T-3}=0,A_{T-3}=1,X_{T-3}=C_D, \mathcal I^c)\cdot (1-p_D)\cdot \lambda\nu_0\cdot p_D \\
 &>\frac{25}{32}\lambda^2\nu_D\nu_0p_D(1-p_D)
 =\frac{25\zeta}{32\lambda \nu_C}.
 \label{eq:ia-terminal-event}
\end{aligned}
\end{equation}

Pair neighboring instances $\omega, \omega^{(j)}$ and use the same kernel $K_j$ and event $\mathcal S_j(\omega)$ as in the work-conserving proof. For each $\omega$, $\mathcal S_j(\omega)^c$ is a wrong decision and includes selecting the smaller probability job or idling at round $T$. Data processing through $K_j$, \Cref{eq:pre-arrival-tv}, and \Cref{eq:neighboring-tv,eq:ia-terminal-event} show that their paired wrong-action probability, including the independent candidate arrival, is at least
\[
 \frac{\lambda \nu_C}{m}\left(\P_\omega(\mathcal H)-\TV(P_\omega^{T-1},P_{\omega^{(j)}}^{T-1})\right)\geq
 \frac{\lambda \nu_C}{m}\left(\frac{25\zeta}{32\lambda \nu_C}-\frac{3\zeta}{32}\right)
 =\frac{25\zeta}{32m}-\frac{3\zeta\lambda\nu_C}{32m}
 \ge\frac{11\zeta}{16m}.
\]
Using the double counting method on the $m$-dimensional hypercube to get an average of $p_{\mathrm{err}}(\omega)$ as shown in the work-conserving proof gives an average wrong-action probability at least $\frac{11\zeta}{32}$. A wrong choice costs $|p_0(\omega)-p_{C_j}(\omega)|\geq \ell_0a_T$ for the case of selecting a small probability job in round $T$, and $\max\{p_0(\omega), p_{C_j}(\omega)\}\geq |p_0(\omega)-p_{C_j}(\omega)|\geq \ell_0a_T$ for the case of idling in round $T$ (here, $p_0(\omega), p_{C_j}(\omega)$ mean $p_0, p_{C_j}$ for the parameter vector $\theta_\omega$). Both cases have wrong choice costs of at least $\ell_0a_T$. Thus there must be some $\omega\in \{-1, 1\}^m$ such that $R_T^{\IA}(\pi, \omega)\geq \frac{11\ell_0\zeta a_T}{32}\ge c_0a_T$, a contradiction. Using the final bound on $a_T$ in the proof of \Cref{thm:lower-wc} for WC case, the theorem holds with
\[
 c_{\IA}=c_0\min\left(\sqrt3a_0,c_S,\frac\alpha{\sqrt3}\right).
\]
\end{proof}

\section{Benchmark Structure}
\label{app:benchmark-proofs}

We first prove the impossibility results for policies that may idle and for work-conserving policies. We then compare the two finite-horizon benchmark values.

\subsection{\texorpdfstring{Proof of \Cref{thm:ia-wc-impossible} for IA case}{Proof of the impossibility theorem for policies that may idle}}
\label{app:ia-impossibility}

\begin{proof}
Let there be two kinds of jobs: a slow one and a fast one, which will be denoted as $s, f$ respectively. Define the distribution $\mathcal{D}$ so that a slow job comes with probability $\nu=10^{-4}$, and a fast job comes with probability $1-\nu$. Choose
\[
 p_s=0.01,\qquad p_f=0.5,\qquad \lambda=0.0205.
\]
as the success probabilities for each job, and the arrival probability. Use contexts $e_1=(1, 0),e_2=(0, 1)$ and parameter vector $\theta^*=(\operatorname{logit}(p_s),\operatorname{logit}(p_f))$. Finally, properly choose fixed constants $S,p_-$, and $p_+$ so that $|\theta^*|<S, p_-\leq 0.01, p_+\geq 0.5$. By this definition, its load is
\[
 \rho=\lambda\left(\frac{\nu}{p_s}+\frac{1-\nu}{p_f}\right)\approx0.0412<1.
\]
For the stationary workload increment $Z=AG$ in \Cref{lem:stationary-workload},
\begin{align*}
 \E Z^2
 &=\lambda \E_X\left[\E_{g\sim \Geom(p(X))}\left(g^2\right)\right]=\lambda\left(
 \nu\frac{2-p_s}{p_s^2}+(1-\nu)\frac{2-p_f}{p_f^2}
 \right)
 \approx 0.1638.
\end{align*}
Here, we used the fact that $\E_{g\sim \Geom(p)}\left(g^2\right)=\frac{2-p}{p^2}$ holds. Substituting this value and value of $\rho$ into \Cref{eq:stationary-workload} gives $\E W\approx 0.1051<0.2$. The zero-started workload is stochastically dominated by the stationary distribution workload. Moreover, every work-conserving policy satisfies $Q_{t+1}\le W_{t+1}$, and every such policy is feasible when idling is allowed. Therefore $V_t^{\IA}\le V_t^{\WC}\le \E W_{t+1}\leq \E W<0.2$ for every $t$.

We will show $\sup_{i\geq t}\{R^{\IA}_{i}(\pi)\}\geq \min\{\frac{\lambda\nu}{4}p_s, \frac{\lambda\nu}{4}g, 0.3\}$ holds for all $t\ge 1$. Here $g=p_s^2-p_s(1+\lambda)+\lambda(1-p_s)\E p(X)\approx 4.1506\times10^{-5}>0$, and this fact shows that $\lim\sup R^{\IA}_{t}(\pi)$ is at least $\min\{\frac{\lambda\nu}{4}p_s, \frac{\lambda\nu}{4}g, 0.3\}$, which is a fixed positive constant.

For simplicity, let $C_0=\min\{\frac{\lambda\nu}{4}p_s, \frac{\lambda\nu}{4}g, 0.3\}$ and suppose, for a contradiction, there exists a $t_0\ge 1$ such that $\sup_{i\geq t_0}\{R^{\IA}_{i}(\pi)\}<C_0$ holds. So $R^{\IA}_{t}(\pi)<C_0$ for all $t\geq t_0$. Since $J_t^\pi=V_t^{\IA}+R_t^{\IA}(\pi)$ and $V_t^{\IA}<0.2$ for all $t$, we have $\E_\pi Q_{t+1}<0.2+C_0\leq 0.5$ for all $t\geq t_0$. Markov's inequality for the nonnegative integer $Q_{t+1}$ gives $\P_\pi(Q_{t+1}=0)=1-\P_\pi(Q_{t+1}\ge 1)\ge 1-\E_\pi Q_{t+1}>1/2$. Now fix $t\ge t_0$. Then the arrival at the end of round $t+1$ is independent of the event $\{Q_{t+1}=0\}$, which means the queue became empty at the beginning of round $t+1$. Hence, with probability greater than $\lambda\nu/2$, the system is empty at the beginning of round $t+1$ and a slow job arrives at the end of that round. At the beginning of round $t+2$, the server is free and exactly one slow job is waiting.

If the policy idles in this state, changing only this action to service reduces the expected terminal queue with horizon $t+2$ by $p_s$, since the expected number of departures differs by $p_s$ in the final round. Now evaluate the same decision with horizon $t+3$, so that two rounds remain. If the policy serves immediately, then the server has two choices, idle or work, in the final round $t+3$ if the server is free and the queue is not zero at the beginning of round $t+3$. In that case, choosing work in round $t+3$ makes the expected number of departures maximal, so in these two rounds the maximum of the expected number of departures is
\[
 p_s(1+\lambda \E p(X))+(1-p_s)p_s 
 =2p_s-p_s^2+\lambda p_s\E p(X).
\]
If it follows the policy that idles in round $t+2$ and serves the fastest job available in the last round $t+3$, the expected number of departures is
\[
 (1-\lambda)p_s+\lambda\E\max\{p_s,p(X)\}
 =(1-\lambda)p_s+\lambda\E p(X),
\]
where the equality uses $p(X)\ge p_s$ in this instance. Thus idling first has an expected-departure advantage at least
\begin{align*}
 g=p_s^2-p_s(1+\lambda)+\lambda(1-p_s)\E p(X)\approx4.1506\times10^{-5}>0.
\end{align*}
So the expectation of queue length decreases by at least $g$ if we change our decision to idle in round $t+2$, rather than serve a slow job.

Let $\mathcal H$ denote the event defined above: $Q_{T+1}=0$ and a slow job arrived at the end of round $t+1$, so the server is free and the slow job is the only waiting job at the beginning of round $t+2$. Let $\mathbf H_{t+1}$ denote the random history through the end of round $t+1$, and, for each realization $\mathfrak h\in\mathcal H$, define

$$
 a(\mathfrak h)
 :=
 \P_\pi\!\left(
 \text{$\pi$ serves the slow job in round $t+2$}
 \,\middle|\,
 \mathbf H_{t+1}=\mathfrak h
 \right).
$$

Because $\pi$ is a horizon-free policy, it uses the same conditional randomized decision rule after the same history $\mathfrak h$, regardless of whether the evaluation horizon is $t+2$ or $t+3$ (this is because the server doesn't know the horizon). A comparison policy that changes only this action is feasible because idling is allowed. Therefore, applying the preceding comparison argument conditionally on each history $\mathfrak h\in\mathcal H$ and then averaging over the common distribution of $\mathbf H_{t+1}$ gives

$$
 R_{t+2}^{\IA}(\pi)+R_{t+3}^{\IA}(\pi)
 \ge
 \E\left[
 \mathbf 1_{\mathcal H}
 \left(
 (1-a(\mathbf H_{t+1}))p_s
 +a(\mathbf H_{t+1})g
 \right)
 \right].
$$

Since $(1-a)p_s+ag\ge\min\{p_s,g\}$ for every $a\in[0,1]$, and since $\P(\mathcal H)\ge\lambda\nu/2$, it follows that

$$
 R_{t+2}^{\IA}(\pi)+R_{t+3}^{\IA}(\pi)
 \ge
 \P(\mathcal H)\min\{p_s,g\}
 \ge
 \frac{\lambda\nu}{2}\min\{p_s,g\}>0.
$$

But by assumption, since $t\ge t_0$ the following inequality holds.
\[
R_{t+2}^{\IA}(\pi)+R_{t+3}^{\IA}(\pi)<2C_0=2\min\{\frac{\lambda\nu}{4}p_s, \frac{\lambda\nu}{4}g, 0.3\}\le \frac{\lambda\nu}{2}\min\{p_s, g\}
\]
This gives a contradiction; therefore we have proved that $\sup_{i\geq t}\{R^{\IA}_{i}(\pi)\}\geq \min\{\frac{\lambda\nu}{4}p_s, \frac{\lambda\nu}{4}g, 0.3\}$ holds. Since $g<p_s$ and $\frac{\lambda \nu}{4}p_s<\frac{1}{4}<0.3$ hold in the above setting, $\min\{\frac{\lambda\nu}{4}p_s, \frac{\lambda\nu}{4}g, 0.3\}=\frac{\lambda\nu}{4}g$ holds and we can simplify the result as follows: for every $t\ge 1$,
\[
\sup_{i\geq t}\{R^{\IA}_{i}(\pi)\}\geq \frac{\lambda\nu}{4}g\approx 2.1272\times 10^{-11}.
\]
\end{proof}

\subsection{\texorpdfstring{Exact certificate for \Cref{prop:wc-reversal}}{Exact work-conserving certificate}}
\label{app:wc-certificate}

Define three jobs as $A, B, F$, and set their arrival probabilities $\nu_A=1/4$, $\nu_B=1/100$, and $\nu_F=37/50$. Together, we set their departure probabilities as $p_A=7/20, p_B=13/20, p_F=19/20$ and $\lambda=1/2$. These arrival probabilities sum to one, and give
\[
 \rho=\frac12\left(\frac{\nu_A}{p_A}+\frac{\nu_B}{p_B}+\frac{\nu_F}{p_F}\right)
 =\frac{6521}{8645}<1.
\]
Let $c=(c_A,c_B,c_F)$ record the waiting job counts, and let $a\in\{0,A,B,F\}$ denote the type in service; $a=0$ means that the server is free. The job in service is not included in $c$, and $e_i$ increments the waiting count of type $i$. For a function $f$, recall the definition of the arrival operator which is defined in \Cref{subsec:bellman-policy}.
\[
 (\mathcal Af)(c,a)=(1-\lambda)f(c,a)+\lambda\sum_{i\in\{A,B,F\}}\nu_i f(c+e_i,a).
\]
Now we will use the Bellman recursion formula which was first shown in \Cref{subsec:bellman-policy}. Let $V_h(c,a)$ be the minimum expected terminal queue with $h$ rounds remaining. Its boundary value is $V_0(c,a)=c_A+c_B+c_F+\ind{a\ne0}$, which denotes the number of remaining jobs (waiting or in service) in the queue. A type-$a$ job in service forces the recursion
\begin{equation}
 V_h(c,a)=p_a(\mathcal AV_{h-1})(c,0)+(1-p_a)(\mathcal AV_{h-1})(c,a).
 \label{eq:certificate-active}
\end{equation}
If the server is free and $c\ne0$, define the value of starting type $i$ for $i$ with $c_i>0$ by
\[
 Q_h(i;c)=p_i(\mathcal AV_{h-1})(c-e_i,0)
 +(1-p_i)(\mathcal AV_{h-1})(c-e_i,i)
\]
Work conservation gives
\begin{equation}
 V_h(c,0)=\min_{i:c_i>0}Q_h(i;c).
 \label{eq:certificate-free}
\end{equation}
When $c=0$, $V_h(0,0)=(\mathcal AV_{h-1})(0,0)$. For each fixed finite horizon $h$, these recursions visit only finitely many states because the distribution has finite support. Moreover, since all probability values are rational, every Bellman term can be computed exactly using rational arithmetic.

For the remainder of this subsection, abbreviate $Q_h(i)=Q_h(i;e_A+e_B)$. Thus, $Q_h(i)$ denotes the expected terminal queue length when the system starts with $h$ rounds remaining, a free server, and the initial queue ${A,B}$, and job $i$ is selected for service in the first round. By calculating the Bellman recursion, we can show that the one-round action values are $Q_1(A)=43/20$ and $Q_1(B)=37/20$, so $B$ is uniquely optimal with gap $3/10$. Iterating \Cref{eq:certificate-active,eq:certificate-free} to $h=14$ gives following values of $Q_{14}(A), Q_{14}(B)$. Here, both of $Q_{14}(A), Q_{14}(B)$ are rational.
\begin{equation}
\begin{gathered}
Q_{14}(A)\approx 1.6375,
\qquad
Q_{14}(B)\approx 1.6396,\\
Q_{14}(B)-Q_{14}(A)
\approx 0.0021>0
\end{gathered}
\label{eq:certificate-gap}
\end{equation}
Thus $A$ is uniquely optimal. Taking contexts $e_A,e_B,e_F$ and $\theta^*=(\operatorname{logit}p_A,\operatorname{logit}p_B,\operatorname{logit}p_F)$ gives a three-dimensional logistic representation with finite $S$. The accompanying script evaluates the same recursion using exact fractions and asserts both action inequalities.

The preceding example further shows that, in a finite-horizon, work-conserving, nonpreemptive system, a policy that selects a job with the highest departure probability whenever the server becomes free need not be optimal. Indeed, when only one round remains, serving $B$, which has the higher departure probability, maximizes the expected number of departures and hence minimizes the expected terminal queue length. In contrast, when 14 rounds remain, serving $A$, which has the lower departure probability, results in a larger expected number of departures and hence a smaller expected terminal queue length. Thus, the optimal decision may depend on the number of remaining rounds and need not coincide with the myopic maximum-departure-probability rule. This is one of the fundamental distinctions between the preemptive and nonpreemptive settings: in the finite-horizon, work-conserving preemptive setting, it is well known that an optimal policy selects a currently available job--server pair with the highest departure probability in every round. See \citet{bae2026algorithm} for more detail.

\subsection{\texorpdfstring{Proof of \Cref{thm:ia-wc-impossible} for WC case}{Proof of the work-conserving impossibility theorem}}
\label{app:wc-impossibility}

\begin{proof}
Use the instance from \Cref{prop:wc-reversal}. By \Cref{eq:lower-empty}, the zero-started common workload satisfies $\P(W_t=0)\ge1-\rho$ for every $t$. Fix $s\ge5$, and starting from the event $W_{s-4}=0$, consider the following event: an $F$ job arrives at the end of round $s-4$; it remains in service in round $s-3$ and an $A$ job arrives; it remains in service in round $s-2$ and a $B$ job arrives; and it departs in round $s-1$, with no arrival at the end of that round. The $F$ job is the only job present when its service begins, and nonpreemption fixes all three service decisions. At the beginning of round $s$, the server is free with exactly one $A$ and one $B$ job waiting. Define the above event as $\mathcal{H}_{\mathrm{WC}}$. Independence of the arrivals and departure indicators gives the lower bound of the probability that this kind of event happens.
\begin{equation}
\begin{aligned}
\P(\mathcal{H}_{\mathrm{WC}})
&=\P(W_{s-4}=0)\cdot \lambda\nu_F\cdot (1-p_F)\cdot \lambda\nu_A\cdot (1-p_F)\cdot \lambda\nu_B\cdot p_F\cdot (1-\lambda) \\
&\ge (1-\rho)\lambda^3(1-\lambda)\nu_F\nu_A\nu_B p_F(1-p_F)^2
\end{aligned}
 \label{eq:wc-decision-event}
\end{equation}
Here, we define $\zeta_{\WC}=(1-\rho)\lambda^3(1-\lambda)\nu_F\nu_A\nu_B p_F(1-p_F)^2\approx 6.7469\times 10^{-8}$.

Let $\mathbf H_{s-1}$ denote the random history through the end of round $s-1$. For each realization $\mathfrak h\in\mathcal H_{\WC}$, define
$$
\alpha_s(\mathfrak h)
:=
\P_\pi\left(
\text{$\pi$ starts $B$ in round $s$}
\,\middle|\,
\mathbf H_{s-1}=\mathfrak h
\right).
$$
Because $\pi$ is horizon-free, it uses the same conditional randomized decision rule after the same history $\mathfrak h$, regardless of whether the terminal horizon is $s$ or $s+13$.

With terminal horizon $s$, one round remains after each history $\mathfrak h\in\mathcal H_{\WC}$. A comparison policy that changes only this last action from $A$ to $B$ is feasible and work-conserving. Therefore
$$
R_s^{\WC}(\pi)
\ge
\E\left[
\mathbf 1_{\mathcal H_{\WC}}
\left(1-\alpha_s(\mathbf H_{s-1})\right)(p_B-p_A)
\right].
$$
Now evaluate the same infinite policy at horizon $s+13$. Fourteen rounds remain after the same history. If the policy starts $B$, even its best continuation has value $Q_{14}(B)$, whereas a comparison policy that starts $A$ and then continues optimally has value $Q_{14}(A)$. Hence
$$
R_{s+13}^{\WC}(\pi)
\ge
\E\left[
\mathbf 1_{\mathcal H_{\WC}}
\alpha_s(\mathbf H_{s-1})
\bigl(Q_{14}(B)-Q_{14}(A)\bigr)
\right].
$$
Adding the two inequalities and using $(1-\alpha)x+\alpha y\ge\min\{x,y\}$ for every $\alpha\in[0,1]$ and $x,y>0$ gives
\begin{equation}
\begin{aligned}
R_s^{\WC}(\pi)+R_{s+13}^{\WC}(\pi)
&\ge 
\E\left[
\mathbf 1_{\mathcal H_{\WC}}
\left(\left(1-\alpha_s(\mathbf H_{s-1})\right)(p_B-p_A)+\alpha_s(\mathbf H_{s-1})
\left(Q_{14}(B)-Q_{14}(A)\right)
\right)\right] \\
&\ge
\P(\mathcal H_{\WC})
\min\left\{p_B-p_A,Q_{14}(B)-Q_{14}(A)\right\}\\
&\ge
\zeta_{\WC}
\min\left\{p_B-p_A,Q_{14}(B)-Q_{14}(A)\right\}\approx 1.425\times 10^{-10}.
\end{aligned}
\end{equation}

So this inequality shows that for $s\ge 5$,

\begin{align*}
\sup_{i\ge s}R_s^{\WC}(\pi)&\ge \frac12\left(R_s^{\WC}(\pi)+R_{s+13}^{\WC}(\pi)\right) \\
&\ge \frac12\zeta_{\WC}
\min\left\{p_B-p_A,Q_{14}(B)-Q_{14}(A)\right\}\\
&\approx 7.125\times 10^{-11}.
\end{align*}

\end{proof}

\subsection{A comparison between the benchmark classes}

\begin{theorem}\label{thm:oracle-bridge}
For every $T\ge1$, $0\le V_T^{\WC}-V_T^{\IA}\le C_W(r)$. For any policy $\pi$, let $D_T^\pi:=\E_\pi\sum_{t=1}^T D_t=\lambda T-J_T^\pi$ be its expected number of departures. In particular, $D_T^{\IA}=\lambda T-V_T^{\IA}$ and $D_T^{\WC}=\lambda T-V_T^{\WC}$. Then
\[
 D_T^{\WC}\ge\left(1-\frac{C_W(r)}{\lambda T}\right)^+D_T^{\IA}.
\]
Moreover, a work-conserving learner with regret $R_T^{\WC}(\pi)$ satisfies
\[
 D_T^\pi\ge\left(1-\frac{C_W(r)+R_T^{\WC}(\pi)}{\lambda T}\right)^+D_T^{\IA}.
\]
\end{theorem}

\begin{proof}
Every work-conserving policy is also feasible when idling is allowed, so $V_T^{\IA}\le V_T^{\WC}$. Under a work-conserving policy, $Q_{T+1}\le W_{T+1}$ pathwise, and \Cref{lem:workload} gives $V_T^{\WC}\le C_W(r)$. This proves the additive comparison.

Summing $Q_{t+1}=Q_t-D_t+A_t$ from $t=1$ to $T$ and using $Q_1=0$ gives
\[
 Q_{T+1}=\sum_{t=1}^TA_t-\sum_{t=1}^TD_t.
\]
Taking expectations and using $\E A_t=\lambda$ yields $D_T^\pi=\lambda T-J_T^\pi$. Hence $D_T^{\WC}\ge\lambda T-C_W(r)$ and $D_T^{\IA}\le\lambda T$. If $\lambda T\le C_W(r)$, the first multiplicative bound has a zero right-hand side. Otherwise,
\[
 D_T^{\WC}\ge\lambda T-C_W(r)
 \ge\left(1-\frac{C_W(r)}{\lambda T}\right)D_T^{\IA}.
\]
Finally, $J_T^\pi\le V_T^{\WC}+R_T^{\WC}(\pi)\le C_W(r)+R_T^{\WC}(\pi)$, and the same two cases prove the learner bound.
\end{proof}

\section{Anytime SEPT Tracking}
\label{app:anytime-proof}

We first state the prediction guarantee and the busy-period bound used in the proof of \Cref{thm:anytime}. We then prove the theorem and establish each supporting result. The two sigma-fields below record the order of service, estimation, and arrivals within a round.

Let $\mathcal F_t$ be the $\sigma$-algebra generated by all information about the queue, the job in service, the observations, and the learner's randomization available at the beginning of round $t$. The scheduling decision is $\mathcal F_t$-measurable. Let $\mathcal G_t$ be the $\sigma$-algebra obtained by additionally including $D_t$ and any estimator update made after service in round $t$, but not $(A_t,X_t)$. Under \Cref{ass:model}, $(A_t,X_t)$ is independent of $\mathcal G_t$, with $A_t\sim\Bern(\lambda)$ and $X_t\sim\mathcal D$ independent of $A_t$.

The following three lemmas are the main ingredients in the proof of \Cref{thm:anytime}. First, \Cref{lem:busy-period-confidence} shows that, with high probability, every sufficiently recent busy period uses a predictor with small prediction error. Recall that $u=\left\lfloor\frac{t}{2}\right\rfloor, m_t=\left\lfloor\frac{\lambda u}{2}\right\rfloor, b_t=e^{-\lambda u/8}+\frac{1}{t^4}$.

\begin{lemma}\label{lem:busy-period-confidence}
Fix $t\ge6$ such that $m_t\ge5$. Then there exists an event $B_t$ with $\P(B_t)\le b_t$ such that, on $B_t^c$, every predictor used in a busy period beginning in a round $s$ with $u+2\le s\le t$ satisfies $\mathcal E(\widehat p)\le\bar\eps_t$.
\end{lemma}

For $s\ge2$, let $\mathcal S_s$ denote the event that the system is empty immediately after service in round $s-1$ and that an arrival occurs at the end of that round. Thus, on $\mathcal S_s$, a new busy period begins in round $s$. The next \Cref{lem:busy-period-tail-comparison} shows that, under two work-conserving policies, the contribution to the expected terminal queue-length difference from a busy period beginning in round $s$ and remaining active up to round $t$ decays exponentially in its age $h=t-s+1$.

\begin{lemma}\label{lem:busy-period-tail-comparison}
Let $h=t-s+1$, and let $C_s$ be the event that a busy period begins in round $s$ and remains active after service in round $t$; equivalently, if we denote $\tau_{\mathrm{bp}}$ as the length of the busy period starting at round $s$ conditional on the event $\mathcal S_s$, $C_s$ is defined as $\mathcal S_s\cap\{\tau_{\mathrm{bp}}>h\}$. Assign each job the same service time under two work-conserving ordering rules. If their queue lengths after round $t$ are $Q_{t+1}^{(1)}$ and $Q_{t+1}^{(2)}$, define $Z_{s,t}=Q_{t+1}^{(1)}-Q_{t+1}^{(2)}$. Then
\begin{equation}
    \E\bigl(|Z_{s,t}|\ind{C_s}\bigr)
    \le
    \E\bigl(W_{t+1}\ind{C_s}\bigr)
    \le
    \frac{M_G(r)}{r}\psi(r)^h.
    \label{eq:busy-period-tail-comparison}
\end{equation}
\end{lemma}
We refer to the above inequality as the busy-period tail bound.

Finally, \Cref{lem:busy-period-perturbation} shows that if a predictor $q$, fixed before a busy period begins, approximates the true departure-probability function $p$ with small prediction error, then $q$-SEPT and true SEPT have close expected terminal queue lengths when both policies are evaluated under the true model $p$. Here, $q$-SEPT is the priority policy that, whenever the server becomes free, selects a waiting job maximizing $q(X)$, thereby treating $q$ as the departure-probability function. For a priority rule $\pi$, $J_{h,\mathrm{stop}}^p(\pi)$ denotes the expected queue length after $h$ rounds when the actual departure probability of a job with context $X$ is $p(X)$ and the jobs are selected according to $\pi$, with the comparison stopped as soon as service leaves no unfinished job. Thus, $J_{h,\mathrm{stop}}^p(q\text{-SEPT})$ is the expected terminal queue length when jobs are ranked using $q$, while their actual departures are governed by $p$.

\begin{lemma}\label{lem:busy-period-perturbation}
Suppose $\mathcal E(q)=\E|q(X)-p(X)|=\eps$. Run a busy period for $h$ rounds, beginning with the first round in which its initial job is served, and stop both comparison systems as soon as service leaves no unfinished job. Then
\begin{equation}
 \left|J_{h,\mathrm{stop}}^p(q\text{-SEPT})-J_{h,\mathrm{stop}}^p(p\text{-SEPT})\right|
 \le \frac{2(h+1)^2}{p_-}\eps.
 \label{eq:busy-period-perturbation}
\end{equation}
\end{lemma}

We conclude this introductory part preceding the main proofs by introducing the function $\Phi_r(\eps)$, which appears in the statement of \Cref{thm:anytime}. For $\eps\in(0,1]$, define
\begin{equation}
 \Phi_r(\eps)=\sum_{a=0}^{\infty}\min\left(
 \frac{2(a+2)^2}{p_-}\eps,
 \frac{M_G(r)}r\psi(r)^{a+1}\right).
 \label{eq:busy-period-sum}
\end{equation}
Here, $M_G(r)$ and $\psi(r)$ are constants defined at the beginning of \Cref{app:workload}. Splitting the series at the smallest $a$ for which $\psi(r)^a\le \varepsilon$ and bounding the terms appropriately, we obtain $\Phi_r(\varepsilon)=O\left(\varepsilon\left(1+\log^3(1/\varepsilon)\right)\right)$, where the constant depends only on $r$, $p_-$, $M_G(r)$, and $\psi(r)$. A detailed proof of this asymptotic bound is provided in \Cref{app:busy-period-sum-proof}.

\subsection{\texorpdfstring{Proof of \Cref{thm:anytime}}{Proof of the anytime theorem}}

\begin{proof}
Run \Cref{alg:anytime} and true SEPT with the same arrivals and contexts, and assign each job the same geometric service time in the two systems. Both policies are work conserving, so \Cref{eq: workload recursion} gives the same workload path. In particular, the two systems become empty after service in exactly the same rounds and therefore have the same busy-period boundaries (i.e. the starting/ending times of busy periods are all the same in both queues).

For $s\in\{2,\ldots,t\}$, let $q_s$ denote the predictor held by the learner immediately after service and any estimator update in round $s-1$. On $\mathcal S_s$, this predictor is fixed before the context of the job initiating the new busy period is observed and remains fixed throughout that busy period. Starting with this initial job which arrived at the end of round $s-1$, compare $q_s$-SEPT and true SEPT using the same subsequent arrivals, contexts, and service times. Stop both comparison systems as soon as service leaves no unfinished job. With $h=t-s+1$, let $Z_{s,t}$ be the difference between their queue lengths after these $h$ rounds, where both stopped queue lengths are defined to be zero if the busy period has already ended.

By definition, on $\mathcal S_s\setminus\mathcal C_s$, the busy period that begins in round $s$ ends by the completion of service in round $t$. The two stopped queue lengths are therefore both zero, and hence $Z_{s,t}=0$. Consequently, $\E\left(Z_{s,t}\ind{\mathcal S_s}\right) =\E\left(Z_{s,t}\ind{\mathcal C_s}\right)$ for every $s$.

Moreover, only the last busy period occurring during the first $t$ rounds can contribute to the terminal queue-length difference. If no busy period remains active after service in round $t$, the two coupled systems have the same terminal queue length. It is therefore sufficient to restrict attention to the event $\bigcup_{s=2}^{t}\mathcal C_s$. Since the events $\mathcal C_2,\ldots,\mathcal C_t$ are pairwise disjoint and $Z_{s,t}$ agrees with the actual terminal queue-length difference on $\mathcal C_s$, we obtain
\begin{equation}
\begin{aligned}
    J_t^{\rm alg}-J_t^{\SEPT}
    &=
    \E\left(
        \left(Q_{t+1}^{\rm alg}-Q_{t+1}^{\SEPT}\right)
        \ind{\bigcup_{s=2}^{t}\mathcal C_s}
    \right) \\
    &=
    \sum_{s=2}^{t}
    \E\left(Z_{s,t}\ind{\mathcal C_s}\right)
    =
    \sum_{s=2}^{t}
    \E\left(Z_{s,t}\ind{\mathcal S_s}\right).
    \label{eq:sept-busy-decomposition}
\end{aligned}
\end{equation}
A busy period beginning in round $t+1$ need not be included, since its initial job has not yet received service and is present in both terminal queues.

For the statistical comparison, define $E_s=\left\{\mathcal E(q_s)\le\bar\eps_t\right\}$ and let $\mathcal H_s=\mathcal G_{s-1}\vee\sigma(A_{s-1})$. The predictor $q_s$ and the events $\mathcal S_s$ and $E_s$ are $\mathcal H_s$-measurable. Conditional on $\mathcal H_s$, the predictor $q_s$ is fixed. Moreover, on the event $\mathcal S_s$, the context of the initial job of the busy period beginning in round $s$ has distribution $\mathcal D$, and the subsequent arrival process is independent of $\mathcal H_s$.

Recall that $Z_{s,t}$ is defined as the difference between the terminal queue lengths in the stopped $h$-round comparison of $q_s$-SEPT and true SEPT, where $h=t-s+1$. Therefore, conditional on $\mathcal H_s$, on the event $\mathcal S_s$, its conditional expectation is precisely the difference between the two expected terminal queue lengths appearing in \Cref{lem:busy-period-perturbation}:
\[
    \E\left(Z_{s,t}\mid\mathcal H_s\right)
    =
    J_{h,\mathrm{stop}}^p\left(q_s\text{-SEPT}\right)
    -
    J_{h,\mathrm{stop}}^p\left(p\text{-SEPT}\right).
\]
Consequently, \Cref{lem:busy-period-perturbation} applies conditionally and gives, on $\mathcal S_s$,
\[
    \left|
        \E\left(Z_{s,t}\mid\mathcal H_s\right)
    \right|
    \le
    \frac{2(h+1)^2}{p_-}\mathcal E(q_s).
\]
In particular, on $\mathcal S_s\cap E_s$,
\[
    \left|
        \E\left(Z_{s,t}\mid\mathcal H_s\right)
    \right|
    \le
    \frac{2(h+1)^2}{p_-}\bar\eps_t.
\]
Since $\mathcal S_s\cap E_s$ is $\mathcal H_s$-measurable, the tower property then gives
\begin{align*}
    \left|
        \E\bigl(
            Z_{s,t}\ind{\mathcal S_s\cap E_s}
        \bigr)
    \right|
    &\le
    \E\left[
        \ind{\mathcal S_s\cap E_s}
        \left|
            \E\bigl(Z_{s,t}\mid\mathcal H_s\bigr)
        \right|
    \right]                                                     \\
    &\le
    \frac{2(h+1)^2}{p_-}\bar\eps_t.
\end{align*}
On $\mathcal S_s\setminus C_s$, the busy period has ended by round $t$, so the stopped discrepancy $Z_{s,t}$ is zero. It follows from \Cref{lem:busy-period-tail-comparison} that
\begin{align*}
    \left|
        \E\bigl(
            Z_{s,t}\ind{\mathcal S_s\cap E_s}
        \bigr)
    \right|
    &=\left|
        \E\bigl(
            Z_{s,t}\ind{\mathcal C_s\cap E_s}
        \bigr)
    \right| \\
    &\le
    \E\bigl(|Z_{s,t}|\ind{C_s\cap E_s}\bigr)                            \\
    &\le \E\bigl(|Z_{s,t}|\ind{C_s}\bigr) \le
    \frac{M_G(r)}{r}\psi(r)^h.
\end{align*}
Combining the preceding two bounds yields
\begin{equation}
    \left|
        \E\bigl(
            Z_{s,t}\ind{\mathcal S_s\cap E_s}
        \bigr)
    \right|
    \le
    \min\left\{
        \frac{2(h+1)^2}{p_-}\bar\eps_t,
        \frac{M_G(r)}{r}\psi(r)^h
    \right\}.
    \label{eq:sept-single-busy-min}
\end{equation}

For $s\in\{u+2,\ldots,t\}$, set $a=h-1=t-s$. Summing \eqref{eq:sept-single-busy-min} over these starting rounds gives
\begin{equation}
\begin{aligned}
    \sum_{s=u+2}^{t}
    \left|
        \E\bigl(
            Z_{s,t}\ind{\mathcal S_s\cap E_s}
        \bigr)
    \right|                                                     
    &\le
    \sum_{a=0}^{t-u-2}
    \min\left\{
        \frac{2(a+2)^2}{p_-}\bar\eps_t,
        \frac{M_G(r)}{r}\psi(r)^{a+1}
    \right\}                                                    
    \le
    \Phi_r(\bar\eps_t).
    \label{eq:term 1}
\end{aligned}
\end{equation}

By \Cref{lem:busy-period-confidence}, on $B_t^c$, every predictor used in a busy period beginning in a round $s\in\{u+2,\ldots,t\}$ satisfies $\mathcal E(q_s)\le\bar\eps_t$. Thus
\[
    \mathcal C_s\cap E_s^c
    \subseteq
    \mathcal S_s\cap E_s^c
    \subseteq B_t,
    \qquad
    u+2\le s\le t.
\]
Since all events $C_s$ are disjoint, $C_s\cap E_s^c$ ($u+2\le s\le t$) are all disjoint in the event $B_t$, so $\bigcup_{s=u+2}^t C_s\cap E_s^c\subset B_t$ holds. Using this fact and the fact that $Z_{s, t}=0$ in the event $\mathcal{S}_s\setminus \mathcal{C}_s$, we can derive the following inequality.
\begin{equation}
\begin{aligned}
    \sum_{s=u+2}^{t}
    \left|
        \E\bigl(
            Z_{s,t}\ind{\mathcal S_s\cap E_s^c}
        \bigr)
    \right|
    &=\sum_{s=u+2}^{t}
    \left|
        \E\bigl(
            Z_{s,t}\ind{\mathcal C_s\cap E_s^c}
        \bigr)
    \right|
    \le\sum_{s=u+2}^{t}
    \E\bigl(
        \left|Z_{s,t}\right|\ind{\mathcal C_s\cap E_s^c}
    \bigr) \\
    &\le\sum_{s=u+2}^{t}
    \E\bigl(
        W_{t+1}\ind{\mathcal C_s\cap E_s^c}
    \bigr)
    =\E\bigl(
        W_{t+1}\ind{\bigcup_{s=u+2}^t \mathcal C_s\cap E_s^c}
    \bigr) \\
    &\le
    \E\bigl(W_{t+1}\ind{B_t}\bigr)
    \label{eq:term 2}
\end{aligned}
\end{equation}
Moreover, $(rw)^2\le e^{rw}$ for every $w\ge0$, so \Cref{lem:workload} gives
\[
    \E W_{t+1}^2
    \le
    \frac{1}{r^2}\E e^{rW_{t+1}}
    \le
    \frac{K_r}{r^2}.
\]
By using the Cauchy--Schwarz inequality for $W_{t+1}$ and $\ind{B_t}$, and the inequality in \Cref{lem:busy-period-confidence}, we get the following upper bound for \Cref{eq:term 2}.
\[
    \E\bigl(W_{t+1}\ind{B_t}\bigr)
    \le
    \sqrt{\E W_{t+1}^2\,\P(B_t)}
    \le
    \frac{\sqrt{K_rb_t}}{r}
\]

Finally, for $C_s$ with $2\leq s\leq u+1$, by applying \Cref{lem:busy-period-tail-comparison} we can get $\E\bigl(W_{t+1}\ind{C_s}\bigr)\leq \frac{M_G(r)}{r}\psi(r)^h$ where $h=t-s+1$. Summing this inequality for $s=2, ..., u+1$, we get the following upper bound.
\begin{equation}
\begin{aligned}
\sum_{s=2}^{u+1} \left|\E\bigl(Z_{s,t}\ind{\mathcal S_s}\bigr)\right|
&=\sum_{s=2}^{u+1} \left|\E\bigl(Z_{s,t}\ind{\mathcal C_s}\bigr)\right|\le\sum_{s=2}^{u+1} \E\bigl(\left|Z_{s,t}\right|\ind{\mathcal C_s}\bigr)\le\sum_{s=2}^{u+1} \E\bigl(W_{t+1}\ind{\mathcal C_s}\bigr) \\
&\le
\frac{M_G(r)}{r}
\sum_{h=t-u}^{t-1}\psi(r)^h
\le
\frac{M_G(r)}{r(1-\psi(r))}
\psi(r)^{t-u}
\le
\frac{M_G(r)}{r(1-\psi(r))}
\psi(r)^u
\label{eq:term 3}
\end{aligned}
\end{equation}
Here, we used the fact $t-u\ge u$.

Combining the three preceding estimates \Cref{eq:term 1}, \Cref{eq:term 2}, \Cref{eq:term 3} with \eqref{eq:sept-busy-decomposition} gives
\begin{equation}
\begin{aligned}
    |J_t^{\rm alg}-J_t^{\SEPT}|
    &=\left|\sum_{s=2}^{t}
    \E\left(Z_{s,t}\ind{\mathcal S_s}\right)\right| \\
    &\le \sum_{s=2}^{u+1}
    \left|\E\left(Z_{s,t}\ind{\mathcal S_s}\right)\right|
    +\sum_{s=u+2}^{t}
    \left|\E\left(Z_{s,t}\ind{\mathcal S_s\cap E_s}\right)\right|
    +\sum_{s=u+2}^{t}
    \left|\E\left(Z_{s,t}\ind{\mathcal S_s\cap E_s^c}\right)\right| \\
    &\le
    \frac{M_G(r)}{r(1-\psi(r))}\psi(r)^u
    +
    \Phi_r(\bar\eps_t)
    +
    \frac{\sqrt{K_rb_t}}{r}
    \label{eq:anytime-decomposition}
\end{aligned}
\end{equation}
Therefore we can obtain \Cref{eq:anytime-bound}.

For fixed model and traffic parameters,
\[
    m_t=\Theta(\lambda t)
    \qquad\text{and}\qquad
    \bar\eps_t
    =
    \wtO\left(
        \sqrt{\frac{d}{\lambda t}}
    \right).
\]
The series calculation in \Cref{app:busy-period-sum-proof} therefore gives the stated rate. If $m_t<5$, both policies are work conserving, so \Cref{lem:workload} implies
\[
    0\le J_t^{\rm alg},J_t^{\SEPT}\le C_W(r),
\]
and hence
\[
    \mathcal T_t^{\SEPT}
    =
    |J_t^{\rm alg}-J_t^{\SEPT}|
    \le
    C_W(r).
\]
\end{proof}

\subsection{\texorpdfstring{Proof of \Cref{lem:busy-period-confidence}}{Proof of the busy-period prediction lemma}}

\begin{proof}
Suppose that a busy period begins in round $s$. Its first job arrived at the end of round $s-1$. The predictor used during this busy period was computed after service in round $s-1$ and before that arrival. Because the queue was empty immediately after service in round $s-1$, every job that arrived by the end of round $s-2$ had already departed, and $(X_i,Y_i)$ had been recorded for each such job. The predictor is therefore fitted to a complete prefix of the i.i.d. job-indexed sequence and is fixed before the context of the initial job is drawn at the end of round $s-1$.

Suppose that $u+2\le s\le t$. Since $u\le s-2$, the data used by the predictor contain all $M_u=\sum_{j=1}^u A_j$ jobs arriving through round $u$. Since $m_t=\lfloor\lambda u/2\rfloor$ and $\E M_u=\lambda u$, a multiplicative Chernoff bound with relative deviation $1/2$ gives
\[
    \P(M_u<m_t)
    \le
    \P\left(M_u<\frac{\lambda u}{2}\right)
    \le
    e^{-\lambda u/8}.
\]
Let $\Omega_t$ be the simultaneous prediction event supplied by \Cref{cor:simultaneous-prediction} with failure probability $t^{-4}$. Define $B_t=\{M_u<m_t\}\cup\Omega_t^c$. Then
\[
    \P(B_t)
    \le
    \P(\{M_u<m_t\})+\P(\Omega_t^c)
    \le
    e^{-\lambda u/8}+\frac{1}{t^4}
    =
    b_t.
\]

On $B_t^c$, consider a predictor used in a busy period beginning in a round $s\in\{u+2,\ldots,t\}$. If $N$ denotes the number of observations used to construct this predictor, then $m_t\le N\le s-2<t$ holds, and \Cref{cor:simultaneous-prediction} gives $\mathcal E(\widehat p_N)\le \eps_N\left(\frac{6}{\pi^2N^2t^4}\right)$. Using $m_t\le N\le t$, we obtain
\begin{align*}
    \mathcal E(\widehat p_N)
    &\le
    \min\left(
        1,
        \sqrt{
            \frac{
                8e^{\bar S}
                \left(
                    d\log(32\bar S N)
                    +
                    \log(\pi^2N^2t^4/6)
                \right)
                +1
            }{2N}
        }
    \right)                                                     \\
    &\le
    \min\left(
        1,
        \sqrt{
            \frac{
                8e^{\bar S}
                \left(
                    d\log(32\bar S t)
                    +
                    \log(\pi^2t^6/6)
                \right)
                +1
            }{2m_t}
        }
    \right)                                                     \\
    &=
    \bar\eps_t.
\end{align*}
This proves the claim.
\end{proof}

\subsection{\texorpdfstring{Proof of \Cref{lem:busy-period-tail-comparison}}{Proof of the busy-period tail lemma}}

\begin{proof}
Couple the two systems so that their arrival indicators, arriving contexts, and job-specific service times agree. Since both policies are work conserving, their workload processes coincide. Consequently, their busy periods begin and end in the same rounds.

Let $G$ be the service time of the job initiating the busy period in round $s$, and let $\widetilde W_h$ be the stopped workload associated with this busy period after $h$ service rounds. On $C_s$, $W_{t+1}=\widetilde W_h$ and $\tau_{\rm bp}>h$ hold. Therefore,
\begin{align*}
    \E\bigl(W_{t+1}\ind{C_s}\bigr)
    &=
    \P(\mathcal S_s)
    \E\left(
        \widetilde W_h
        \ind{\tau_{\rm bp}>h}
        \,\middle|\,
        \mathcal S_s
    \right).
\end{align*}

Conditional on $\mathcal S_s$ and $G$, the future workload increments have the same distribution as those in the stopped busy-period process. Using $w\le e^{rw}/r$ and \Cref{eq:busy-period-workload-tail}, we obtain
\begin{align*}
    \E\left(
        \widetilde W_h
        \ind{\tau_{\rm bp}>h}
        \,\middle|\,
        \mathcal S_s,G
    \right)
    \le
    \frac{1}{r}
    \E\left(
        e^{r\widetilde W_h}
        \ind{\tau_{\rm bp}>h}
        \,\middle|\,
        \mathcal S_s,G
    \right)
    \le
    \frac{e^{rG}}{r}\psi(r)^h.
\end{align*}

The event $\mathcal S_s$ is determined before the context of the initial job is observed. Hence the conditional distribution of $G$ given $\mathcal S_s$ is the usual mixture $G\mid X\sim\Geom(p(X))$, and
\[
    \E\bigl(e^{rG}\mid\mathcal S_s\bigr)
    =
    \E_X\left[
        \frac{p(X)e^r}
        {1-(1-p(X))e^r}
    \right]
    =
    M_G(r).
\]
It follows that
\begin{align*}
    \E\bigl(W_{t+1}\ind{C_s}\bigr)
    &\le
    \P(\mathcal S_s)
    \frac{M_G(r)}{r}\psi(r)^h                                  \\
    &\le
    \frac{M_G(r)}{r}\psi(r)^h.
\end{align*}

Finally, on $C_s$, both terminal queue lengths lie in $[0,W_{t+1}]$. Therefore,
\[
    |Z_{s,t}|
    =
    |Q_{t+1}^{(1)}-Q_{t+1}^{(2)}|
    \le
    W_{t+1}.
\]
This proves \Cref{eq:busy-period-tail-comparison}.
\end{proof}

\subsection{Sensitivity of SEPT to departure probabilities}
\label{app:sept-lipschitz}

\begin{lemma}\label{lem:matching-sept-lipschitz}
Fix $N$ jobs, their arrival rounds, a horizon, and a deterministic tie-breaking rule. For $v\in[p_-,p_+]^N$, let $\mathcal V(v)$ be the expected terminal queue length when job $i$ has service time $\Geom(v_i)$ and the jobs are served in decreasing order of $v_i$. Then
\begin{equation}
    |\mathcal V(v)-\mathcal V(w)|
    \le
    \frac{N}{p_-}\|v-w\|_1.
    \label{eq:matching-sept-lipschitz}
\end{equation}
The same conclusion holds if the process is stopped when the queue first becomes empty.
\end{lemma}

\begin{proof}
First suppose that the coordinates of $v$ and $w$ induce the same strict order. The two systems then use the same priority rule. Couple their service times coordinatewise as in \Cref{lem:geometric-coupling}. A union bound gives
\[
    \P(\text{some pair of service times differs})
    \le
    \frac{1}{p_-}
    \sum_{i=1}^N|v_i-w_i|.
\]
If every pair of service times agrees, the two schedules and their terminal queues agree. Since each terminal queue contains at most $N$ jobs, multiplying the disagreement probability by $N$ proves \Cref{eq:matching-sept-lipschitz} whenever $v$ and $w$ induce the same strict order.

Next, suppose that $v_i=v_j$. Jobs $i$ and $j$ then have identical service-time distributions. If at most one of them is waiting, their relative priority is irrelevant. If both are waiting, exchanging their labels together with their i.i.d. geometric service times leaves the law of the queue unchanged. Repeating this argument after the first departure, or equivalently proceeding by induction on the number of remaining rounds, shows that either relative order gives the same expected terminal queue length. Thus $\mathcal V$ is continuous when two coordinates cross.

For vectors $v$ and $w$ with distinct coordinates, divide the line segment joining them at the finitely many points at which two coordinates become equal. On each resulting subsegment, the coordinate order is fixed, so the preceding coupling argument applies. Summing the resulting bounds proves \Cref{eq:matching-sept-lipschitz}, since the sum of the $\ell_1$-lengths of the subsegments equals $\|v-w\|_1$. General $v,w$ follow by continuity. Stopping the process when the queue first becomes empty does not affect the label-exchange argument, so the same proof applies to the stopped process.
\end{proof}

\subsection{\texorpdfstring{Proof of \Cref{lem:busy-period-perturbation}}{Proof of the busy-period comparison lemma}}
\label{app:busy-period-perturbation}

\begin{proof}
Condition on the context of the initial job and on all subsequent arrival indicators and contexts in the $h$-round comparison. Stop the process as soon as service leaves no unfinished job. Let $N$ be the number of jobs described by these variables. Then $N\le h+1$, and these variables are independent of $q$, which was fixed before the context of the initial job was drawn. Write
\[
    p_i=p(X_i),
    \qquad
    q_i=q(X_i).
\]

Let $J_p(q\text{-SEPT})$ be the conditional expected terminal queue length when job $i$ has departure probability $p_i$ and priority $q_i$. Compare this system with the intermediate system in which job $i$ has both departure probability and priority $q_i$. We have
\begin{align}
    J_p(q\text{-SEPT})-J_p(p\text{-SEPT})
    ={}&
    \left(
        J_p(q\text{-SEPT})-J_q(q\text{-SEPT})
    \right)
    +
    \left(
        \mathcal V(q)-\mathcal V(p)
    \right).
    \label{eq:sept-insertion}
\end{align}

The priority order is fixed in the first difference. Coordinatewise geometric coupling therefore bounds its absolute value by
\[
    \frac{N}{p_-}
    \sum_{i=1}^N|p_i-q_i|.
\]
The second difference is bounded by the same quantity using \Cref{lem:matching-sept-lipschitz}. Hence
\[
    |J_p(q\text{-SEPT})-J_p(p\text{-SEPT})|
    \le
    \frac{2N}{p_-}
    \sum_{i=1}^N|p_i-q_i|.
\]

Averaging over the arrival variables and contexts gives
\[
    \E\sum_{i=1}^N|p_i-q_i|
    \le
    (h+1)\mathcal E(q).
\]
Since $N\le h+1$, we can show the desired result.
\end{proof}

\subsection{Asymptotic dependence of $\Phi_r(\eps)$ on $\eps$}
\label{app:busy-period-sum-proof}

For $\eps\in(0,1]$, let
\[
    L_\eps
    =
    \left\lceil
        \frac{\log(1/\eps)}
        {\log(1/\psi(r))}
    \right\rceil.
\]
By definition, $\psi(r)^{L_\eps}\le\eps$. Bounding the terms indexed by $a\in\{0,\ldots,L_\eps\}$ by their statistical bound and the remaining terms by their busy-period tail bound gives
\begin{align*}
    \Phi_r(\eps)
    &\le
    \frac{2\eps}{p_-}
    \sum_{a=0}^{L_\eps}(a+2)^2
    +
    \frac{M_G(r)}{r}
    \sum_{a=L_\eps+1}^{\infty}\psi(r)^{a+1}                    \\
    &<
    \frac{
        2\eps
        (L_\eps+2)(L_\eps+3)(2L_\eps+5)
    }{6p_-}
    +
    \frac{M_G(r)}
    {r(1-\psi(r))}
    \psi(r)^{L_\eps+2}                                         \\
    &\le
    \frac{2\eps(L_\eps+3)^3}{3p_-}
    +
    \frac{M_G(r)}
    {r(1-\psi(r))}
    \eps.
\end{align*}
Since
\[
    L_\eps
    \le
    1+
    \frac{\log(1/\eps)}
    {\log(1/\psi(r))},
\]
we conclude that
\[
    \Phi_r(\eps)
    =
    O\left(
        \eps\bigl(1+\log^3(1/\eps)\bigr)
    \right),
\]
with the dependence stated after \Cref{eq:busy-period-sum}.

\subsection{Average tracking and a uniformly random horizon}

From \Cref{thm:anytime} and $\sum_{t=1}^Tt^{-1/2}\le2\sqrt T$,
\[
    \frac{1}{T}
    \sum_{t=1}^T
    \mathcal T_t^{\SEPT}
    =
    \wtO\left(
        \sqrt{\frac{d}{\lambda T}}
    \right).
\]
If $\tau$ is independent of the queueing process and uniformly distributed on $[T]$, then conditioning on $\tau$ gives
\[
    \E Q_{\tau+1}
    =
    \frac{1}{T}
    \sum_{t=1}^T
    \E Q_{t+1}.
\]
Thus, the same rate compares the learner and SEPT when the evaluation horizon is chosen independently and uniformly from $[T]$.

\subsection*{Use of Generative-AI Tools}

The authors formulated the research problem and used GPT-5.6 Sol to explore possible solutions alongside their own investigation. The model proposed an alternative approach and generated initial theoretical results and proof drafts. The authors checked and revised these drafts, correcting errors and filling gaps in the arguments. In particular, the authors corrected and refined the proofs of the lower bound in Theorem 2 and the impossibility of an anytime policy with vanishing regret under an unknown horizon in Theorem 3. GPT-5.6 Sol was also used for language editing and manuscript polishing. The authors take responsibility for the correctness and presentation of the final work.

\end{document}